\documentclass[10pt,journal,compsoc]{IEEEtran}

\usepackage[nocompress]{cite}
\usepackage{threeparttable}
\usepackage{booktabs}
\usepackage{float}
\usepackage{extarrows}
\usepackage{graphicx}
\usepackage{epstopdf}
\usepackage{xcolor}
\usepackage{epsfig}
\usepackage{mathtools}
\usepackage{makecell}
\usepackage{multirow}
\usepackage{amsmath}
\usepackage{amssymb}
\usepackage{amsthm}
\usepackage{mathrsfs}
\usepackage{color}
\usepackage{url}
\usepackage{array}
\usepackage{multirow}
\usepackage{algorithm}
\usepackage{algorithmic}
\usepackage{bigstrut}
\usepackage[T1]{fontenc}
\usepackage[justification=centering]{caption}

\usepackage{rotating}
\usepackage{xr}

\renewcommand{\algorithmicrequire}{\quad \textbf{Input:}} 
\renewcommand{\algorithmicensure}{\quad \textbf{Output:}} 

\newtheorem{definition}{Definition}

\newtheorem{proposition}{Proposition}

\newtheorem{remark}{Remark}

\numberwithin{equation}{section}
\numberwithin{definition}{section}
\numberwithin{theorem}{section}
\numberwithin{lemma}{section}
\numberwithin{corollary}{section}
\numberwithin{proposition}{section}
\numberwithin{remark}{section}
\numberwithin{assumption}{section}
\newlength{\figurewidth}
\newlength{\smallimage}
\newlength{\mediumimage}
\newlength{\largeimage}
\newlength{\smallplot}
\newlength{\mediumplot}
\begin{document}

\title{Robust Principal Component Completion}

\author{Yinjian~Wang,
  Wei~Li,~\IEEEmembership{Senior Member,~IEEE}, Yuanyuan~Gui, 
  James~E.~Fowler,~\IEEEmembership{Fellow,~IEEE}, \\
  and Gemine~Vivone,~\IEEEmembership{Senior Member,~IEEE}
  
  \IEEEcompsocitemizethanks{\IEEEcompsocthanksitem This paper is supported by NSFC Projects of International Cooperation and
  	Exchanges under Grant W2411055 (Corresponding author: Wei Li).
  }
  \IEEEcompsocitemizethanks{\IEEEcompsocthanksitem Y.~Wang, W.~Li and
    Y.~Gui are with the School of Information and Electronics, Beijing
    Institute of Technology, and the National Key Laboratory of
    Science and Technology on Space-Born Intelligent Information
    Processing, 100081 Beijing, China 
    (e-mail: yinjw@bit.edu.cn, liwei089@ieee.org, 953647315@qq.com).
  } \IEEEcompsocitemizethanks{\IEEEcompsocthanksitem J.~E.~Fowler is
    with the Department of Electrical and Computer Engineering,
    Mississippi State University, Starkville, MS 39762 USA (e-mail:
    fowler@ece.msstate.edu).}
  \IEEEcompsocitemizethanks{\IEEEcompsocthanksitem G.~Vivone is with
    the National Research Council, Institute of Methodologies for
    Environmental Analysis (CNR-IMAA), 85050 Tito, Italy, and also
    with National Biodiversity Future Center (NBFC), 90133 Palermo,
    Italy (e-mail: gemine.vivone@imaa.cnr.it).}  }

\IEEEtitleabstractindextext{%
\begin{abstract}
  Robust principal component analysis (RPCA) seeks a low-rank component
and a sparse component from their summation.  Yet, in many
applications of interest, the sparse foreground actually replaces, or
occludes, elements from the low-rank background.  To address this
mismatch, a new framework is proposed in which the sparse component is
identified indirectly through determining its support.  This approach,
called robust principal component completion (RPCC), is solved via
variational Bayesian inference applied to a fully probabilistic
Bayesian sparse tensor factorization.  Convergence to a hard
classifier for the support is shown, thereby eliminating the post-hoc
thresholding required of most prior RPCA-driven
approaches. Experimental results reveal that the proposed approach
delivers near-optimal estimates on synthetic data as well as robust
foreground-extraction and anomaly-detection performance on real color video
and hyperspectral datasets, respectively. Source implementation and Appendices are available at  \url{https://github.com/WongYinJ/BCP-RPCC}.

\end{abstract}

}

\maketitle

\IEEEdisplaynontitleabstractindextext

\ifCLASSOPTIONcompsoc
\IEEEraisesectionheading{\section{Introduction}\label{sec:introduction}}
\else
\section{Introduction}
\label{sec:introduction}
\fi
\IEEEPARstart{T}{he} presence of unexpected outliers is ubiquitous in a wide range of data. Ranging from meaningless noise to highly informative anomalies,
outliers vary substantially in form but are commonly characterized by
their distinctness from a well-structured background as well as a
low-probability of occurrence---two key properties encompassed by the
concept of sparsity. Though sparse, outliers can detach the data by a
large degree from the low-dimensional manifold upon which it would
otherwise lie.  Given the widespread use of principal component analysis (PCA)
for low-dimensional, or low-rank, analysis and
representation, a desire for resilience to outliers has thus inspired
significant interest in robust PCA (RPCA)
\cite{WGR2009,LLS2011,CLM2011,BJZ2018}.
Mathematically, the
RPCA problem seeks a low-rank matrix $\mathbf{L}$ coupled with a
sparse matrix $\mathbf{S}$ such that observed matrix $\mathbf{Y}$ is
\begin{equation}
  \mathbf{Y} = \mathbf{L} + \mathbf{S} .
  \label{eq:rpcamatrix}
\end{equation}
Despite the rapid development in recent years of techniques for solving
\eqref{eq:rpcamatrix},
this RPCA formulation is somewhat divorced from the reality
faced by many would-be applications of RPCA.
That is, \eqref{eq:rpcamatrix} formulates the sparse outliers, or
anomalies, $\mathbf{S}$ as being additive to the low-dimensional
background $\mathbf{L}$.
Yet, in most real applications, the outliers physically
replace, or occlude, the background rather than add to it.
It can be argued that a more accurate model of the outlier phenomenon would 
be embodied by
\begin{equation}
  \mathbf{Y} = \mathscr{P}_{\mathtt{\Omega}^{\bot}}\bigl[\mathbf{L}\bigr] +
  \mathbf{S} ,
  \label{eq:RPCAConstraint}
\end{equation}
where $\mathtt{\Omega}$ denotes the support of $\mathbf{S}$ with
$\mathtt{\Omega}^{\bot}$ being its complement, and
$\mathscr{P}_{\mathtt{\Omega}^{\bot}}\left[\cdot\right]$ denotes the
orthogonal projector onto the subspace of matrices supported on
$\mathtt{\Omega}^{\bot}$.

In general, solving \eqref{eq:rpcamatrix} will not yield an
$(\mathbf{L}, \mathbf{S})$ pair that satisfies the more
application-meaningful \eqref{eq:RPCAConstraint}.  Consequently, most
RPCA-based techniques strive to identify one of the two components
first, and then seek the other one in a second step---either
explicitly or implicitly.  That is, in practice, RPCA models actually
work toward finding either $\left(\mathbf{L},
\mathbf{S}-\mathscr{P}_{\mathtt{\Omega}}\left[\mathbf{L}\right]\right)$
or $\left(\mathscr{P}_{\mathtt{\Omega}^\bot}\left[\mathbf{L}\right],
\mathbf{S}\right)$.  This conundrum is illustrated in
Fig.~\ref{fig:rpcabias} wherein we see that $\mathbf{Y} \neq
\mathbf{L} + \mathbf{S}$ in the case that $\mathbf{S}$ occludes
$\mathbf{L}$, and yet solving \eqref{eq:rpcamatrix} produces
$\mathbf{S} - \mathscr{P}_{\mathtt{\Omega}}\left[\mathbf{L}\right]
\neq \mathbf{S}$, or, alternatively,
$\mathscr{P}_{\mathtt{\Omega}^{\bot}}\left[\mathbf{L}\right] \neq
\mathbf{L}$.  If we assume that \eqref{eq:rpcamatrix} is solved for
$\left(\mathbf{L},
\mathbf{S}-\mathscr{P}_{\mathtt{\Omega}}\left[\mathbf{L}\right]\right)$---which
is the more common scenario---then, typically,
$\mathbf{S}-\mathscr{P}_{\mathtt{\Omega}}\left[\mathbf{L}\right]$ is
thresholded to determine an estimated support,
$\widehat{\mathtt{\Omega}}$, in a second step. Finally, the sparse
component is estimated as $\widehat{\mathbf{S}} =
\mathscr{P}_{\widehat{\mathtt{\Omega}}}\left[\mathbf{Y}\right]$ using
this estimated support. In this sense, RPCA provides a soft classifier
for the support which must be subjected to subsequent thresholding in
order to arrive at an estimate of the sparse component.

Such a mismatch between RPCA and real-world scenarios is not subtle, nor is it difficult to formulate a well-matched alternative like \eqref{eq:RPCAConstraint}. The true difficulty lies in solving it, due to its NP-hardness. An efficient solution was proposed in \cite{DECOLOR2013}, but it is largely confined to grayscale video data and does not generalize easily to data of arbitrary order or type. Moreover, the graph-cut algorithm it relies on for estimating the support $\mathtt{\Omega}$ is not scalable to high-dimensional data. Furthermore, incorporating graph cuts into an alternating minimization framework remains heuristic at best. Lacking a robust theoretical foundation to analyze convergence or optimality, its potential as a general solution framework is questionable. 

In this paper, we depart from this usual RPCA formulation and instead
propose a framework we call robust principal component completion
(RPCC) as a direct solution to \eqref{eq:RPCAConstraint}.
In the RPCC paradigm, the sparse component is
found indirectly by determining its support.
Additionally, in order to handle arbitrary dimensionality of the
datasets, the entire problem is cast in the form of tensors employing
a fully probabilistic Bayesian sparse tensor factorization (BSTF)
solved via variational Bayesian inference (VBI) \cite{WB2005}.
Specifically, we make the following primary contributions:

(1) The proposed RPCC formulation accurately models real-world
problems in which the sparse component does not merely ``add to'' but
actually replaces or occludes the corresponding elements in the
low-dimensional background by substituting the support of the sparse
component for the component itself in the optimization variable
set. Formulated from a tensor perspective, RPCC handles data of
arbitrary dimensions, and, by casting the support in the form of
blocks, RPCC flexibly adapts to data exhibiting blockwise patterns (as
is common) while easily accommodating the simpler element-wise form if
needed.

(2) We focus on the canonical polyadic (CP) tensor decomposition,
developing a Bayesian CP (BCP) factorization
to solve RPCC; we call the resulting algorithm BCP-RPCC.
BCP-RPCC exploits sparsity-inducing
hierarchical distributions to capture low CP-rank as well as
blockwise sparsity of the two latent components. In RPCC:
\begin{itemize}
\item The determination of support $\mathtt{\Omega}$ being NP-hard,
  we approximate the problem as a probabilistic binary classification.
  With each block in the observation being generated from a
  mixture distribution of two latent patterns controlled by a
  Bernoulli random variable, the posterior expectation of this latter
  variable indicates the probability of the corresponding block
  belonging to the sparse pattern.
\item In constructing a generative model for the observation, we
  manually introduce additive Gaussian noise with known variance.
  While Gaussian noise is commonly assumed in BSTF formulations as a
  source of randomness, its variance must typically be estimated from
  a noisy observation. To avoid over-parameterization in RPCC, we
  assume the original observation to be noiseless, consistent with
  conventional RPCA. Thus, the added noise provides the randomness
  necessary for BSTF without introducing additional unknown variables.
\item The posterior distributions are inferred via VBI, and we show
  that the proposed framework results in a hard classifier in terms of
  support separation, as long as the variance of the manually added
  noise is sufficiently small. This result imbues BCP-RPCC with a clear
  advantage over RPCA-based solutions,
  which produce merely soft classifiers requiring
  a subsequent thresholding step with a threshold that is
  typically very difficult to determine in practice.
\end{itemize}

(3) Eventually, BCP-RPCC yields robust performance across different applications. It demonstrates, for the first time, that problem \eqref{eq:RPCAConstraint} (or more precisely, \eqref{eq:RPCC}) is generally solvable for data of arbitrary order or type. Grounded on Bayesian probabilistic learning and VBI, the convergence and optimality of BCP-RPCC are both analyzable, establishing it as a theoretically robust solution framework. In fact, RPCA serves as a convex surrogate for RPCC---a compromise that was previously necessitated by the lack of a viable solver for the latter. With the proposal of BCP-RPCC, this compromise is no longer necessary.

The remainder of the manuscript is organized as follows.
First, in Sec.~\ref{sec:notations}, we introduce various notations and
preliminary mathematical concepts, and, in Sec.~\ref{sec:background},
we review relevant background on RPCA and BSTF.
The main RPCC formulation and corresponding BCP-RPCC algorithm are
presented in Sec.~\ref{sec:rpcc}, while empirical
results are presented in Sec.~\ref{sec:results}. Finally, some concluding remarks are
drawn in Sec.~\ref{sec:conclusion}. Sec.~\ref{sec:Discussion} discusses potential follow-up studies stemming from this paper.
 
\begin{figure}[t]
  \centering
  \setlength{\tabcolsep}{0.1mm}
  \begin{tabular}{ccc}
    \includegraphics[width=\largeimage]{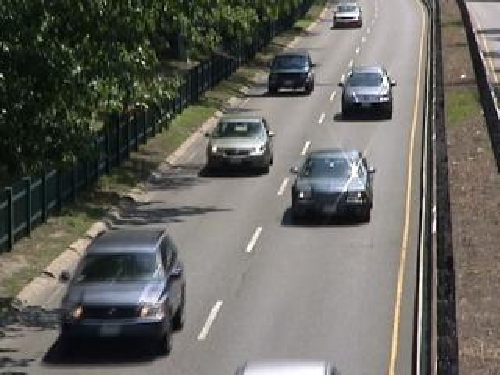} &
    \includegraphics[width=\largeimage]{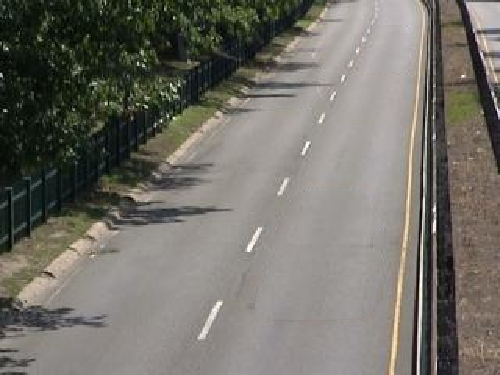} &
    \includegraphics[width=\largeimage]{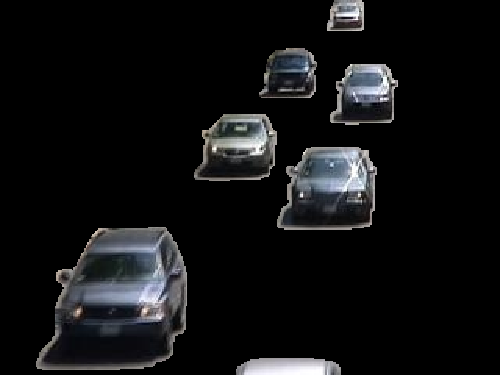} \\[-0.25em]
    \multicolumn{1}{c}{\footnotesize $\mathbf{Y}$} &
    \multicolumn{1}{c}{\footnotesize $\mathbf{L}$} &
    \multicolumn{1}{c}{\footnotesize $\mathbf{S}$} \\[0.25em]
    \includegraphics[width=\largeimage]{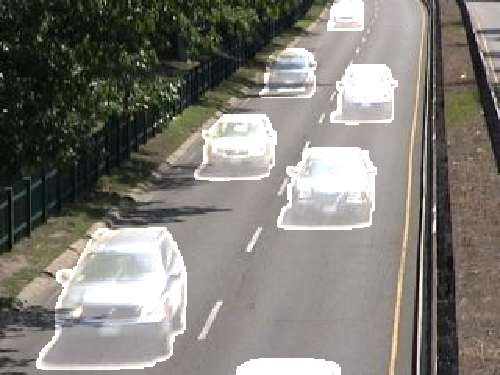} &
    \includegraphics[width=\largeimage]{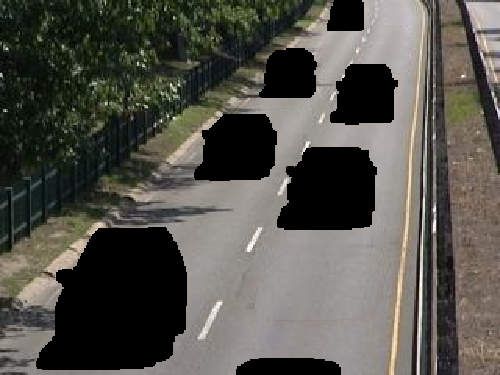} &
    \includegraphics[width=\largeimage]{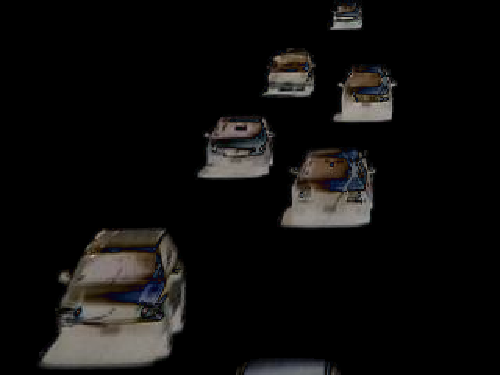} \\[-0.25em]
    \multicolumn{1}{c}{\footnotesize $\mathbf{L}+\mathbf{S} \neq \mathbf{Y}$} &
    \multicolumn{1}{c}{\footnotesize $\mathscr{P}_{\mathtt{\Omega}^\bot}\left[\mathbf{L}\right] \neq \mathbf{L}$} &
    \multicolumn{1}{c}{\footnotesize $\mathbf{S}-\mathscr{P}_{\mathtt{\Omega}}\left[\mathbf{L}\right] \neq \mathbf{S}$} 
  \end{tabular}
  \caption{Visual comparison of various quantities of the RPCA
    formulation for an image.}
  \label{fig:rpcabias}
\end{figure}

\section{Notations and Preliminaries}
\label{sec:notations}
\begin{figure}[t]
  \centering
  \setlength{\tabcolsep}{1mm}
  \begin{tabular}{c}
    \includegraphics[width=0.9\columnwidth]{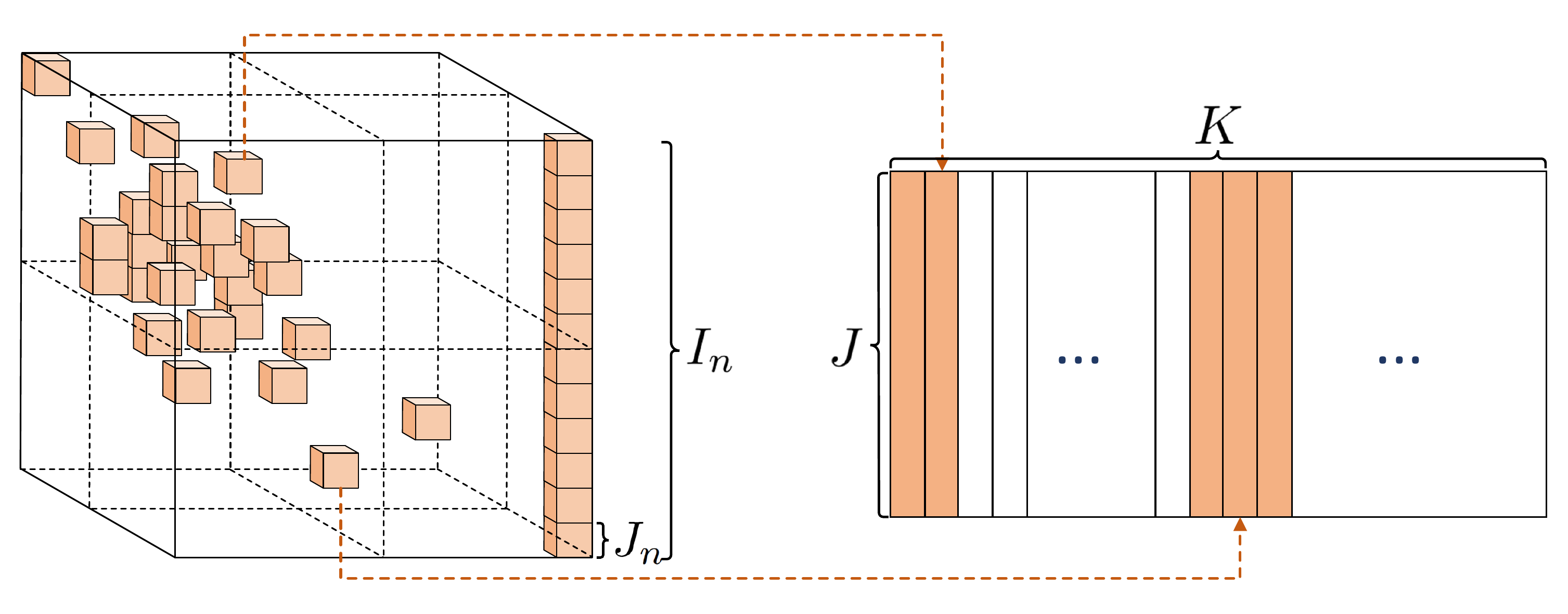}
  \end{tabular}
  \caption{\label{fig:B-unfolding}B-Unfolding: The reorganization of each
    tensor block into a single column of a matrix.}
\end{figure}

We denote scalar, vector, matrix, and higher-order tensors
with $x$, $\mathbf{x}$, $\mathbf{X}$, and $\mathcal{X}$,
respectively. $\mathbf{I}_n$ denotes the $n\times n$ identity matrix,
while $\operatorname{diag}(\mathbf{x})$ forms a diagonal matrix from
vector $\mathbf{x}$. $[N]$ denotes the set of positive integers no
greater than $N$, i.e., $[N]=\left\{1,2,\dots,N\right\}$. For any
$N$-order tensor $\mathcal{X}\in\mathbb{R}^{I_1\times
  \cdots\times I_N}$, $\mathcal{X}_{i_1\cdots i_N}$ is the
scalar at location $(i_1\cdots i_N)$. For a matrix
$\mathbf{X}\in\mathbb{R}^{I_1\times I_2}$, $\mathbf{X}^T$ is its
conjugate transpose, and $\mathbf{x}_{i_2}$ denotes column $i_2$ of
$\mathbf{X}$.  We define the inner product between tensors
$\mathcal{X}$ and $\mathcal{Y}$ of the same size as
$\langle\mathcal{X},\,\mathcal{Y}\rangle=\sum_{i_1,\dots,i_N}\mathcal{X}_{i_1
  \cdots i_N}\mathcal{Y}_{i_1 \cdots i_N}$ such that the Frobenius
norm of a tensor is defined as
$\left\|\mathcal{X}\right\|_F=\sqrt{\langle\mathcal{X},\,\mathcal{X}\rangle}$.
For a set of matrices
$\{\mathbf{A}^{(n)}\in\mathbb{R}^{I\times J}\}_{n=1}^N$,
the generalized inner product is defined as
\begin{equation}
  \langle\mathbf{A}^{(1)},\dots,\mathbf{A}^{(N)}\rangle\triangleq\sum_{i,j}\prod_n\mathbf{A}^{(n)}_{ij} ,
\end{equation}
while the generalized Hadamard product is 
\begin{equation}
  \mathop{\scalebox{2}{$\odot$}}_{n}\mathbf{A}^{(n)}\triangleq\mathbf{A}^{(1)}\odot\cdots\odot\mathbf{A}^{(N)} ,
\end{equation}
where $\odot$ denotes the Hadamard product between two
matrices. For any $N$-order tensor
$\mathcal{X}\in\mathbb{R}^{I_1\times\cdots\times I_N}$, its
CP factorization is 
\begin{equation}
  \mathcal{X}=\operatorname{CP}\left\{\mathbf{A}^{(1)},\mathbf{A}^{(2)},\dots,\mathbf{A}^{(N)}\right\} ,
\end{equation}
where $\bigl\{\mathbf{A}^{(n)}\in\mathbb{R}^{I_n\times R}\bigr\}_{n=1}^N$ are
the CP factors with $R$ being the CP rank. Note that, with a slight
abuse of notation, for a CP factor
$\mathbf{A}^{(n)}\in\mathbb{R}^{I_n\times R}$,
$\mathbf{a}^{(n)}_{i_n}$ is no longer its ${i_n}^{\text{th}}$ column
vector but rather its ${i_n}^{\text{th}}$ row transposed into a
vector.

Throughout the paper, $p(\cdot)$ and $q(\cdot)$ are reserved for
probability density functions, and
$\mathbf{a}\sim\mathcal{N}\left(\boldsymbol{\mu},\mathbf{\Sigma}\right)$
indicates that random vector $\mathbf{a}$ is drawn from a multivariate
Gaussian distribution with mean $\boldsymbol{\mu}$ and covariance
$\mathbf{\Sigma}$. Additionally, $\Gamma(\cdot)$, $\psi(\cdot)$, and
$\gamma(\cdot,\cdot)$ are the gamma, digamma, and lower incomplete
gamma functions, respectively. $E[\cdot]$ represents the expectation
of a random variable, and $H\left(\cdot\right)$ is entropy.
$\mathbf{1}_{\mathsf{K}}$ denotes the indicator function of certain
subset $\mathsf{K}$.

Real-world objects usually occupy contiguous volumes in space, time,
and wavelength. Consequently, such objects exhibit blockwise patterns
in sensed data. This motivates a view from a block
standpoint. As illustrated by 
Fig.~\ref{fig:B-unfolding}, an order-$N$ tensor
$\mathcal{X}\in\mathbb{R}^{I_1\times\cdots\times I_N}$ can be
imagined to be assembled from smaller order-$N$ tensor
blocks. That is, suppose $I_n=J_nK_n,\,\forall\,n\in[N]$. Then, each small
tensor block is of size $J_1\times\cdots\times J_N$, and
there are a total of $\prod_{n}K_n$ such blocks.
These blocks can be used as the basis of tensor unfolding.
\begin{definition}[Blockwise Unfolding (B-unfolding) \cite{WLG2025}]
  For an order-$N$ tensor $\mathcal{X}\in\mathbb{R}^{I_1\times\cdots\times I_N}$ with $I_n=J_nK_n$, $\forall n\in[N]$, its
  B-unfolding, denoted by matrix $\mathbf{X}^{[K]}$, where
  $\mathbf{X}^{[K]}\in\mathbb{R}^{J\times K}$,
  $J\triangleq\prod_{n}J_n$, and $K\triangleq\prod_{n}K_n$, is defined
  as
  \begin{multline}
    \mathbf{X}^{[K]}=\mathtt{reshape}\Big(\mathtt{permute}\big(\mathtt{reshape}(\mathcal{X},\\
	       [J_1,K_1,\dots,J_N,K_N]),
	       [1,3,\dots,2N-1,2,4,\cdots,2N]\big), \\
	       J,[\,]\Big)
  \end{multline}
  where $\mathtt{reshape}(\cdot)$ and $\mathtt{permute}(\cdot)$ denote
  tensor operators with the corresponding functionality of their
  \textsc{Matlab} namesakes.
  \label{Def:B-Unfolding}
\end{definition}
\noindent
Note that, in accordance with the notations above,
$\mathbf{x}^{[K]}_k$ now denotes the $k^{\text{th}}$ column of
$\mathbf{X}^{[K]}$, which is the vectorization of the $k^{\text{th}}$
block in $\mathcal{X}$, as depicted by
Fig.~\ref{fig:B-unfolding}. 
\begin{definition}[Blockwise Support]
  Following the notation in Def.~\ref{Def:B-Unfolding}, suppose we
  have a set $\mathtt{\Omega}\subset[K]$. If $k\in \mathtt{\Omega}$
  iff the $k^{\text{th}}$ block (column) of $\mathcal{X}$
  ($\mathbf{X}^{[K]}$) is non-zero, then $\mathtt{\Omega}$ is defined
  as the blockwise support of $\mathcal{X}$.
  \label{Def:BSupp}
\end{definition}
\noindent
If the blockwise support of a tensor is sparse, we say that this
tensor is blockwisely sparse. Additionally, $\mathtt{\Omega}^\bot$ is
the complement of $\mathtt{\Omega}$ in $[K]$, which is also a
blockwise support.
\begin{definition}[Blockwise Projector]
  When $\mathtt{\Omega}$ is a blockwise support, we define
  $\mathscr{P}_{\mathtt{\Omega}}[\cdot]:\mathbb{R}^{I_1\times\cdots\times
    I_N}\rightarrow\mathbb{R}^{I_1\times\cdots\times I_N}$ to be a
  blockwise projector that forces all the blocks not in
  $\mathtt{\Omega}$ to zero. Specifically.
    \begin{equation}
      \mathscr{P}_{\mathtt{\Omega}}\left[\mathcal{X}\right]_{i_1\cdots i_N}=\left\{  
      \begin{aligned}  
	&\mathcal{X}_{i_1\cdots i_N},\,&(i_1\cdots i_N)\in\mathtt{\Omega}_k,\,k\in\mathtt{\Omega},\\
	&0,\,&\text{otherwise} ,
      \end{aligned}  
      \right.	
    \end{equation}
    where $\mathtt{\Omega}_k\subset[I_1]\times\cdots\times[I_N]$
    consists of the indexes of all elements in the $k^{\text{th}}$ block of
    $\mathcal{X}$.  
  \label{Def:BProj}
\end{definition}

\section{Background}
\label{sec:background}
\subsection{RPCA}
\label{sec:backgroundrpca}
It has long been recognized that even a few outliers can greatly skew
the results of PCA, leading to study of RPCA as posed in
\eqref{eq:rpcamatrix}.  Early efforts toward solving
\eqref{eq:rpcamatrix} date back to \cite{GK1972,Hub1981,FB1981}, yet
polynomial-time algorithms with strong recovery guarantees did not
emerge until RPCA was formulated as a convex optimization in
\cite{WGR2009,LLS2011,CLM2011}.  Perhaps the most prominent of these
latter approaches is principal component pursuit (PCP) \cite{CLM2011}.
Subsequent work approached the RPCA problem via Bayesian learning
\cite{DHC2011,OMK2016}; developing more efficient algorithmic
frameworks \cite{YPC2016,SLC2018}; establishing recovery theory for
non-convex formulation \cite{NNS2014}; or unfolding RPCA into deep
networks trained on true observation-component pairs
\cite{CLY2021,WDZ2025}.

Yet, the traditional approach to RPCA is driven by matrices, requiring
that some form of matricization be adopted to handle tensor data of
higher orders.
Consequently, a number of efforts have emerged to extend the RPCA
paradigm to process tensor data directly, casting \eqref{eq:rpcamatrix}
into tensor form as
\begin{equation}
  \mathcal{Y} = \mathcal{L} + \mathcal{S} .
  \label{eq:rpcatensor}
\end{equation}
For example, tensor RPCA (TRPCA) \cite{LFC2020} defined a tensor
nuclear norm as the convex envelope of the tensor average rank and
then showed that TRPCA driven by this norm is exactly solvable via
convex optimization under mild conditions, analogous to the matrix
case. Subsequently, enhanced TRPCA (ETRPCA) \cite{GZX2021} replaced
this tensor nuclear norm with a weighted tensor Schatten p-norm to
more flexibly capture salient differences between tensor singular
values. In \cite{XZM2017}, a Kronecker basis representation (KBR) was
employed as a tensor-sparsity measure to reconcile the sparsity
induced by both Tucker and CP decompositions \cite{CC1970,Har1970},
significantly improving performance bounds in the process. Going a
step further, \cite{ZLF2021} strove to recover not only the low-rank
component itself but also its tensor subspace representation, enabling
tensor-based RPCA to be applied to data clustering. Later, tensor correlated
total variation (t-CTV) \cite{WPQ2023} jointly exploited low-rankness
and smoothness priors in multi-dimensional data. To alleviate
sensitivity to rotation as well as to overcome dimension limitations of
tensor singular value decomposition, multiplex transformed tensor
decomposition (MTTD) \cite{FZL2023} decomposed an order-$N$ tensor into
a core order-$N$ tensor coupled with $N$ order-$3$ orthogonal tensors,
chained together by a newly defined mode-$n$ t-product. Targeted at
streaming data, online non-convex TRPCA (ONTRPCA) \cite{FLL2025}
employed a Schatten p-norm for a tighter approximation of the tensor
rank with the Schatten p-norm being updated in an online manner.

On the other hand, many of the most recent efforts for RPCA of tensors
leverage deep networks.  These include \cite{DSD2023}, which employed
a self-supervised deep unfolding model learning with only a few
hyperparameters, requiring no ground truth labeling, and maintaining
robust performance in data-scarce scenarios; and \cite{LZM2022}, which
embedded a deep neural network into the tensor singular value
decomposition so as to capture the underlying low-rank structures of
multi-dimensional images.  Finally, low-rank tensor function
representation (LRTFR) \cite{LZL2024a} used multi-layer perceptrons to
enable continuous data representation beyond discrete grids, offering
powerful representational capabilities.

\subsection{BSTF}
BSTF extends Bayesian sparse matrix factorization
\cite{MS2007,BLM2012} to tensors with the aim of obtaining a
probabilistic factor representation along with automatic rank
determination, thereby providing distributional
estimates for solving inverse problems. To this end, BSTF exploits
sparsity-inducing hierarchical distributions to characterize the
sparsity inherent to tensors. Accordingly, \cite{ZZC2015} developed a
BCP factorization capable of automatic CP rank
determination. This was then generalized by \cite{ZZZ2016} to handle
sparse outliers, effectively yielding a solution for tensor-based RPCA,
while \cite{YYW2025} provided a robust non-negative CP
factorization. Built on a t-product, \cite{ZC2021} enhanced the
preservation of low-rank structure for BSTF-based RPCA models,
while \cite{TCW2023} bridged Tucker and CP decompositions under a
Gaussian-inverse Wishart prior, enabling flexible core-tensor learning
and demonstrating superior rank-determination and recovery
performance. Finally, further studies constructed various formulations
to place sparsity-inducing priors for block-term \cite{GRK2022},
tensor-train \cite{XCW2023}, and tensor-ring \cite{LZL2021}
decompositions to automatically determine corresponding ranks.

\section{The Proposed RPCC Framework}
\label{sec:rpcc}
\subsection{The RPCC Problem}
\begin{definition}[RPCC]
  Suppose
  $\mathcal{L},\mathcal{S}\in\mathbb{R}^{I_1\times
    \cdots\times I_N}$ with $I_n=J_nK_n$, $n\in[N]$. 
  Let $J=\prod_{n}J_n$ and $K=\prod_{n}K_n$. Suppose further that
  $\mathcal{S}$ is blockwisely supported on
  $\mathtt{\Omega}\subset[K]$. Given an observation
  \begin{equation}
    \mathcal{Y}=\mathscr{P}_{\mathtt{\Omega}^\bot}\left[\mathcal{L}\right]+\mathcal{S} ,
    \label{eq:RPCCConstraint}
  \end{equation}
  RPCC seeks to identify
  $\left\{\mathcal{L},\,\mathtt{\Omega}\right\}$, or equivalently
  $\left\{\mathcal{L},\,\mathcal{S}\right\}$, via solving
  \begin{equation}
    \min_{\mathcal{L},\,\mathtt{\Omega}}
    \mathscr{R}\left(\mathcal{L}\right) +
    \mathscr{B}\left(\mathscr{P}_{\mathtt{\Omega}}\left[\mathcal{Y}\right]\right)
    \,\,\text{s.t.}\,\,
    \mathscr{P}_{\mathtt{\Omega}^\bot}\left[\mathcal{Y}\right] =
    \mathscr{P}_{\mathtt{\Omega}^\bot}\left[\mathcal{L}\right] ,
    \label{eq:RPCC}
  \end{equation}
  where
  $\mathscr{R}(\cdot):\mathbb{R}^{I_1 \times\cdots\times I_N}\rightarrow\mathbb{R}$
  and $\mathscr{B}\left(\cdot\right):\mathbb{R}^{I_1\times\cdots\times
    I_N}\rightarrow\mathbb{R}$ quantify the dimensionality and
  blockwise sparsity of a tensor, respectively, and
  $\mathtt{\Omega}\in[K]$ is a blockwise support.
  \label{Def:RPCC}
\end{definition}

\begin{remark}
  To see that this is a ''low-rank/sparse decomposition model,''
  we note that
  $\mathscr{P}_{\mathtt{\Omega}^\bot}\left[\mathcal{Y}\right]=\mathscr{P}_{\mathtt{\Omega}^\bot}\left[\mathcal{L}\right]$
  is equivalent
  to $\mathcal{Y}=\mathscr{P}_{\mathtt{\Omega}^\bot}\left[\mathcal{L}\right]+\mathscr{P}_{\mathtt{\Omega}}\left[\mathcal{Y}\right]$.
\end{remark}
Def.~\ref{Def:RPCC} is constructed in a generalized form, and we
expect the $\mathscr{R}(\cdot)$ operator to encompass most
dimensionality quantizers, including classic matrix rank, tensor
ranks under multiple definitions, factor rank/sparsity from various
matrix/tensor factorizations of $\mathcal{L}$, and, most generally,
maximal coding rate reduction \cite{CYY2022} in the context of deep
learning. On the other hand, $\mathscr{B}(\cdot)$ is a generalized
quantizer of blockwise sparsity, which, in principle, reflects the
cardinality of the blockwise support of a tensor.

We note that, by setting $\mathtt{\Omega}$ to be a blockwise support,
RPCC avoids underfitting the blockwise pattern of real-world data
\cite{ZZW2023}. For example, consider a tensor representing a
red-green-blue (RGB) color image and note that each element in such an
RGB image can never appear in $\mathcal{S}$ alone, since all three
channels of one pixel encode the information of the same spatial
location. Thus, every pixel (which is a block of three elements) is
either included in or excluded from $\mathcal{S}$ as a
whole. Therefore, the support set should separate $\mathcal{L}$ and
$\mathcal{S}$ on a block-by-block basis; i.e., $\mathtt{\Omega}$
should be a blockwise support.

RPCC is distinguished from RPCA primarily through the incorporation of
blockwise support $\mathtt{\Omega}$ for the sparse component
$\mathcal{S}\triangleq\mathscr{P}_{\mathtt{\Omega}}\left[\mathcal{Y}\right]$
in the optimization variables. Though RPCC is an accurate model of
scenarios following \eqref{eq:RPCCConstraint}, optimizing the
support set is NP-hard as it is essentially a combinatorial
optimization problem. Consequently, algorithms to exactly solve
\eqref{eq:RPCC} in polynomial time do not exist. Our task then is to
develop tractable approaches that approximately solve \eqref{eq:RPCC} with
tolerable precision loss. To do so, we resort to BSTF to obtain a
fully probabilistic model that casts support separation into a
binary classification problem.

\subsection{A BCP-Based Generative Model}
\label{sec:Model}
\subsubsection{Randomness}
A probabilistic generative model is usually constructed via exploiting
sparsity-inducing hierarchical distributions to characterize sparsity
in the data. While $\mathcal{S}$ is sparse by definition,
$\mathcal{L}$ is always dense, with latent sparsity emerging only
when it is factored. Due to complex interactions between those
factors, an explicit distribution of $\mathcal{L}$ is
intractable. To compensate, we propose to use instead a simple
probabilistic model consisting of additive noise
$\mathcal{E}\in\mathbb{R}^{I_1\times\cdots\times I_N}$ with entries
drawn i.i.d. from $\mathcal{N}(0,\sigma)$ incorporated into the
observation $\mathcal{Y}$. That is to say, instead of estimating
$\left(\mathcal{L},\,\mathtt{\Omega}\right)$ (or equivalently
$\left(\mathcal{L},\,\mathcal{S}\right)$) from $\mathcal{Y}$
directly, we do so from a noisy version,
\begin{equation}
  \widehat{\mathcal{Y}}\triangleq\mathscr{P}_{\mathtt{\Omega}^\bot}\left[\mathcal{L}\right]+\mathcal{S}+\mathcal{E}=\mathcal{Y}+\mathcal{E}.
  \label{eq:NoisyRPCC}
\end{equation}
We note that a Gaussian noise term is commonly employed in BSTF as a
source of randomness; however, it is assumed that this noise is
inherent in the observation $\mathcal{Y}$, and thus the noise variance
is sought. In contrast, to avoid over-parameterization in RPCC, we
assume the observation $\mathcal{Y}$ is noiseless in line with
conventional RPCA. Then, the incorporation of $\mathcal{E}$ with a
variance $\sigma$ via \eqref{eq:NoisyRPCC} provides the randomness
necessary to VBI without adding an additional variable to be
determined.

\subsubsection{A Mixture Distribution for Support Separation}
Constructing a tractable model to identify the blockwise support of
$\mathcal{S}$ is the key challenge in RPCC. From
\eqref{eq:NoisyRPCC}, we have
\begin{align}
  \widehat{\mathcal{Y}}&=\mathscr{P}_{\mathtt{\Omega}^\bot}\left[\mathcal{L}\right]+\mathscr{P}_{\mathtt{\Omega}}\left[\mathcal{S}\right]+\mathcal{E} \nonumber \\
  &=\mathscr{P}_{\mathtt{\Omega}^\bot}\left[\mathcal{L}+\mathcal{E}\right]+\mathscr{P}_{\mathtt{\Omega}}\left[\mathcal{S}+\mathcal{E}\right] ,
\end{align}
which suggests that blocks of $\widehat{\mathcal{Y}}$ emanate from
one of two sources---$\mathcal{L}+\mathcal{E}$ or
$\mathcal{S}+\mathcal{E}$---in a manner reminiscent of a Gaussian
mixture distribution \cite{Bis2006}. So inspired, we 
model $\widehat{\mathcal{Y}}$ accordingly. That is, for each
vectorized block in $\widehat{\mathcal{Y}}$ (i.e., each column $k$ in
$\widehat{\mathbf{Y}}^{[K]}$, $\forall k\in[K]$), we have
\begin{multline}
  p\left(\widehat{\mathbf{y}}^{[K]}_k\bigm\vert\mathtt{\Theta}_L,\,\mathtt{\Theta}_S,\sigma,\eta\right) = \\
  (1-\eta)p\left(\widehat{\mathbf{y}}^{[K]}_k\bigm\vert\mathtt{\Theta}_L,\sigma\right)
  +\eta p\left(\widehat{\mathbf{y}}^{[K]}_k\bigm\vert\mathtt{\Theta}_S,\sigma\right),
  \label{eq:MD}
\end{multline}
where
$p\left(\widehat{\mathbf{y}}^{[K]}_k\bigm\vert\mathtt{\Theta}_L,\sigma\right)$
and
$p\left(\widehat{\mathbf{y}}^{[K]}_k\bigm\vert\mathtt{\Theta}_S,\sigma\right)$
are the generative models for $\mathcal{L}+\mathcal{E}$ and
$\mathcal{S}+\mathcal{E}$, respectively,
and $\eta\in[0,1]$. According to \cite{Bis2006},
\eqref{eq:MD} can be equivalently generated via a two-level model
\begin{multline}
  p\Big(\widehat{\mathbf{y}}^{[K]}_k\bigm\vert\mathtt{\Theta}_L,\,\mathtt{\Theta}_S,\sigma,z_k\Big) = \\
  p^{1-z_k}\left(\widehat{\mathbf{y}}^{[K]}_k\bigm\vert\mathtt{\Theta}_L,\sigma\right) p^{z_k}\left(\widehat{\mathbf{y}}^{[K]}_k\bigm\vert\mathtt{\Theta}_S,\sigma\right),
\end{multline}
where
\begin{equation}
  p\left(z_k\bigm\vert \eta\right)=\eta^{z_k}(1-\eta)^{1-z_k}
\end{equation}
for
$z_k\in\left\{0,1\right\}.$
Then,
$\widehat{\mathbf{y}}^{[K]}_k\bigm\vert\left\{\mathtt{\Theta}_L,\,\mathtt{\Theta}_S,\sigma,z_k\right\}$
is distributed as either
$p\left(\widehat{\mathbf{y}}^{[K]}_k\bigm\vert\mathtt{\Theta}_L,\sigma\right)$
or
$p\left(\widehat{\mathbf{y}}^{[K]}_k\bigm\vert\mathtt{\Theta}_S,\sigma\right)$
as the Bernoulli variable $z_k$ can be valued either $0$ or $1$. To
achieve tractable posterior inference, a beta distribution is further
assigned to $\eta$ as a conjugate prior; i.e.,
\begin{equation}
  p\left(\eta\bigm\vert\alpha_0,\beta_0\right)\propto \eta^{\alpha_0-1}(1-\eta)^{\beta_0-1}, 
\end{equation}
with $\alpha_0,\beta_0>0$.  As a result, the posterior of $z_k$ is
still Bernoulli distributed, and we are inspired to use $\bar{z}_k
\triangleq E\left[z_k\right]$ to gauge the probability of the
$k^{\text{th}}$ block taking $\mathcal{S}$ as its source. While one
might then imagine that the blockwise support $\mathtt{\Omega}$ would
be estimated by thresholding the $\bar{z}_k$ probabilities for all
blocks, as a key development in this paper, we show below in
Sec.~\ref{sec:convergence} that the RPCC formulation inherently
results in a hard classifier with no need for thresholding.

\subsubsection{The BCP-Based Low-Rank Prior}
We now consider
$p\left(\widehat{\mathbf{y}}^{[K]}_k\bigm\vert\mathtt{\Theta}_L,\sigma\right)$,
the generative model for $\mathcal{L}+\mathcal{E}$ source. 
Under CP factorization, we have
\begin{equation}
  \mathcal{L}=\operatorname{CP}\left(\mathbf{A}^{(1)},\mathbf{A}^{(2)},\dots,\mathbf{A}^{(N)}\right) ,
\end{equation}
where $\bigl\{\mathbf{A}^{(n)}\in\mathbb{R}^{I_n\times
  R}\bigr\}_{n=1}^N$ are the CP factors.
Given that $\mathcal{E}$ is i.i.d.\ Gaussian, the generative model
for the $\mathcal{L}+\mathcal{E}$ source is then
\begin{equation}
\widehat{\mathbf{y}}^{[K]}_k\bigm\vert \{\mathcal{L}, \sigma\} \sim
\widehat{\mathbf{y}}^{[K]}_k\bigm\vert\left\{\{\mathbf{A}^{(n)}\}_{n=1}^N,\sigma\right\}\sim
\mathcal{N}\left({\mathbf{l}}^{[K]}_k,\sigma\mathbf{I}_{J}\right)
  \nonumber
\end{equation}
where we note that ${\mathbf{l}}^{[K]}_k$ is the $k^{\text{th}}$ column of $\mathbf{L}^{[K]}$. 
Adopting the CP rank as $\mathscr{R}(\cdot)$ in \eqref{eq:RPCC}
results in sparsity in
$\bigl\{\mathbf{A}^{(n)}\bigr\}_{n=1}^N$ once these CP factors are
over-parameterized with a large $R$. To capture such sparsity, a
Gaussian-Gamma conjugate distribution pair is assigned to
$\bigl\{\mathbf{A}^{(n)}\bigr\}_{n=1}^N$,
\begin{align}
    p\left(\mathbf{A}^{(n)}\bigm\vert \boldsymbol{\lambda}\right)&=\prod_{i_n}p\left(\mathbf{a}^{(n)}_{i_n}\bigm\vert \boldsymbol{\lambda}\right), \\
    \mathbf{a}^{(n)}_{i_n}\bigm\vert\boldsymbol{\lambda}&\sim\mathcal{N}\left(\mathbf{0},\boldsymbol{\lambda}^{-1}\right), \\
    p\left(\boldsymbol{\lambda}\bigm\vert a_0,b_0\right)&\propto\prod_{r=1}^{R}\lambda_r^{a_0-1}e^{-b_0\lambda_r}, 
\end{align}
where $a_0,b_0>0$, 
$\boldsymbol{\lambda}=\begin{bmatrix} \lambda_1 & \cdots & \lambda_R\end{bmatrix}$
is the precision vector, and
$\mathbf{\Lambda}=\operatorname{diag}\left(\boldsymbol{\lambda}\right)$ is
the inverse of the covariance matrix.

\subsubsection{The Blockwise-Sparsity Prior}
To model the
$p\left(\widehat{\mathbf{y}}^{[K]}_k\bigm\vert\mathtt{\Theta}_S,\sigma\right)$
source, 
\cite{ZZZ2016,ZC2021} assume that each
element in $\mathcal{S}$ follows a univariate Gaussian distribution
with its own variance. However, this element-wise model underfits the
blockwise sparsity here. Accordingly, we propose to use a single variance
$\tau_k$
for all elements in block~$k$. Therefore, each block
converges to zero or not as a whole. We thus have
\begin{equation}
    \mathbf{s}^{[K]}_k\bigm\vert\tau_k\sim\mathcal{N}\left(\mathbf{0},\tau_k\mathbf{I}_{J}\right)
\end{equation}
where we note again that ${\mathbf{s}}^{[K]}_k$ is the $k^{\text{th}}$ column of $\mathbf{S}^{[K]}$.
Since $\mathcal{E}$ is independent from $\mathcal{S}$, we
immediately obtain the generative model for the $\mathcal{S}+\mathcal{E}$ source as
\begin{equation}
  \widehat{\mathbf{y}}^{[K]}_k\bigm\vert\widehat{\tau}_k\sim\mathcal{N}\left(\mathbf{0},\widehat{\tau}_k\mathbf{I}_{J}\right),
\end{equation}
where $\widehat{\tau}_k=\tau_k+\sigma$.
While an inverse gamma hyperprior on $\widehat{\tau}_k$ is needed
for tractable inference, we have necessarily that
$\widehat{\tau}_k>\sigma$ and so employ
a truncated inverse-gamma (TIG) distribution;
i.e., for $\forall k\in[K]$,
\begin{equation}
  p\left(\widehat{\tau}_k\bigm\vert c_0,d_0,\sigma\right)\propto\widehat{\tau}_k^{-c_0-1}e^{-d_0/\widehat{\tau}_k}
\end{equation}
with  $c_0,d_0>0$ and $\widehat{\tau}_k>\sigma$.

\subsubsection{Overall Generative Model}
We designate the collection of all the inferable latent variables as
\begin{equation}
  \mathbf{\Theta}=\left\{\bigl\{\mathbf{A}^{(n)}\bigr\}_{n=1}^N,\boldsymbol{\lambda},\left\{\widehat{\tau}_k\right\}_{k=1}^K,\left\{z_k\right\}_{k=1}^K,\eta\right\} 
\end{equation} 
such that the resulting joint distribution is
\begin{multline}
  p\left(\widehat{\mathcal{Y}},\mathbf{\Theta}
  \bigm\vert a_0,b_0,c_0,d_0,\alpha_0,\beta_0,\sigma\right)=\\
  \prod_{k} p^{1-z_k}\left(\widehat{\mathbf{y}}^{[K]}_k\bigm\vert\left\{\mathbf{A}^{(n)}\right\}_{n=1}^N,\sigma\right) p^{z_k}\left(\widehat{\mathbf{y}}^{[K]}_k\bigm\vert\widehat{\tau}_k\right) \times \\ \prod_{n} p\left(\mathbf{A}^{(n)}\bigm\vert\boldsymbol{\lambda}\right) p\left(\boldsymbol{\lambda}\bigm\vert a_0,b_0\right) p\left(\widehat{\tau}_k\bigm\vert c_0,d_0,\sigma\right)\times \\ p\left(z_k\bigm\vert\eta\right)p\left(\eta\bigm\vert\alpha_0,\beta_0\right) .
  \label{eq:joint}
\end{multline}
The logarithm of \eqref{eq:joint}
is derived as (A.1) in App.~A.

\subsection{Inference and Posterior Distributions}
\label{sec:inference}
For the generative model of \eqref{eq:joint}, we introduce
variational posterior distribution
$q\left(\mathbf{\Theta}\right)$
as an approximation to the true joint posterior distribution
$p\left(\mathbf{\Theta} \bigm\vert \widehat{\mathcal{Y}}\right)$, 
and, adopting mean-field approximation, assume $q\left(\mathbf{\Theta}\right)$
can be factored as
\begin{equation}
  q\left(\mathbf{\Theta}\right) = 
  \prod_{n, i_n, k} q_{\mathbf{a}^{(n)}_{i_n}}\left(\mathbf{a}_{i_n}^{(n)}\right)q_{\boldsymbol{\lambda}}\left(\boldsymbol{\lambda}\right) q_{\widehat{\tau}_k}\left(\widehat{\tau}_k\right) q_{z_k}\left(z_k\right)q_{\eta}(\eta).
\end{equation}
Accordingly,
the optimized functional form for the $m^{\text{th}}$ variable of
$\mathbf{\Theta}$ is
\begin{equation}
  \ln q\left(\mathbf{\Theta}_m\right)=E_{q\left(\mathbf{\Theta}\backslash\mathbf{\Theta}_m\right)}\left[\ln p\left(\widehat{\mathcal{Y}},\mathbf{\Theta}\bigm\vert-\right)\right] + \text{const}
  \label{eq:Inference Rule}
\end{equation}
where $q\left(\mathbf{\Theta}\backslash\mathbf{\Theta}_m\right)$
denotes the joint distribution of all variables except variable
$\mathbf{\Theta}_m$. Below, we present only
the most salient results, deferring details to App.~B.

We start first with the variational posterior for the CP factors.
Specifically, by
applying \eqref{eq:Inference Rule} to $\mathbf{a}^{(n)}_{i_n}$,
we find that $q_{\mathbf{a}^{(n)}_{i_n}}\left(\mathbf{a}_{i_n}^{(n)}\right)$
remains Gaussian. That is
\begin{equation}
  \mathbf{a}_{i_n}^{(n)}\sim\mathcal{N}\left(\mathbf{\mu}^{(n)}_{i_n},\mathbf{\Sigma}_{i_n}^{(n)}\right)
\end{equation}
with 
covariance
\begin{multline}
  \mathbf{\Sigma}_{i_n}^{(n)}=\\\Bigg(\hspace*{-0.25em}\sum_{\substack{k\in\mathbf{K}_{i_n} \\ (i_1 \cdots i_N)\in\mathtt{\Omega}_k}}\hspace*{-1em}\frac{1-E[z_k]}{\sigma}
  \mathop{\scalebox{2}{$\odot$}}_{n'\neq n}E\left[\mathbf{a}_{i_{n'}}^{(n')}\left(\mathbf{a}_{i_{n'}}^{(n')}\right)^T\right] + 
  E\left[\mathbf{\Lambda}\right]\Bigg)^{-1}
  \label{eq:FactorUp1}
\end{multline}
and mean
\begin{equation}
  \mathbf{\mu}^{(n)}_{i_n} =
  \mathbf{\Sigma}_{i_n}^{(n)}\hspace*{-1em}\sum_{\substack{k\in\mathbf{K}_{i_n} \\ (i_1 \cdots i_N)\in\mathtt{\Omega}_k}}\hspace*{-1em}\frac{1-E[z_k]}{\sigma}\widehat{\mathcal{Y}}_{i_1\cdots i_N}\mathop{\scalebox{2}{$\odot$}}_{n'\neq n}E\left[\mathbf{a}_{i_{n'}}^{(n')}\right] ,
  \label{eq:FactorUp2}
\end{equation}
where
$\mathbf{K}_{i_n}=\left\{k\bigm\vert(i_1\cdots i_N)\in\mathtt{\Omega}_k\right\}$,
and $\mathtt{\Omega}_k$ is defined as in Def.~\ref{Def:BProj}.
See App.~B.1 for the complete derivation.

Next, we consider $\boldsymbol{\lambda}$ which is updated by receiving
messages from CP factors and its own hyperprior. From
\eqref{eq:Inference Rule}, we obtain
\begin{equation}
  q_{\boldsymbol{\lambda}}\left(\boldsymbol{\lambda}\right)\propto\prod_{r=1}^{R}\lambda_r^{a_r-1}e^{-b_r\lambda_r} ,
\end{equation}
where
\begin{align}
  a_r &= a_0+\frac{1}{2}\sum_n I_n ,
  \label{eq:arupdate}
  \\
  b_r&=b_0+\frac{1}{2}\sum_nE\left[\left(\mathbf{a}_r^{(n)}\right)^T\mathbf{a}_r^{(n)}\right].
  \label{eq:brupdate}
\end{align}
Hence, $\boldsymbol{\lambda}$ remains gamma distributed in the posterior;
see App.~B.2 for details.

Collect the $\widehat{\tau}_k$ variables as vector
$\widehat{\boldsymbol{\tau}}\triangleq
\begin{bmatrix} \widehat{{\tau}}_1 &\cdots &\widehat{{\tau}}_K\end{bmatrix}$.
Letting $\mathbf{\Theta}_m=\widehat{\boldsymbol{\tau}}$ in
\eqref{eq:Inference Rule}, we find that the
$\widehat{\tau}_k$ follow independent TIG
distributions; i.e.,
\begin{equation}
  q_{\widehat{\boldsymbol{\tau}}}\left(\widehat{\boldsymbol{\tau}}\right)\propto\prod_k \widehat{{\tau}}_k^{-c_k-1}e^{-d_k/\widehat{{\tau}}_k}, \quad \text{s.t.}\,\,\widehat{{\tau}}_k>\sigma,
\end{equation}
where
\begin{align}
  c_k&=c_0+\frac{1}{2}JE[z_k],
  \label{eq:ckupdate}
  \\
  d_k&=d_0+\frac{1}{2}E[z_k]\left(\widehat{\mathbf{y}}^{[K]}_k\right)^T\widehat{\mathbf{y}}^{[K]}_k ,
  \label{eq:dkupdate}
\end{align}
with full details in App.~B.3.

The $z_k$ variables are critical to separating the two mixture sources.
Collecting them as
$\mathbf{z}\triangleq[z_1,\cdots,z_K]$, and
denoting $\bar{z}_k = E[z_k]$, we have
\begin{equation}
  q_{{\mathbf{z}}}\left({\mathbf{z}}\right)=\prod_k\bar{z}_k^{z_k}(1-\bar{z}_k)^{1-z_k}, 
\end{equation}
where
\begin{equation}
  \bar{z}_k=\frac{1}{1+\exp(\theta_k^{(2)}-\theta_k^{(1)})} ,
  \label{eq:zkupdate}
\end{equation}
and
\begin{multline}
  \theta_k^{(1)}=-\frac{1}{2}\left(\widehat{\mathbf{y}}^{[K]}_k\right)^T\widehat{\mathbf{y}}^{[K]}_k E\left[\frac{1}{\widehat{\tau}_k}\right]-\frac{J}{2}\ln2\pi
  \\
  -\frac{J}{2}E[\ln\widehat{\tau}_k] + E[\ln \eta],
  \label{eq:theta1update}
\end{multline}
\begin{multline}
    \theta_k^{(2)}=-\frac{1}{2\sigma}\biggl(\bigl(\widehat{\mathbf{y}}^{[K]}_k\bigr)^T\widehat{\mathbf{y}}^{[K]}_k-2\bigl(\widehat{\mathbf{y}}^{[K]}_k\bigr)^TE\bigl[{\mathbf{l}}^{[K]}_k\bigr]+\\E\bigl[\bigl(\widehat{\mathbf{l}}^{[K]}_k\bigr)^T\widehat{\mathbf{l}}^{[K]}_k\bigr]\biggr)\\-\frac{J}{2}\ln2\pi\sigma+E[\ln(1-\eta)] ,
  \label{eq:theta2update}
\end{multline}
with details in App.~B.4.
We note that, in \eqref{eq:theta2update}, 
$E\bigl[\bigl(\widehat{\mathbf{l}}^{[K]}_k\bigr)^T\widehat{\mathbf{l}}^{[K]}_k\bigr]$
is given directly by \cite[Thm. 3.3]{ZZC2015}. Thus, the primary
computational challenge lies with \eqref{eq:theta1update},
namely $E[\ln\widehat{\tau}_k]$ and
$E\left[{\widehat{\tau}_k}^{-1}\right]$. For a TIG distribution with
threshold $\sigma$, evaluating these expectations precisely is often
numerically unstable and computationally expensive. In RPCC, however,
$\sigma$ is specified manually. This allows us to leverage
the asymptotic regime where $\sigma \to 0$, leading to the
stable approximations,
\begin{align}
    E[\ln\widehat{\tau}_k]&\overset{\sigma\rightarrow 0}{\approx}\ln d_k-\psi(c_k),\\
    E\left[\frac{1}{\widehat{\tau}_k}\right]&\overset{\sigma\rightarrow 0}{\approx}\frac{c_k}{d_k} .
\end{align}

Finally, considering $\eta$, we find in App.~B.5
that a beta distribution still holds with
\begin{align}
  q_{\eta}(\eta)&\propto\eta^{\alpha-1}(1-\eta)^{\beta-1} , \\
  \alpha&=\alpha_0+\sum_{k}\bar{z}_k,
  \label{eq:alphaupdate}
  \\
  \beta&=\beta_0+K-\sum_{k}\bar{z}_k .
  \label{eq:betaupdate}
\end{align}

\subsection{Convergence Analysis---A Hard Classifier}
\label{sec:convergence}
The inference steps above follow a typical VBI framework whose
convergence is well understood \cite{Bis2006}: we proceed to
update each posterior distribution in turn by
maximizing the evidence lower bound (ELBO) until it reaches the
stationary point. We derive the ELBO, denoted by
$\mathscr{L}\left(\widehat{\mathcal{Y}},\mathbf{\Theta}\mid-\right)$,
as (B.25) in App.~B.6.

While convergence of the overall VBI framework is not
in question,  what is of interest is
the convergence behavior of $\bar{z}_k$, since $\bar{z}_k=E[z_k]$ provides
the solution to the key challenge in RPCC---namely, the partitioning
of the support. To
this end, we extract the $\bar{z}_k$-related terms from
$\mathscr{L}\left(\widehat{\mathcal{Y}},\mathbf{\Theta}\mid-\right)$,
which we denote as
\begin{equation}
  \mathscr{L}_{k}\left(\bar{z}_k\right)\triangleq\bar{z}_k\left(A_k-B_k\right)+h(\bar{z}_k) ,
\end{equation}
where
\begin{align}
  A_k&= -\frac{1}{2}\left(\widehat{\mathbf{y}}^{[K]}_k\right)^T\widehat{\mathbf{y}}^{[K]}_k E\left[\frac{1}{\widehat{\tau}_k}\right]-\frac{J}{2}\ln2\pi
  \nonumber\\
  &\quad-\frac{J}{2}E[\ln\widehat{\tau}_k]+E[\ln \bar{z}_k],\\
  B_k&= -\frac{1}{2\sigma} E\left[ \left(\widehat{\mathbf{y}}^{[K]}_k - \mathbf{l}^{[K]}_k\right)^T \left(\widehat{\mathbf{y}}^{[K]}_k - \mathbf{l}^{[K]}_k\right) \right] \nonumber\\
  &\quad - \frac{J}{2} \ln(2\pi \sigma)+E[\ln(1-\bar{z}_k)], \\
  h(\bar{z}_k)&\triangleq H\left(q_{z_k}\left(z_k\right)\right)=-\bar{z}_k\ln\bar{z}_k-(1-\bar{z}_k)\ln(1-\bar{z}_k) .
\end{align}
\begin{proposition}
  \textit{Suppose that $\widehat{\mathcal{Y}}$ is bounded and let
    $\xi_0$ denote the stationary point of
    $\mathscr{L}_{k}(\xi)$. Then
    $\lim_{\sigma\rightarrow0}\xi_0(1-\xi_0)=0.$}
  \label{Prop:HardClass}
\end{proposition}
\begin{proof}
  It follows from \eqref{eq:ckupdate}, \eqref{eq:dkupdate},
  \eqref{eq:alphaupdate}, and \eqref{eq:betaupdate} that
    \begin{gather}
      0<c_0\leq c_k\leq c_0+\frac{J}{2} , \\
      0<d_0\leq d_k\leq d_0+\frac{1}{2}\left(\widehat{\mathbf{y}}^{[K]}_k\right)^T\widehat{\mathbf{y}}^{[K]}_k ,\\
      0<\alpha_0\leq \alpha\leq \alpha_0+K , \\
      0<\beta_0\leq \beta\leq \beta_0+K .
    \end{gather}
    Thus, provided that $\widehat{\mathcal{Y}}$ is bounded, $c_k$,
    $d_k$, $\alpha$, and $\beta$ are all bounded from both sides,
    resulting in $E\left[\frac{1}{\widehat{\tau}_k}\right]$,
    $E[\ln\widehat{\tau}_k]$, $E[\ln\eta]$, $E[\ln(1-\eta)]$, and
    $A_k$ being bounded as well. Denoting $\Delta_k=A_k-B_k$, we have
    $\mathscr{L}_{k}(\xi)=\xi\Delta_k+h(\xi)$.  Since $h(\xi)$ is
    concave, so then is $\mathscr{L}_k(\xi)$. Thus, $\xi_0$ is
    necessarily the stationary point of $\mathscr{L}_k(\xi)$. Setting
    \begin{equation}
      \frac{d}{d\xi} \mathscr{L}_k(\xi) =
      \Delta_k + \ln\left(\frac{1-\xi}{\xi}\right) 
      \end{equation}
    to zero and solving for $\xi_0$ yields
    \begin{equation}
      \xi_0 = \frac{1}{1 + e^{-\Delta_k}} .
      \label{eq:x0}
    \end{equation}
    If $E\Bigl[ \left(\widehat{\mathbf{y}}^{[K]}_k -
      \mathbf{l}^{[K]}_k\right)^T
      \left(\widehat{\mathbf{y}}^{[K]}_k -
      \mathbf{l}^{[K]}_k\right) \Bigr] = o(\sigma),$ then
    $B_k\overset{\sigma\rightarrow0}{\rightarrow}\infty$ and
    $\Delta_k\overset{\sigma\rightarrow0}{\rightarrow}-\infty$.  From
    \eqref{eq:x0}, we see then that, in this case,
    \begin{equation}
      \lim_{\sigma\rightarrow0} \xi_0 =
      \lim_{\Delta_k \rightarrow -\infty} \frac{1}{1 + e^{-\Delta_k}} =
      0 .
      \label{eq:x00}
    \end{equation}
    Otherwise, we have
    $B_k\overset{\sigma\rightarrow0}{\rightarrow}-\infty$,
    $\Delta_k\overset{\sigma\rightarrow0}{\rightarrow}\infty$, and
    \begin{equation}
      \lim_{\sigma\rightarrow0}\xi_0 =
      \lim_{\Delta_k \rightarrow \infty} \frac{1}{1 + e^{-\Delta_k}} =
      1.
      \label{eq:x01}
    \end{equation}
\end{proof}

Interestingly, we observe that the proposed generative RPCC model,
although originally conceived as a soft classifier akin to most RPCA
formulations, exhibits a remarkable and advantageous property. As
established by Prop.~\ref{Prop:HardClass}, the variational parameter
$\bar{z}_k$, which represents the probability of the $k^{\text{th}}$
data block belonging to the outlier component $\mathcal{S}$, converges
unequivocally to either $0$ or $1$ as we take $\sigma$ toward 0.  This
convergence effectively transforms the model into a hard classifier,
yielding almost sure decisions for support separation.  Thus, the
proposed model eliminates the dependency on a post-hoc thresholding
step typically required by soft classifiers in RPCA formulations. Such
a threshold is often difficult to determine in practice, as prior
knowledge is virtually non-existent, and its selection can be as
challenging as designing the classifier itself. In contrast, our
model, by converging to definitive assignments, bypasses this
hurdle entirely.

\subsection{The BCP-RPCC Algorithm}
\label{sec:algorithm}
\begin{algorithm}[t]
  \renewcommand{\algorithmicrequire}{\textbf{Input:}}
  \renewcommand{\algorithmicensure}{\textbf{Output:}}
  \caption{The BCP-RPCC Algorithm}
  \label{alg:BCP-RPCC}
  \begin{algorithmic}[1]
    \REQUIRE Observation
    $\mathcal{Y}\in\mathbb{R}^{I_1\times\cdots\times I_N}$;
    CP rank $R$;
    noise variance $\sigma$;
    blockwise parameters $\left\{J_n\right\}^N_{n=1}$ and
    $\left\{K_n\right\}^N_{n=1}$; and
    hyperprior parameters $a_0,b_0,c_0,d_0,\alpha_0,\beta_0$
    \ENSURE
    $\widehat{\mathcal{L}}$, $\widehat{\mathcal{S}}$, and
    $\widehat{\mathtt{\Omega}}$
    \STATE Draw an $\mathcal{E}$ i.i.d. from $\mathcal{N}(0,\sigma)$
    \STATE$\widehat{\mathcal{Y}}\leftarrow\mathcal{Y}+\mathcal{E}.$
    \STATE $\bigl\{\bar{\mathbf{A}}^{(1)},\dots,\bar{\mathbf{A}}^{(N)}\bigr\}\leftarrow$ \\\hfill$\arg\min_{\left\{\bar{\mathbf{A}}^{(1)},\ldots,\bar{\mathbf{A}}^{(N)}\right\}}\bigl\|\widehat{\mathcal{Y}}-\operatorname{CP}\bigl\{\bar{\mathbf{A}}^{(1)},\dots,\bar{\mathbf{A}}^{(N)}\bigr\}\bigr\|_F$
    \STATE$K\leftarrow\prod_{n}K_n$
    \STATE Initialization: $\forall n \in [N], \forall r \in [R], \forall k \in [K], \forall i_n \in [I_n]$:
    \begin{gather*}
      \mathbf{\mu}^{(n)}_{i_n}\leftarrow\bar{\mathbf{a}}^{(n)}_{i_n}, \,\,
      \mathbf{\Sigma}_{i_n}^{(n)}\leftarrow\mathbf{I}_R, \,\,
      a_r\leftarrow a_0, \,\,
      b_r\leftarrow b_0 , \\
      c_k\leftarrow c_0, \,\,
      d_k\leftarrow d_0, \,\,
      \bar{z}_k\leftarrow z_0, \,\,
      \alpha\leftarrow\alpha_0, \,\,
      \beta\leftarrow\beta_0
    \end{gather*}
    \WHILE{ELBO not converged}
    \FOR{$n=1:N$}
    \FOR{$i_n=1:I_n$}
    \STATE Update
    $\mathbf{\mu}^{(n)}_{i_n}$, $\mathbf{\Sigma}_{i_n}^{(n)}$
    via \eqref{eq:FactorUp1}, \eqref{eq:FactorUp2}
    \ENDFOR
    \ENDFOR
    \FOR{$r=1:R$}
    \STATE Update
    $a_r$, $b_r$ via \eqref{eq:arupdate}, \eqref{eq:brupdate}
    \ENDFOR
    \FOR{$k=1:K$}
    \STATE Update $c_k$, $d_k$ via \eqref{eq:ckupdate}, \eqref{eq:dkupdate}
    \STATE Update $\bar{z}_k$ via \eqref{eq:zkupdate}
    \STATE Update
    $\alpha$, $\beta$ via \eqref{eq:alphaupdate}, \eqref{eq:betaupdate}
    \ENDFOR
    \ENDWHILE
    \STATE
    $\mathbf{a}^{(n)}_{i_n}\leftarrow\mathbf{\mu}^{(n)}_{i_n}$,
    $\forall\,i_n\in[I_n]$, $\forall\,n\in[N]$
    \STATE
    $\widehat{\mathcal{L}}\leftarrow\operatorname{CP}\bigl\{\mathbf{A}^{(1)},\dots,\mathbf{A}^{(N)}\bigr\}$
    \STATE$\widehat{\mathtt{\Omega}}\leftarrow\left\{k|\bar{z}_k>0\right\}$
    \label{step:support}
    \STATE$\widehat{\mathcal{S}}\leftarrow
    \mathscr{P}_{\widehat{\mathtt{\Omega}}}\left[\mathcal{Y}\right]$
  \end{algorithmic}
\end{algorithm}

Applying VBI to the posteriors derived in Sec.~\ref{sec:inference} results in an algorithmic solution to the RPCC problem of
\eqref{eq:RPCC} as detailed in Alg.~\ref{alg:BCP-RPCC}.  Given its
reliance on a BCP-based formulation of low rank, we refer to the
algorithm as BCP-RPCC. We note that, due to the hard classifier
provided by $\bar{z}_k$ as discussed in
Sec.~\ref{sec:convergence}, Step~\ref{step:support} of
Alg.~\ref{alg:BCP-RPCC} determines the support
$\widehat{\mathtt{\Omega}}$ of $\widehat{\mathcal{S}}$ without
thresholding.

\subsection{Complexity Analysis}
In Alg.~\ref{alg:BCP-RPCC}, the primary computational burden for
updating the posterior distributions of the CP factors lies in the
matrix inverse in \eqref{eq:FactorUp1}
 which has a complexity of
$O\bigl(R^2(N+K)N+R^3\sum_nK_n\bigr)$. The complexity for
updating $q_{\boldsymbol{\lambda}}(\boldsymbol{\lambda})$ fully lies
in the calculation of $b_r$ via \eqref{eq:brupdate} which is
$O\left(R^2\sum_nI_n\right)$. Updating
$q_{\widehat{\boldsymbol{\tau}}}(\widehat{\boldsymbol{\tau}})$ via
\eqref{eq:ckupdate} and \eqref{eq:dkupdate} is
quite easy and involves only $O(K)$ multiplications. During the updating of
$q_{{\mathbf{z}}}({\mathbf{z}})$ via \eqref{eq:zkupdate}, the calculation of
$E\bigl[\bigl(\widehat{\mathbf{l}}^{[K]}_k\bigr)^T\widehat{\mathbf{l}}^{[K]}_k\bigr]$ in \eqref{eq:theta2update}
costs no more than $O\left(R^2J\right)$ complexity, and the computational
cost for updating $q_\eta(\eta)$ via \eqref{eq:alphaupdate} and
\eqref{eq:betaupdate} is negligible. Thus, the proposed
BCP-RPCC algorithm has polynomial complexity; specifically,
\begin{equation}
  O\Bigl(R^2\bigl(N^2+KN+\sum_nI_n+J\bigr)+R^3\sum_nK_n+K\Bigr).
\end{equation}

\section{Experimental Evaluation}
\label{sec:results}
We now gauge the performance of the BCP-RPCC algorithm relative to a
number of other RPCA-based techniques.  To do so, we first consider
performance on an ideally configured synthetic dataset with known
ground truth. We then turn attention to two real-world tasks that have
been widely used to demonstrate RPCA performance, namely,
foreground-background extraction from motion video (e.g.,
\cite{BZ2014,CWS2016,AJ2025,CLM2011,LFC2020,ZC2021,WPQ2023,FZL2023,FLL2025,LZL2024a})
and anomaly detection within hyperspectral imagery (e.g.,
\cite{KZL2017,LLD2021,LLQ2022,WWH2023,WLG2025}).  All experiments are
implemented on an Intel\textsuperscript{\textregistered}
Core\textsuperscript{TM} i9-13900K CPU @ 3.00 GHz with 128-GB RAM and
an NVIDIA\textsuperscript{\textregistered} RTX 4090 GPU with 24-GB of
memory. We initialize the hyperprior parameters in Alg.~\ref{alg:BCP-RPCC}
as $a_0=b_0=c_0=d_0=10^{-5}$, $\alpha_0=\beta_0=1$ and $z_0=0.5$ to be non-informative.

\subsection{Identifiability on Synthetic Data}
\label{sec:resultssynthetic}
\begin{figure}[t]
  \centering
  \setlength{\tabcolsep}{0.2mm}
  \begin{tabular}{m{0.16\linewidth}m{0.16\linewidth}m{0.16\linewidth}m{0.16\linewidth}m{0.16\linewidth}m{0.16\linewidth}}
    \multicolumn{2}{c}{\includegraphics[width=0.32\linewidth]{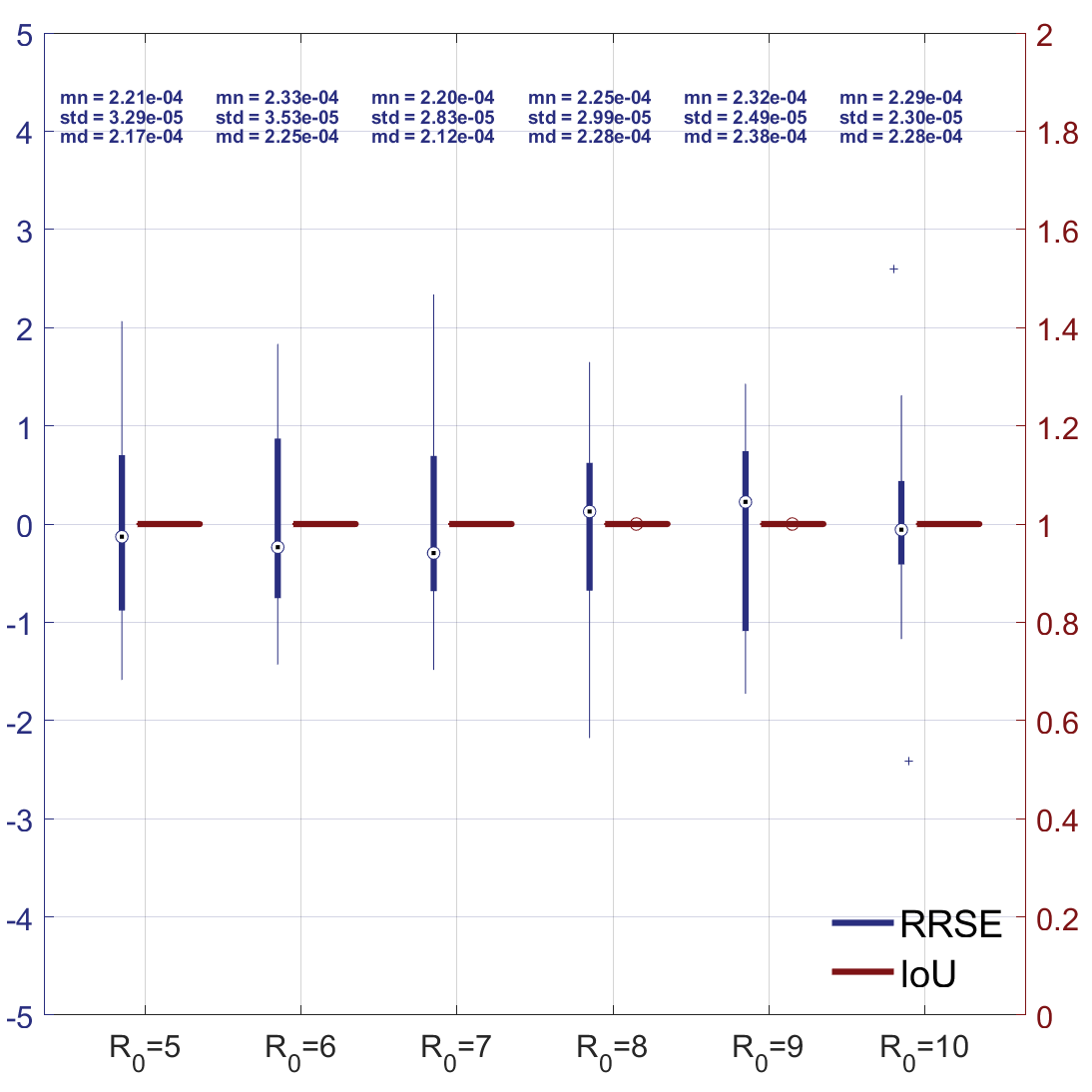}}&\multicolumn{2}{c}{\includegraphics[width=0.32\linewidth]{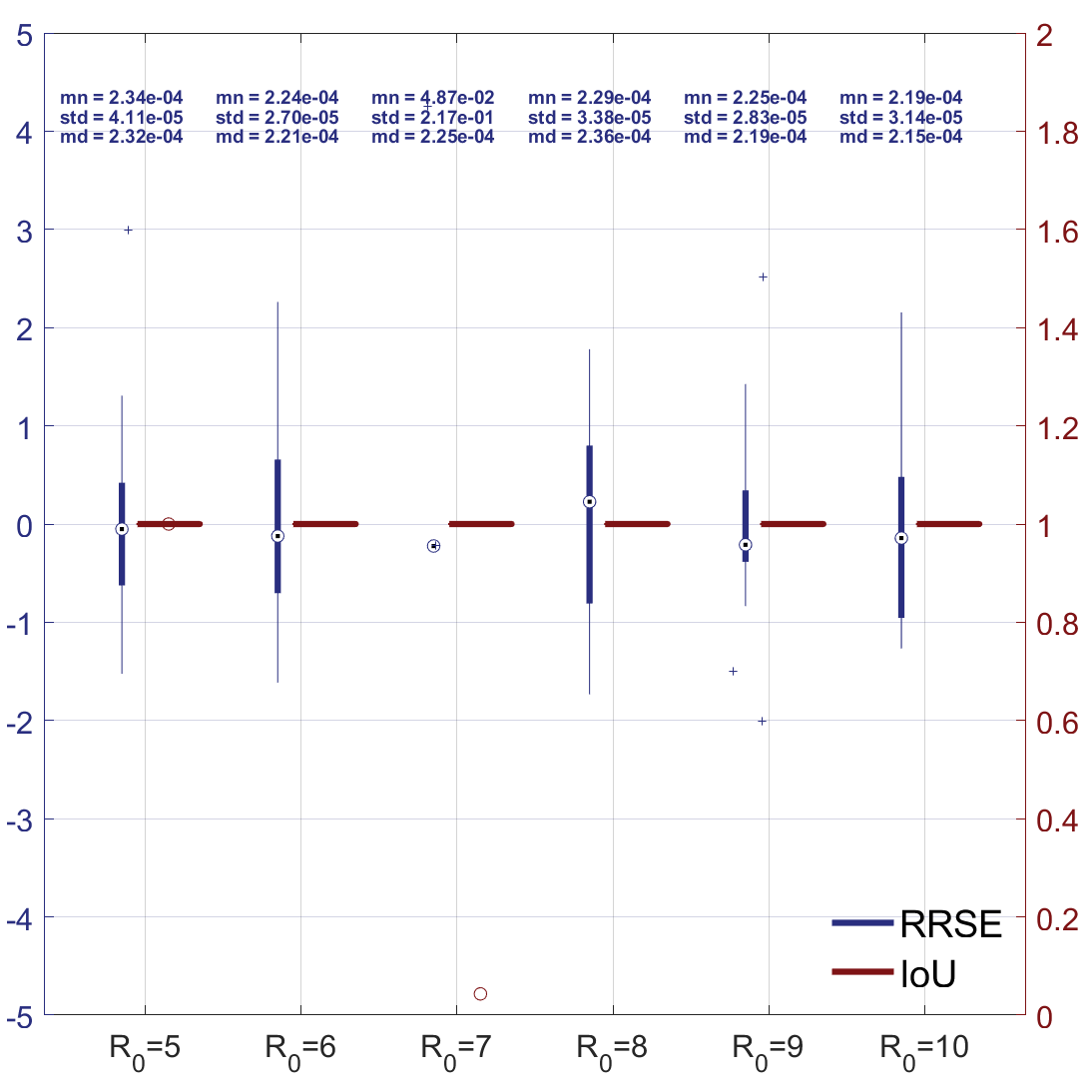}}&\multicolumn{2}{c}{\includegraphics[width=0.32\linewidth]{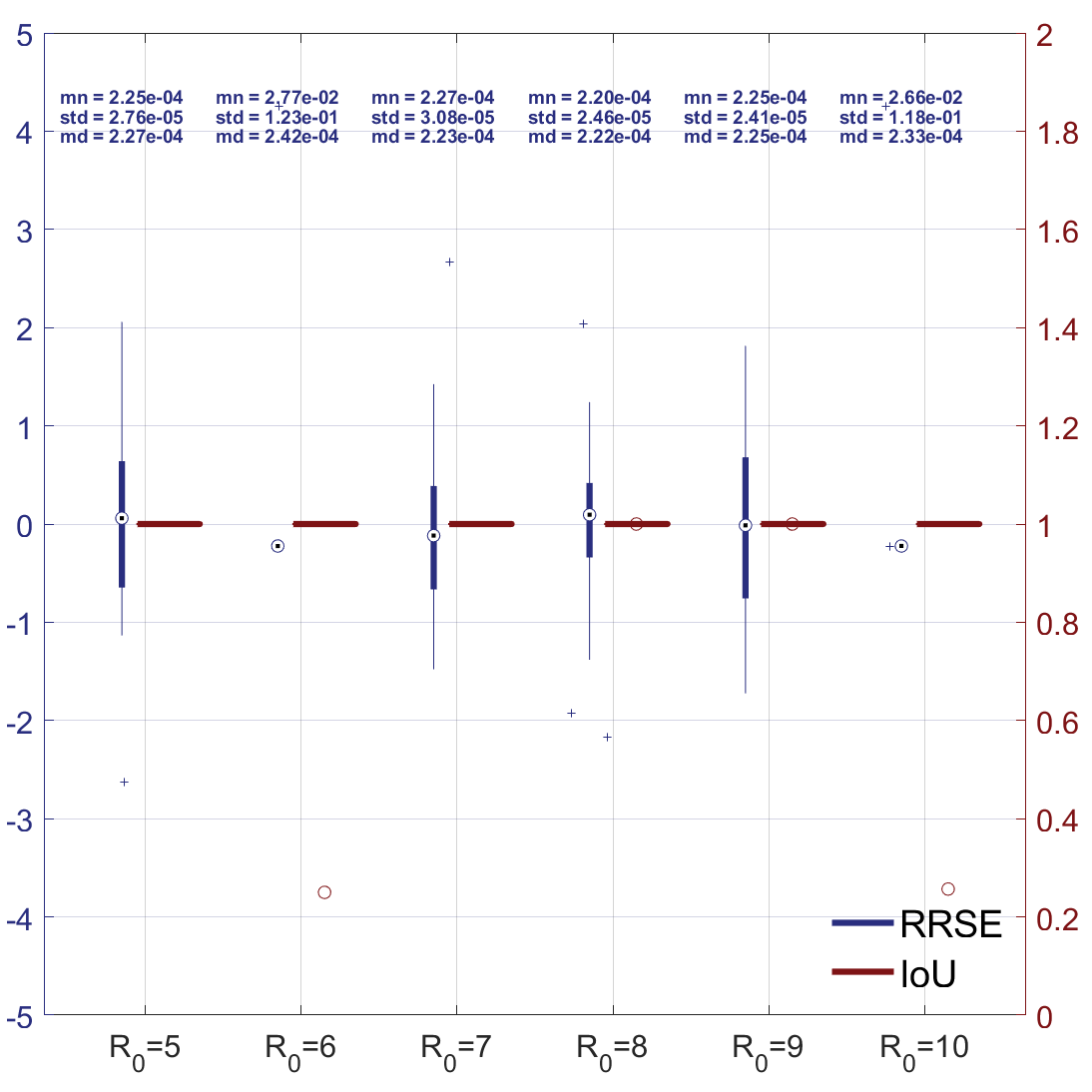}}\\
    \multicolumn{2}{c}{$\rho=0.02$}&\multicolumn{2}{c}{$\rho=0.04$}&\multicolumn{2}{c}{$\rho=0.06$}\\
    &\multicolumn{2}{c}{\includegraphics[width=0.32\linewidth]{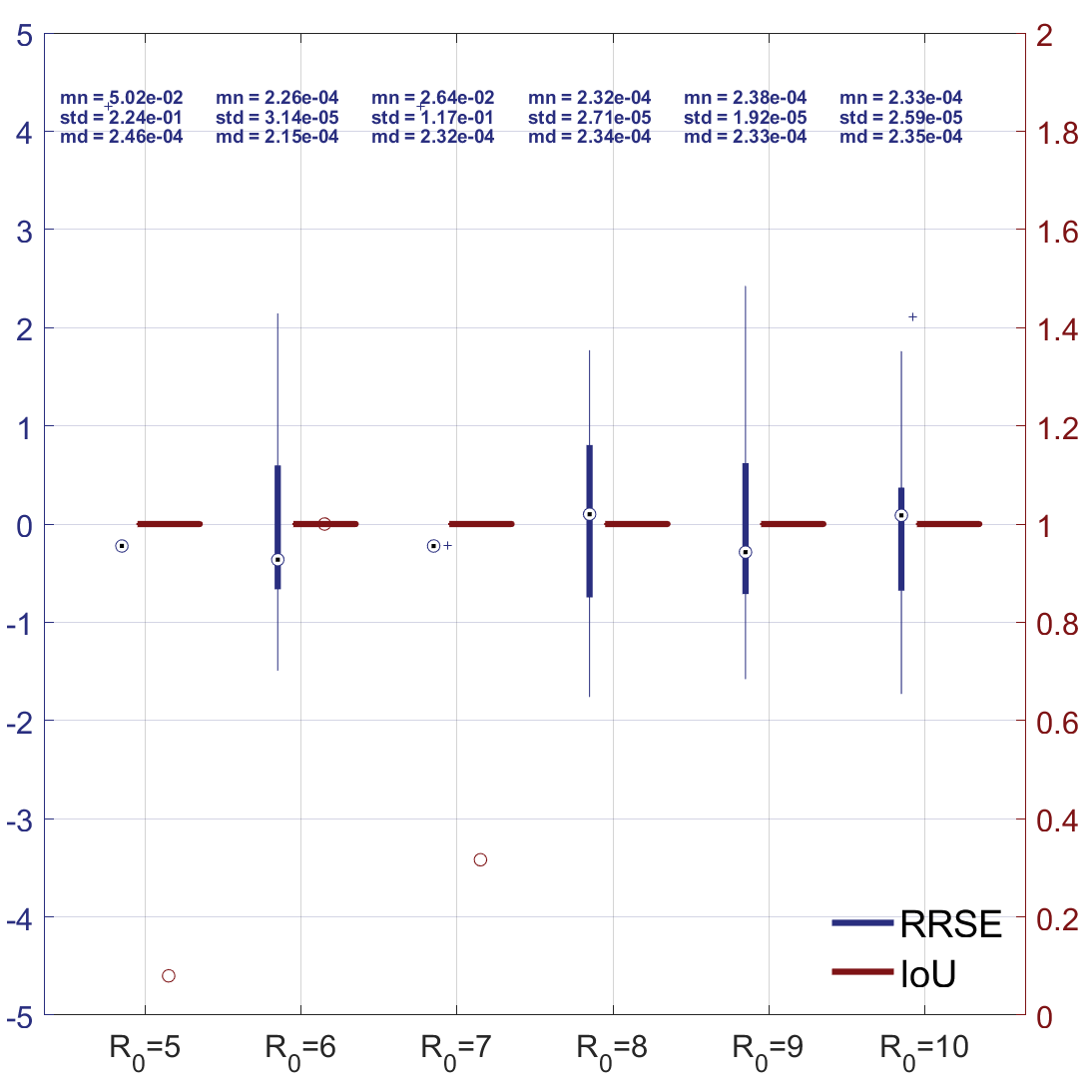}}&\multicolumn{2}{c}{\includegraphics[width=0.32\linewidth]{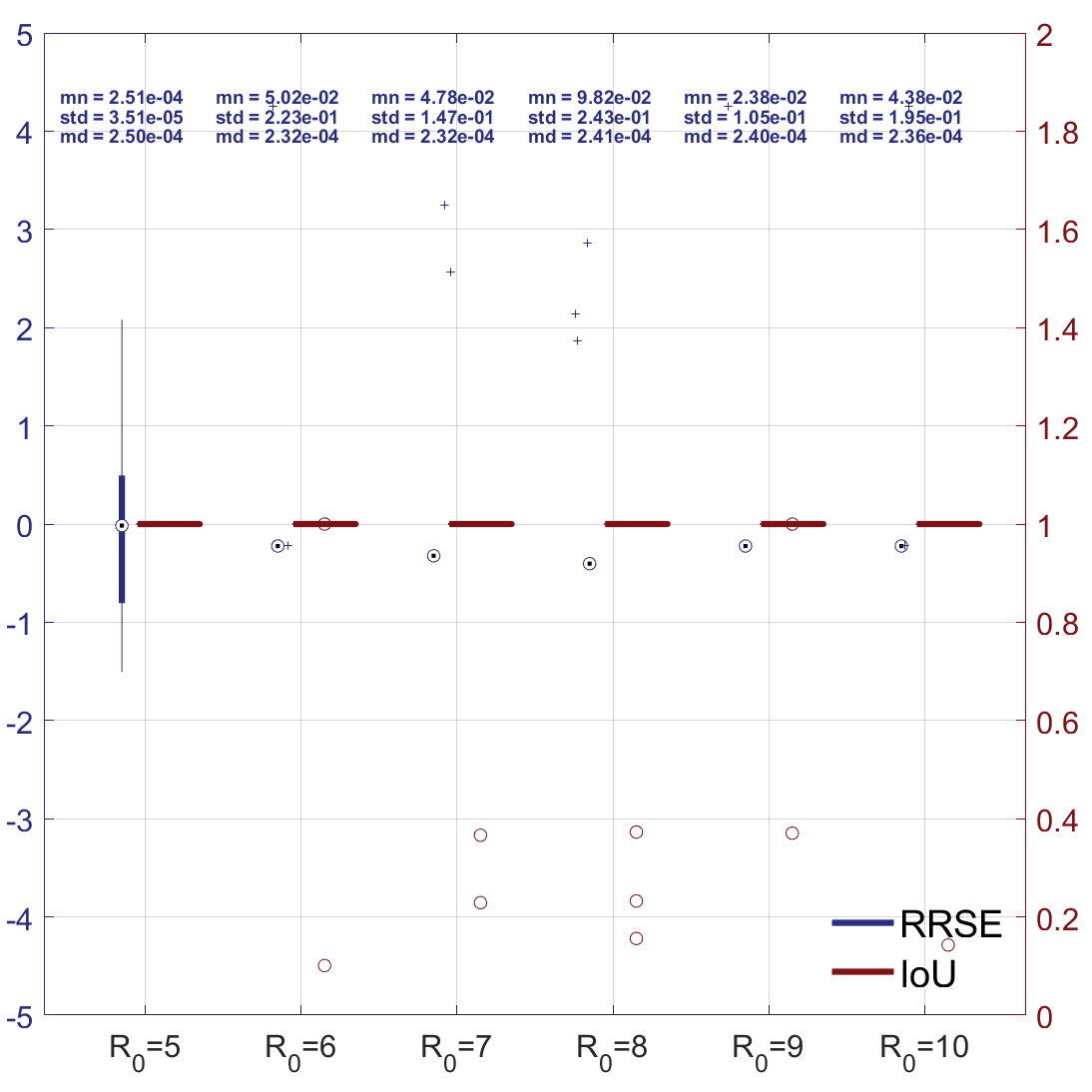}}&\\
    &\multicolumn{2}{c}{$\rho=0.08$}&\multicolumn{2}{c}{$\rho=0.1$}&
  \end{tabular}
  \caption{Box plots of RRSE and IoU on synthetic data. RRSE is
    standardized via Z-score transformation. The statistics in 
    blue in the upper area are the original mean (mn), standard
    deviation (std) and median (md) of RRSE for each group of 20
    runs.}
  \label{fig:Syn}
\end{figure}

The synthetic dataset is created by drawing the ground-truth
$\mathcal{L}$ and $\mathcal{S}$ from random
distributions. Specifically, to simulate an
$\mathcal{L}\in\mathbb{R}^{20\times20\times20\times20}$ with CP rank
of $R_0$, we generate CP factors
$\bigl\{\mathbf{A}^{(n)}\in\mathbb{R}^{20\times R_0}\bigr\}_{n=1}^4$
with
$\mathbf{a}^{(n)}_{i_n}\sim\mathcal{N}\left(\mathbf{0},\mathbf{I}_{R_0}\right)$,
$i_n\in[20]$, such that
$\mathcal{L}=\operatorname{CP}\bigl(\mathbf{A}^{(1)},\dots,\mathbf{A}^{(4)}\bigr).$
To simulate the sparse component
$\mathcal{S}\in\mathbb{R}^{20\times20\times20\times20}$, we first
generate an
$\widetilde{\mathcal{S}}\in\mathbb{R}^{20\times20\times20\times20}$
with each element $\widetilde{\mathcal{S}}_{i_1\cdots
  i_4}\sim\mathcal{N}(0,1)$. We then set $J_n=4$ and $K_n=5$ for
$n\in[4],$ so that there are a total of $5^4$ blocks of size
$4\times4\times4\times4$ in $\mathcal{S}$. Next, a blockwise support
$\mathtt{\Omega}$ with $\lvert\mathtt{\Omega}\rvert=5^4\rho$ is
uniformly drawn from $[5^4]$, where $\rho$ is the blockwise sampling
ratio and $\lvert\cdot\rvert$ returns the cardinality of a
set. Finally, the synthetic sparse component is
$\mathcal{S}=\mathscr{P}_{\Omega}\bigl[\widetilde{\mathcal{S}}\bigr]$, and
the resulting observation is 
$\mathcal{Y}=\mathscr{P}_{\mathtt{\Omega}^\bot}\bigl[\mathcal{L}\bigr]+\mathcal{S}$.
To measure the identifiability of the proposed model, we employ
the relative root squared error (RRSE), 
\begin{equation}
  \operatorname{RRSE}(\widehat{\mathcal{L}},\mathcal{L}) =
  \frac{\bigl\|\widehat{\mathcal{L}}-\mathcal{L}\bigr\|_F}{\bigl\|\mathcal{L}\bigr\|_F} ,
  \label{eq:rrse}
\end{equation}
and the intersection over union (IoU),
\begin{equation}
  \operatorname{IoU}(\widehat{\mathtt{\Omega}}, \mathtt{\Omega}) =
  \frac{\bigl\lvert\widehat{\mathtt{\Omega}}\cap\mathtt{\Omega}\bigr\rvert}{\bigl\lvert\widehat{\mathtt{\Omega}}\cup\mathtt{\Omega}\bigr\rvert}.
  \label{eq:iou}
\end{equation}
Note that the accuracy of the $\widehat{\mathcal{S}}$ estimate is
directly tied to the closeness of $\widehat{\mathtt{\Omega}}$ to
$\mathtt{\Omega}$, thus we need no separate measure for
$\widehat{\mathcal{S}}$.

To comprehensively evaluate performance, we vary $R_0$ over
$\left\{5,6,7,8,9,10\right\}$ and $\rho$ over
$\left\{0.02,0.04,0.06,0.08,0.1\right\}$, leading to a total of
$6\times5=30$ groups of experiments; each group is repeated for
$20$ trials. In Alg.~\ref{alg:BCP-RPCC}, parameters $R$ and $\sigma$
are set to $R=2R_0$ and $\sigma=10^{-4}$.  The results are reported as
box plots in Fig.~\ref{fig:Syn}. We see that the medians of RRSE stay
below $2.5\times10^{-4}$ with the standard deviation being yet another
order of magnitude lower in most cases. Even in other cases in which
the standard deviation increases to around $10^{-1}$ (e.g., for
parameters $\left(\rho=0.04,R_0=7\right)$ and
$\left(\rho=0.08,R_0=5\right)$), the corresponding box plot collapses
to a single dot, indicating that the large standard deviation is
attributable to a few outliers (shown as blue crosses).  The box plot
of IoU is a single line at $\operatorname{IoU}=1$ for all groups with
merely a few outliers visible (shown as red circles). Given the sparsity
of these outliers, we conclude that the proposed model achieves a
near-optimal solution to the original RPCC problem with high
probability, a surprising result in light of the NP-hardness of the problem.

\subsection{Foreground Extraction from Color Video}
\label{sec:resultsforeground}
\begin{figure}[t]
  \centering
  \setlength{\tabcolsep}{0.1mm}
  \begin{tabular}{cm{\largeimage}m{\largeimage}m{\largeimage}}
    &\multicolumn{1}{c}{\footnotesize Original Frame} &
    \multicolumn{1}{c}{\footnotesize Ground Truth} &
    \multicolumn{1}{c}{\footnotesize ROI} \\
    \rotatebox[origin=c]{90}{\makecell{\footnotesize Highway\\[-0.25em]\footnotesize $240\times320$}} &
    \includegraphics[width=\largeimage]{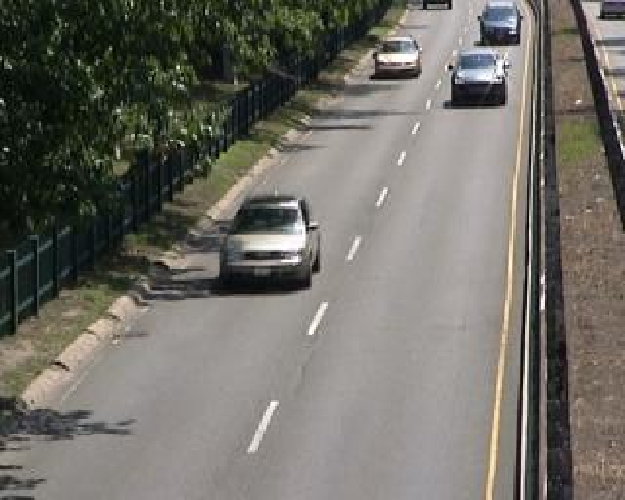} &
    \includegraphics[width=\largeimage]{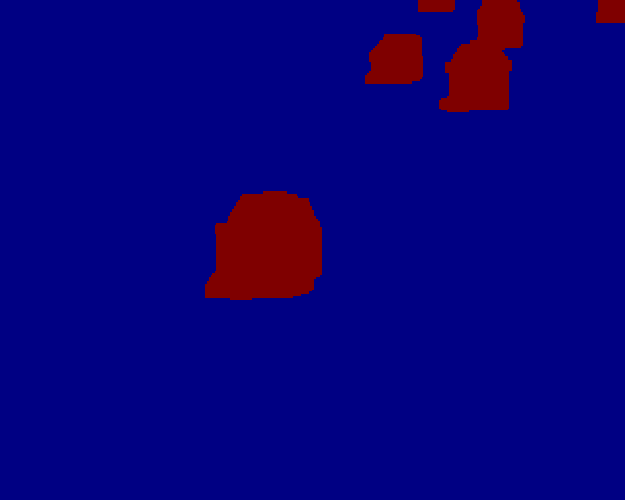} &
    \includegraphics[width=\largeimage]{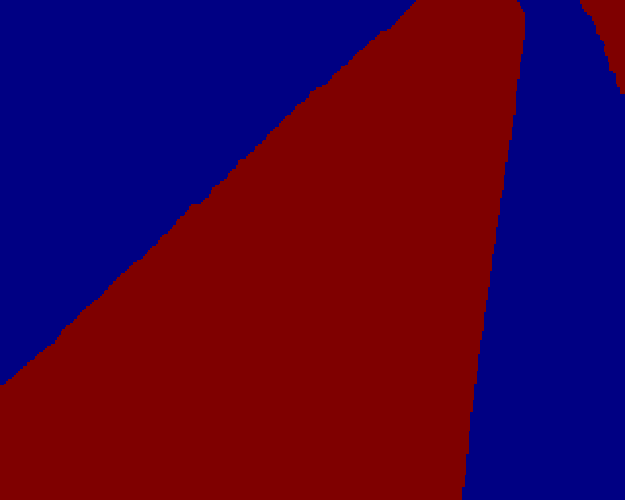} \\
    \rotatebox[origin=c]{90}{\makecell{\footnotesize Turnpike\\[-0.25em]\footnotesize $240\times320$}} &
    \includegraphics[width=\largeimage]{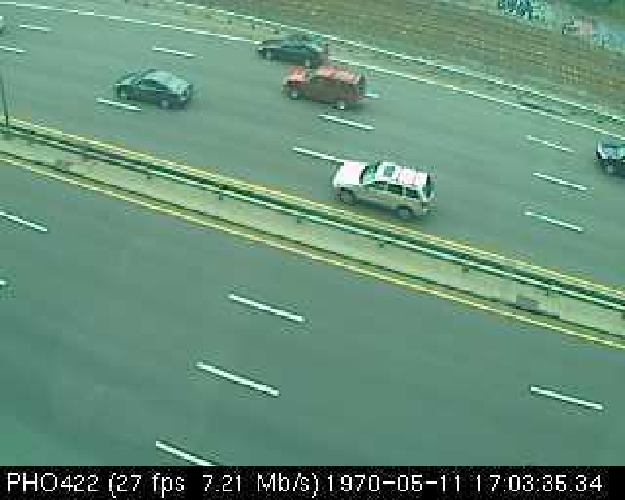} &
    \includegraphics[width=\largeimage]{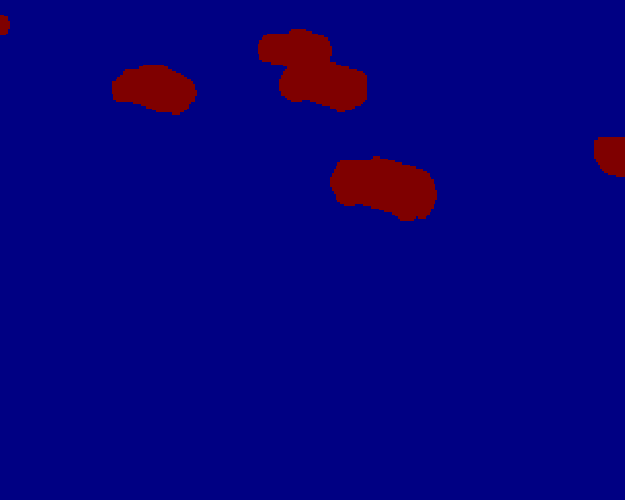} &
    \includegraphics[width=\largeimage]{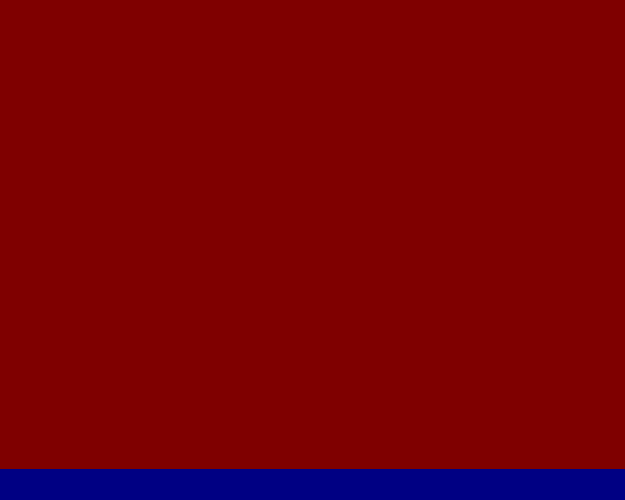} \\
    \rotatebox[origin=c]{90}{\makecell{\footnotesize Crossroad\\[-0.25em]\footnotesize $190\times400$}} &
    \includegraphics[width=\largeimage]{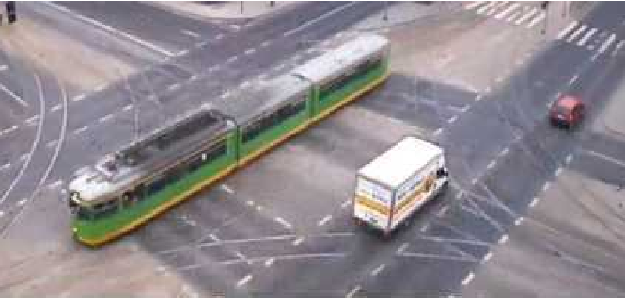} &
    \includegraphics[width=\largeimage]{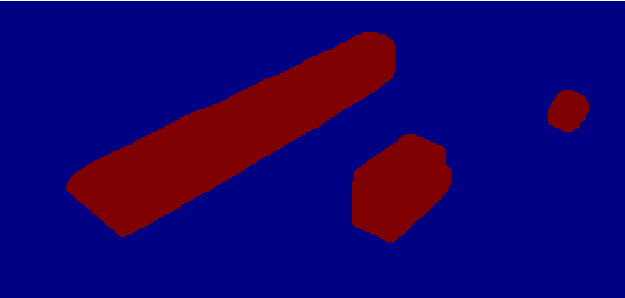} &
    \includegraphics[width=\largeimage]{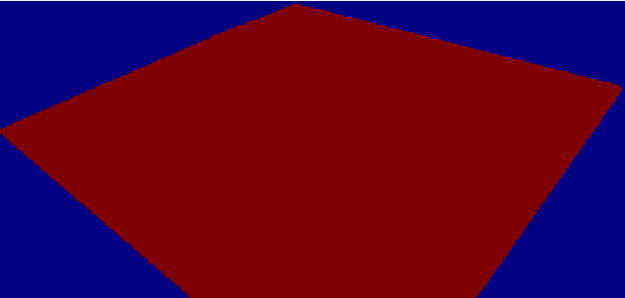} \\
    \rotatebox[origin=c]{90}{\makecell{\footnotesize Busstation\\[-0.25em]\footnotesize $240\times360$}} &
    \includegraphics[width=\largeimage]{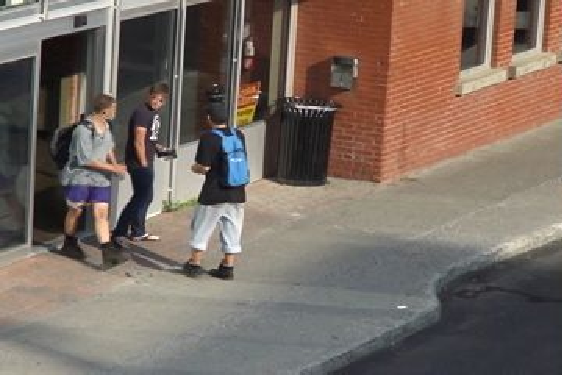} &
    \includegraphics[width=\largeimage]{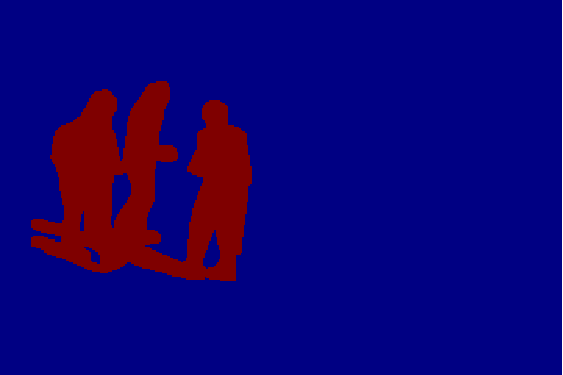} &
    \includegraphics[width=\largeimage]{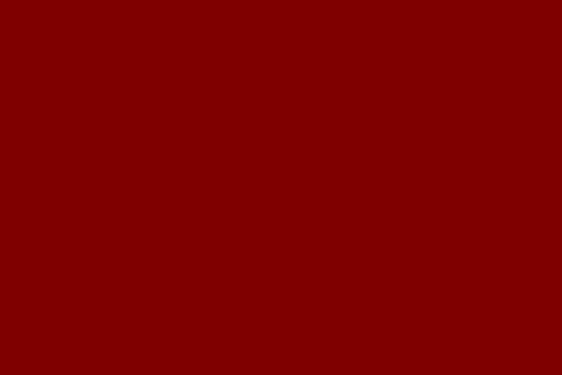}
    \\
    &\multicolumn{3}{c}{\includegraphics[width=2.5\largeimage]{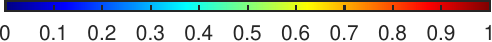}}
  \end{tabular}
  \caption{Example frames, along with ground-truth foreground masks
    and ROIs, from the four CDnet videos used for the
    foreground-extraction experiments.}
  \label{fig:FgMdata}
\end{figure}

\begin{figure}[t]
  \centering
  \setlength{\tabcolsep}{0.2mm}
  \begin{tabular}{ccc}
    \includegraphics[width=\mediumplot]{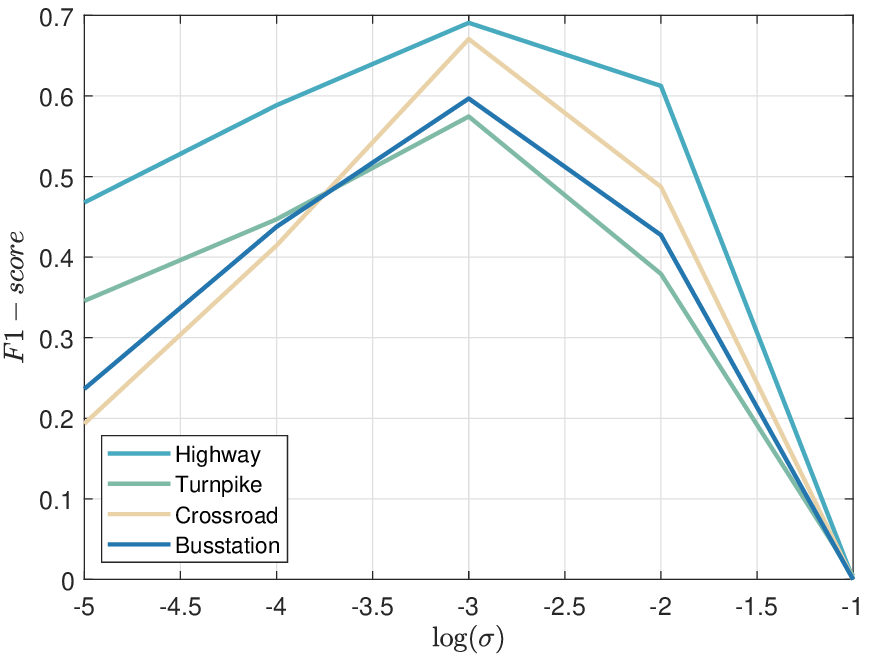} &
    \includegraphics[width=\mediumplot]{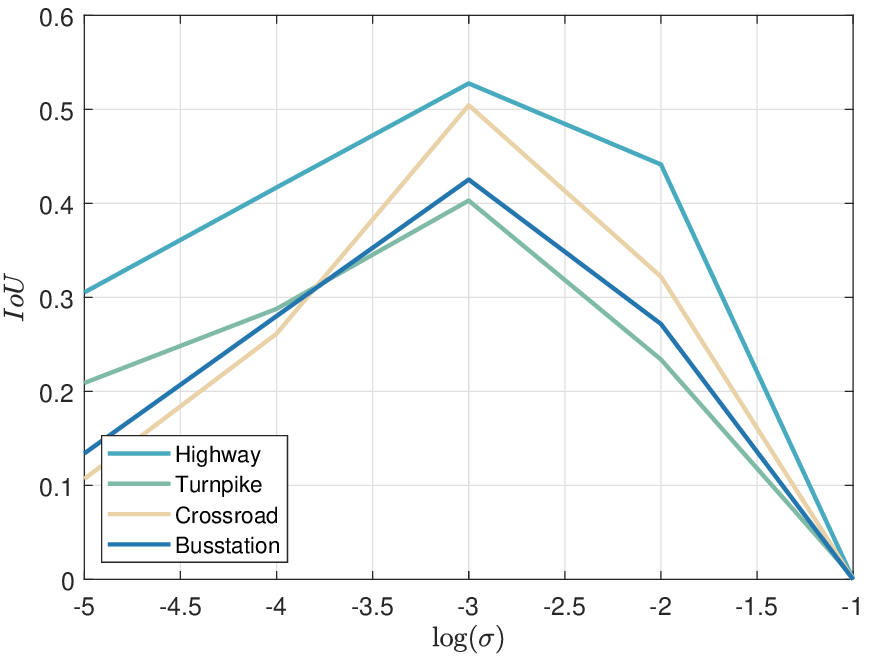} \\
    \includegraphics[width=\mediumplot]{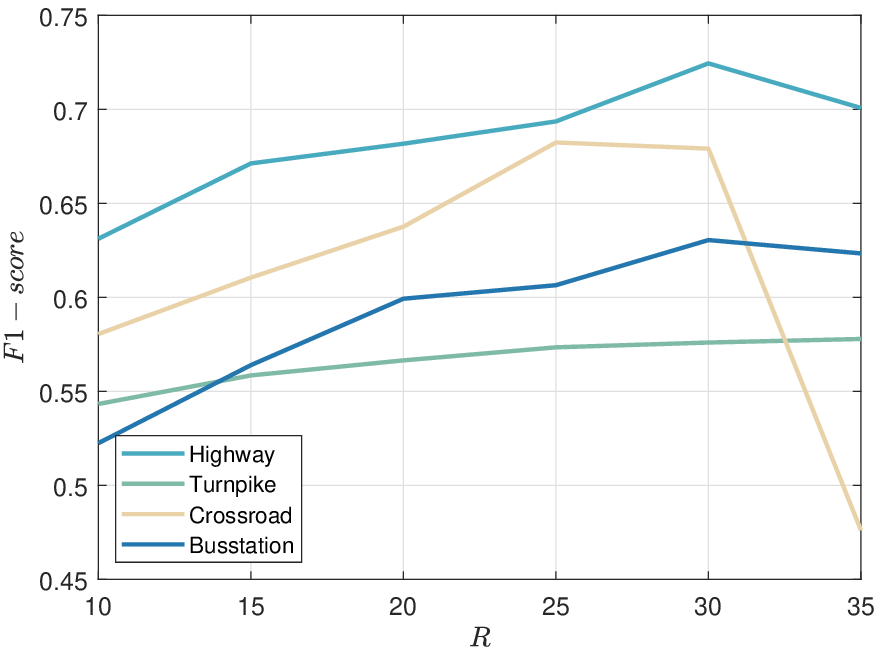} &
    \includegraphics[width=\mediumplot]{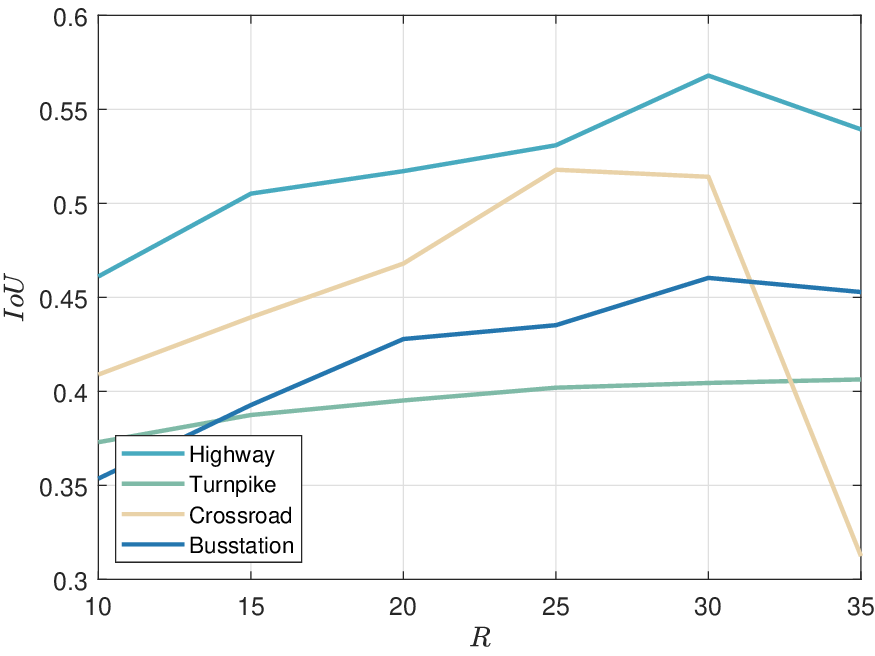}
  \end{tabular}
  \caption{Hyperparameter tuning for foreground extraction
    using $\operatorname{F1}$ and $\operatorname{IoU}$.
    Row 1: Tuning $\sigma$ when $R=25$. Row 2: Tuning $R$
    when $\sigma=10^{-3}$.}
  \label{fig:FgMTuning}
\end{figure}

\begin{figure*}[t]
  \centering
  \setlength{\tabcolsep}{0.1mm}
  \begin{tabular}{cccccccccc}
    \includegraphics[width=\smallimage]{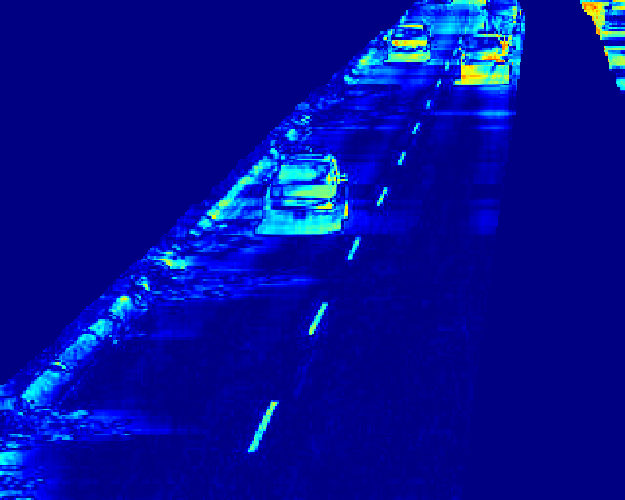} &
    \includegraphics[width=\smallimage]{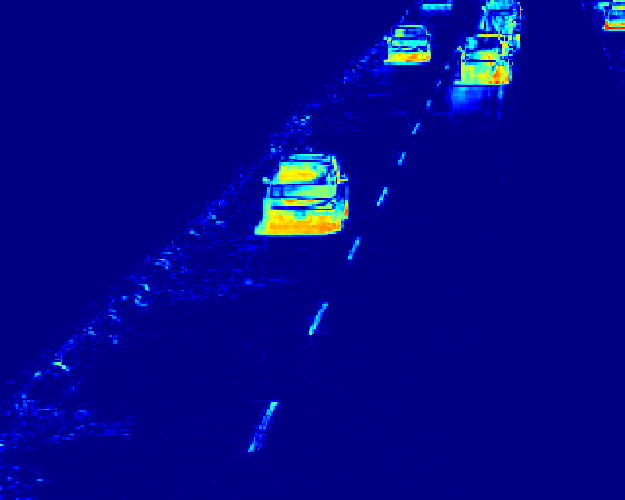} &
    \includegraphics[width=\smallimage]{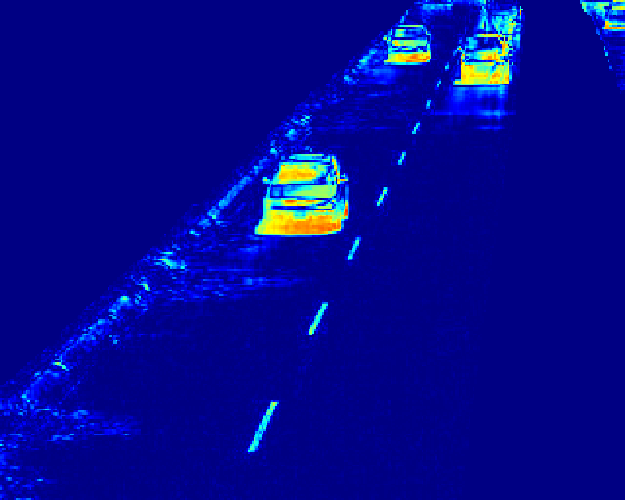} &
    \includegraphics[width=\smallimage]{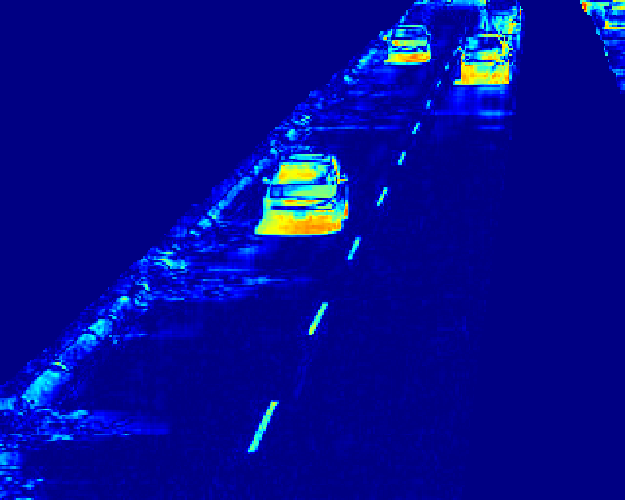} &
    \includegraphics[width=\smallimage]{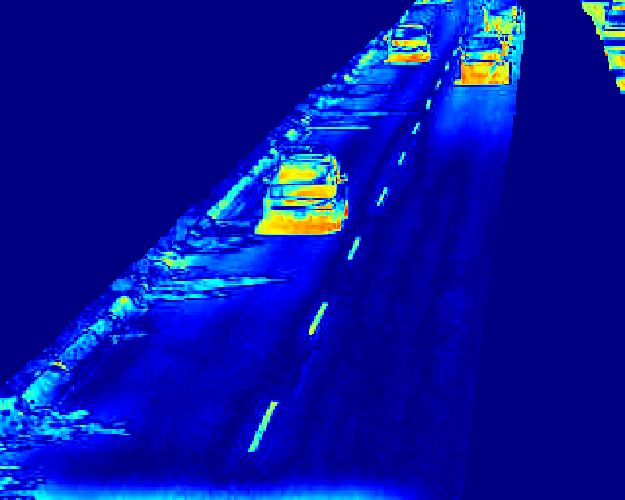} &
    \includegraphics[width=\smallimage]{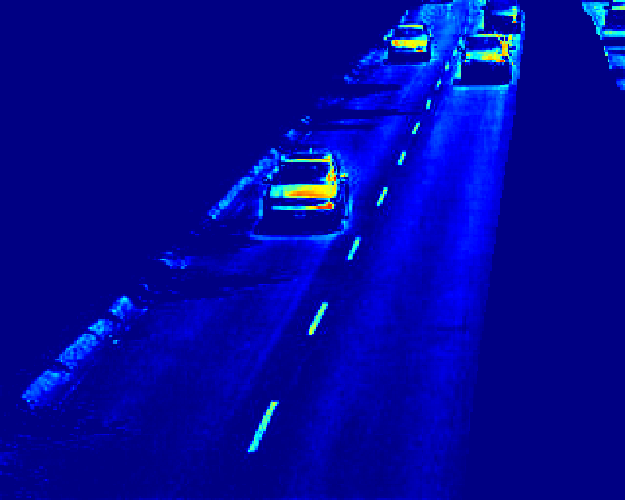} &
    \includegraphics[width=\smallimage]{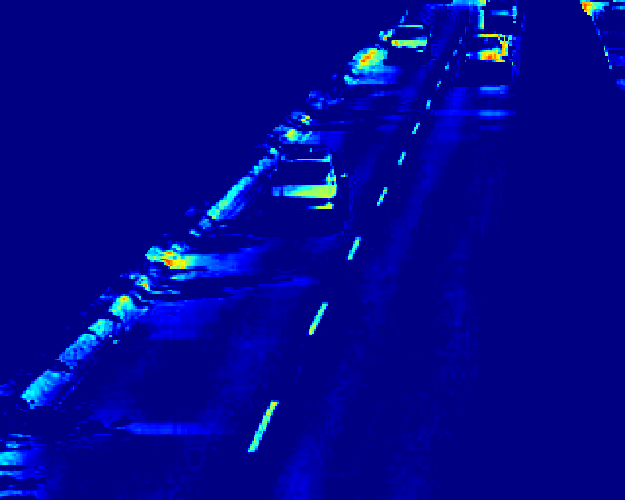} &
    \includegraphics[width=\smallimage]{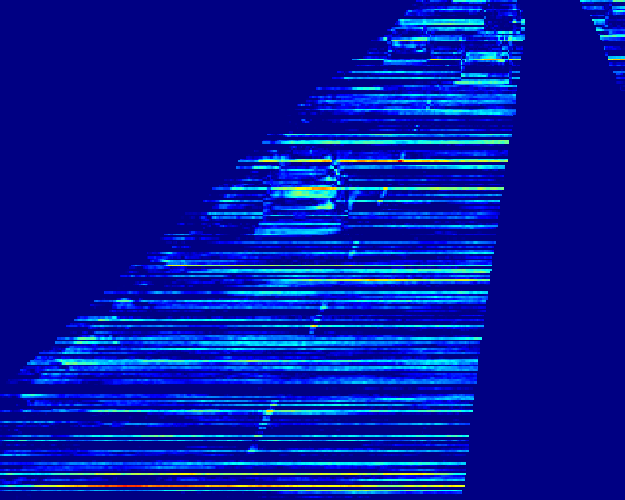} &
    \includegraphics[width=\smallimage]{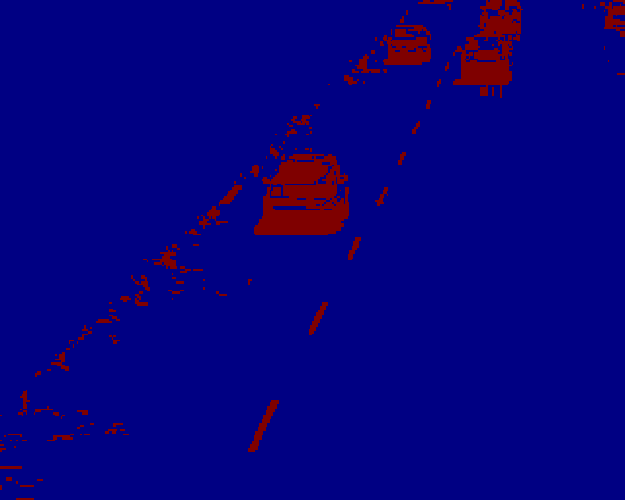} &
    \includegraphics[width=\smallimage]{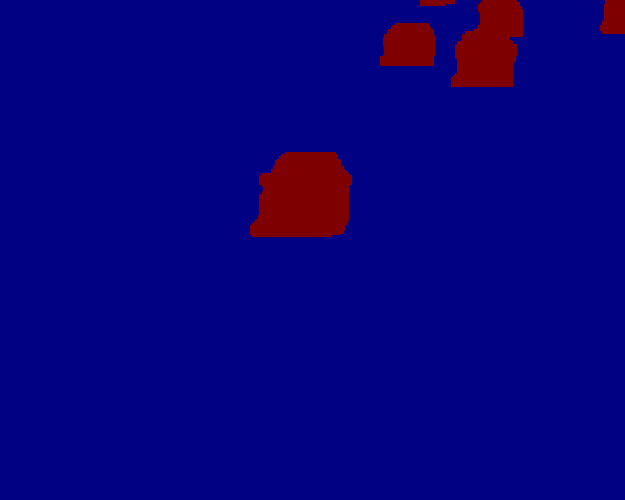}
    \\
    \includegraphics[width=\smallimage]{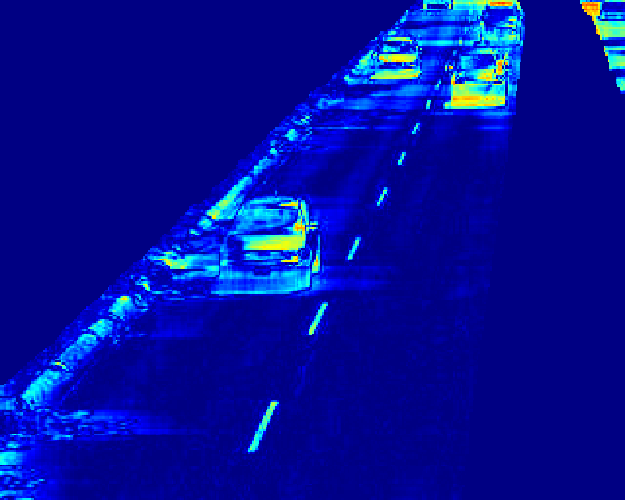} &
    \includegraphics[width=\smallimage]{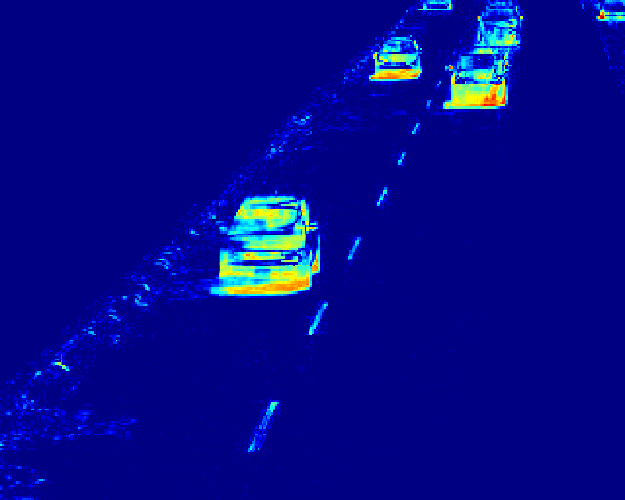} &
    \includegraphics[width=\smallimage]{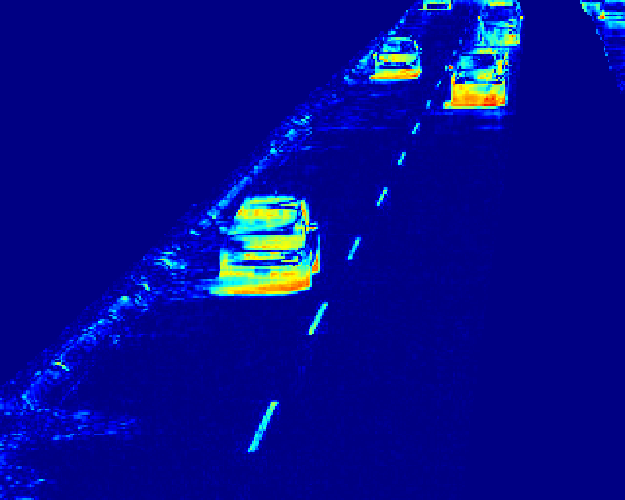} &
    \includegraphics[width=\smallimage]{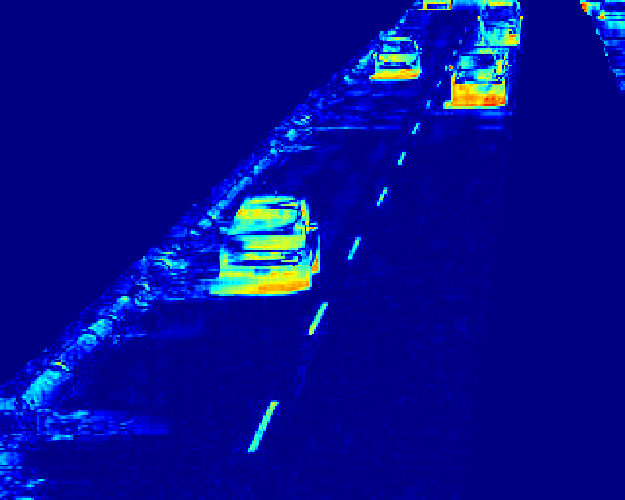} &
    \includegraphics[width=\smallimage]{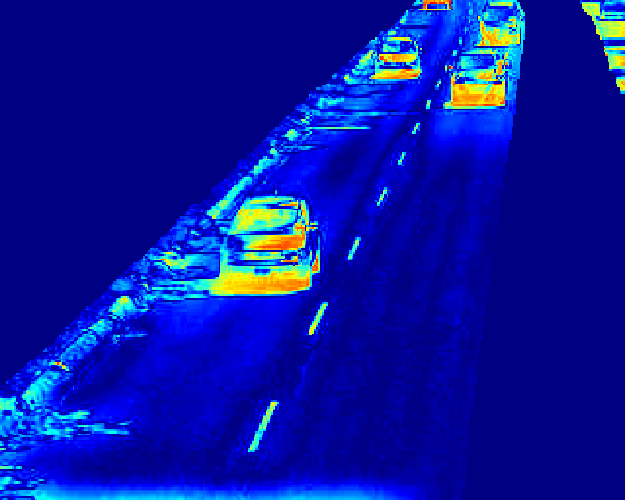} &
    \includegraphics[width=\smallimage]{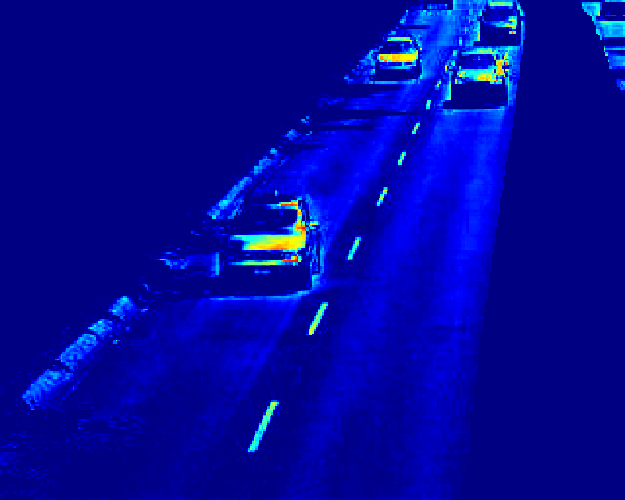} &
    \includegraphics[width=\smallimage]{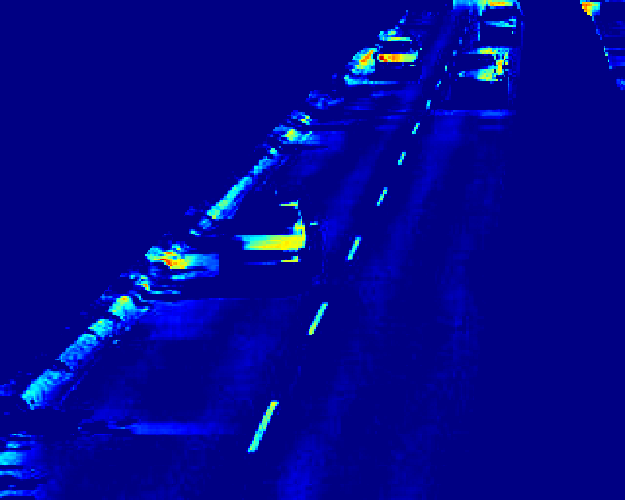} &
    \includegraphics[width=\smallimage]{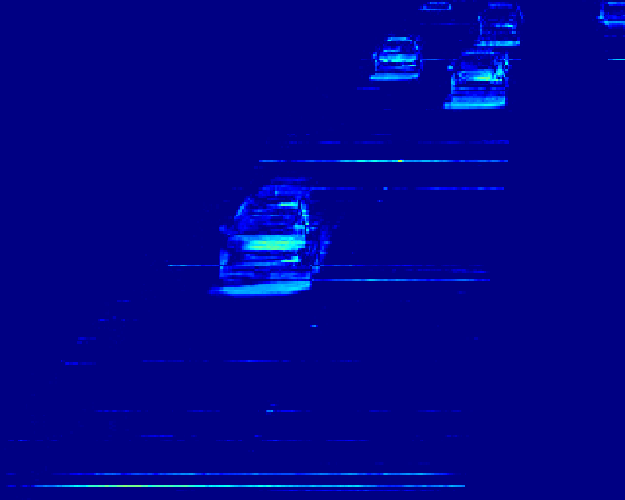} &
    \includegraphics[width=\smallimage]{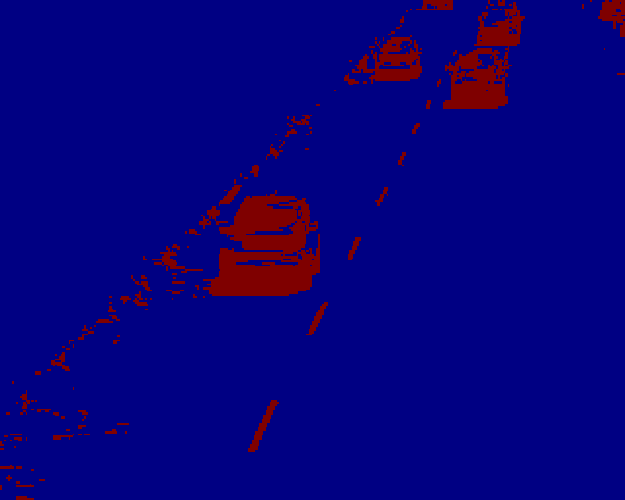} &
    \includegraphics[width=\smallimage]{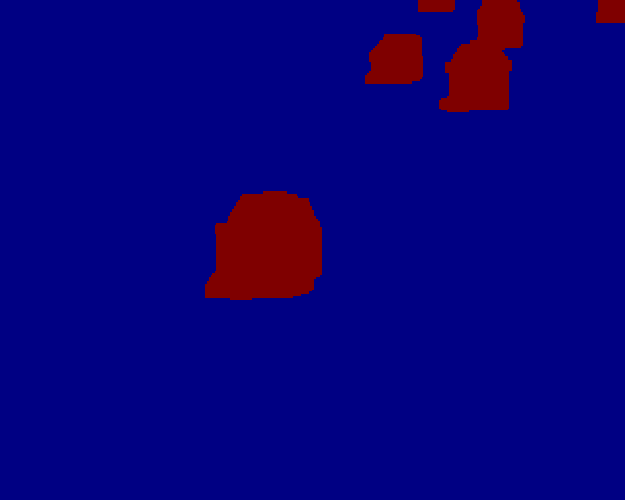}
    \\
    \includegraphics[width=\smallimage]{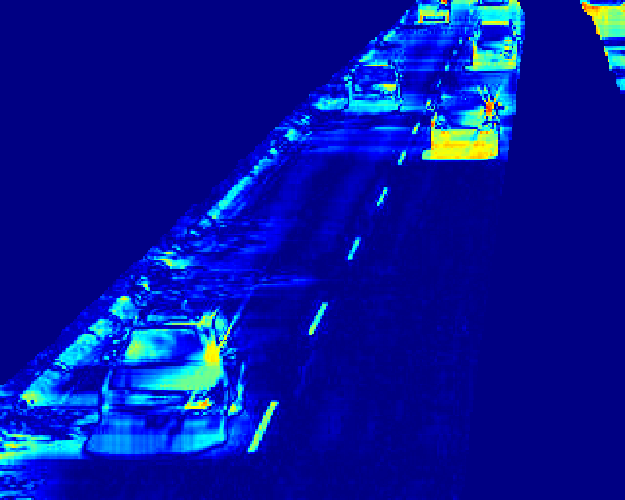} &
    \includegraphics[width=\smallimage]{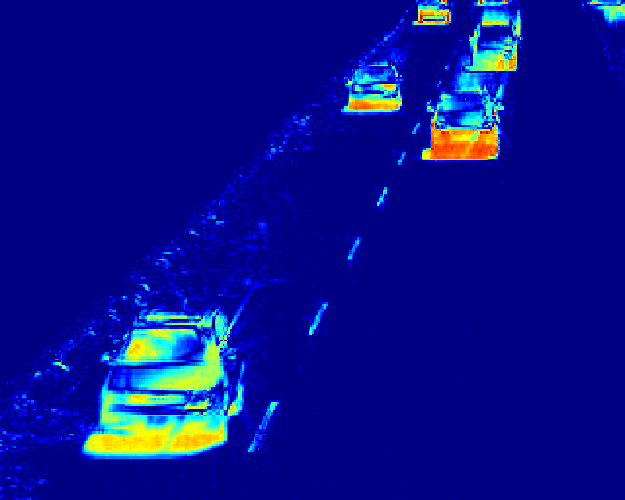} &
    \includegraphics[width=\smallimage]{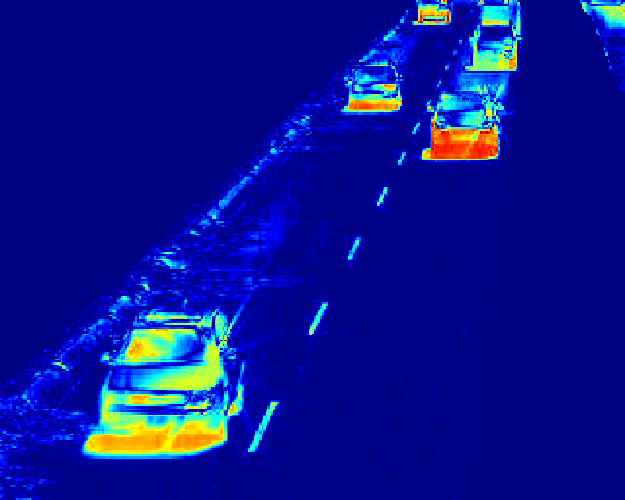} &
    \includegraphics[width=\smallimage]{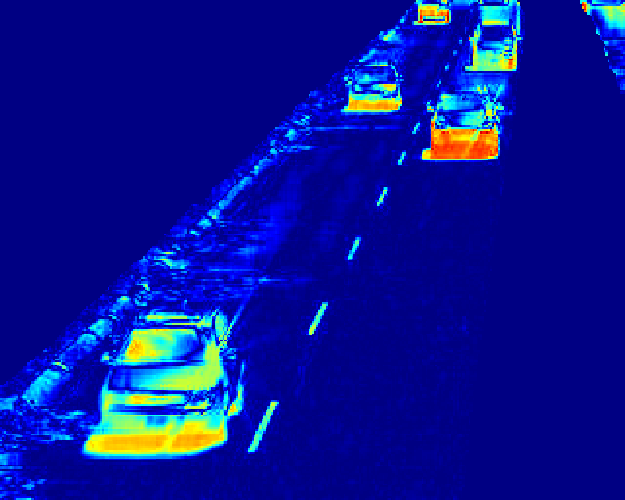} &
    \includegraphics[width=\smallimage]{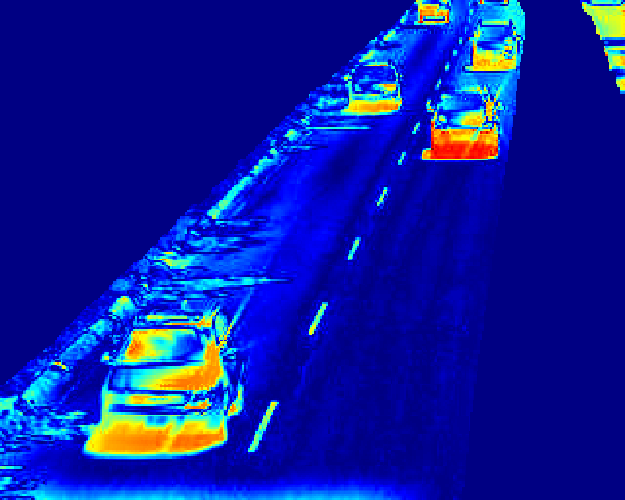} &
    \includegraphics[width=\smallimage]{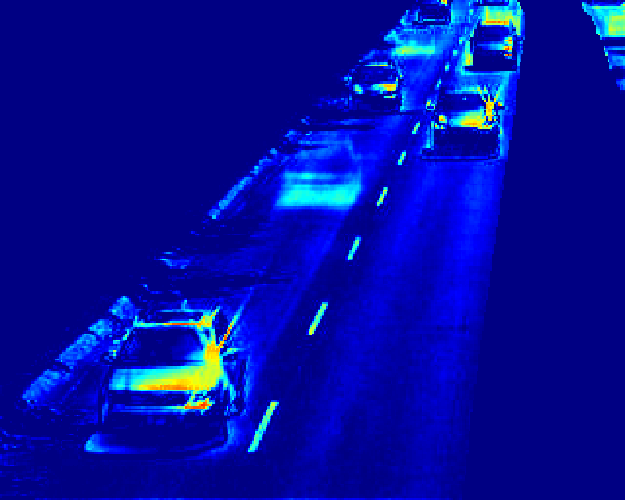} &
    \includegraphics[width=\smallimage]{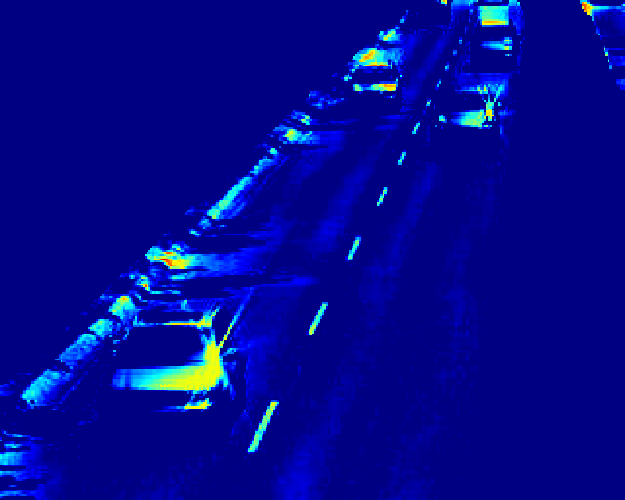} &
    \includegraphics[width=\smallimage]{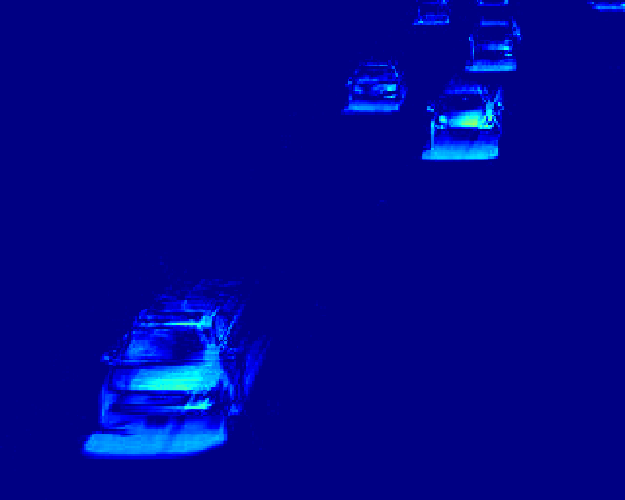} &
    \includegraphics[width=\smallimage]{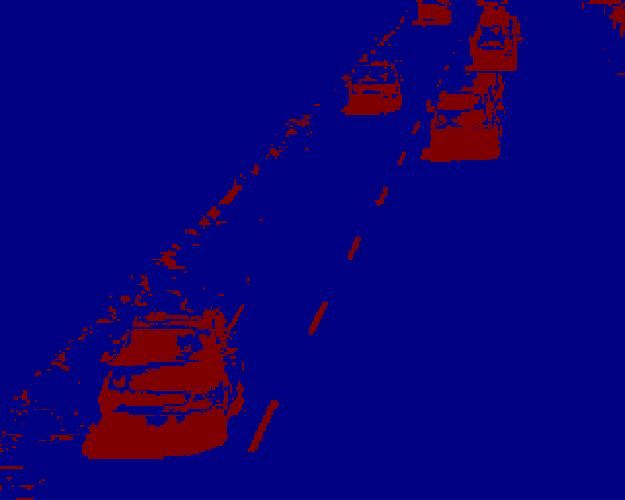} &
    \includegraphics[width=\smallimage]{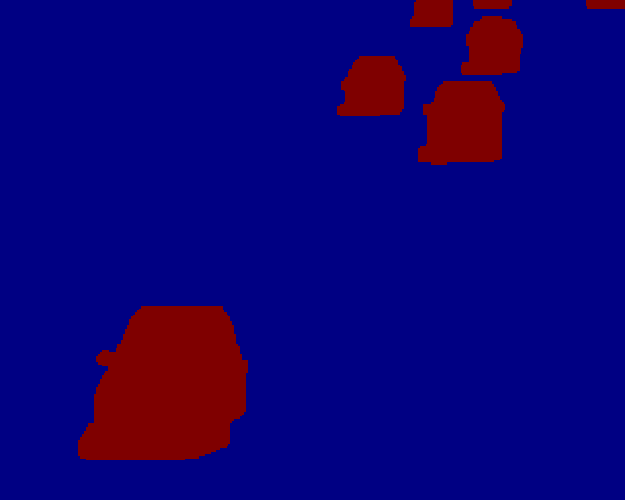}
    \\
    \multicolumn{1}{c}{\footnotesize PCP} &
    \multicolumn{1}{c}{\footnotesize KBR} &
    \multicolumn{1}{c}{\footnotesize TRPCA} &
    \multicolumn{1}{c}{\footnotesize ETRPCA} &
    \multicolumn{1}{c}{\footnotesize t-CTV} &
    \multicolumn{1}{c}{\footnotesize MTTD} &
    \multicolumn{1}{c}{\footnotesize LRTFR} &
    \multicolumn{1}{c}{\footnotesize ONTRPCA} &
    \multicolumn{1}{c}{\footnotesize BCP-RPCC} &
    \multicolumn{1}{c}{\footnotesize Ground truth}
  \end{tabular}
  \caption{Foreground extraction for the Highway
    dataset. Row 1: Frame 10. Row 2: Frame 25. Row 3: Frame 50.}
  \label{fig:FMHighway}
\end{figure*}

\begin{figure*}[t]
  \centering
  \setlength{\tabcolsep}{0.1mm}
  \begin{tabular}{cccccccccc}
    \includegraphics[width=\smallimage]{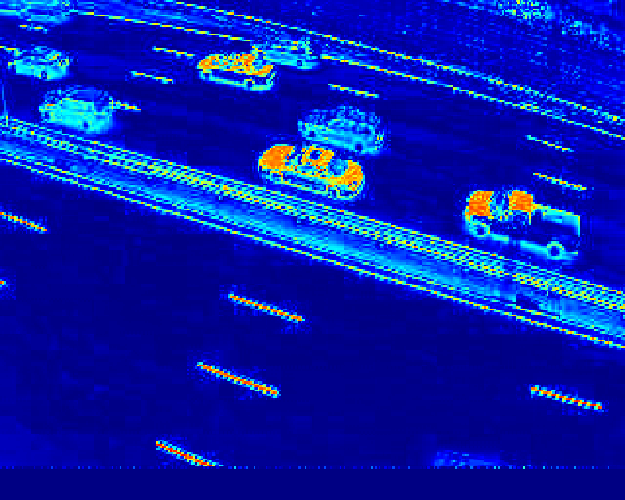} &
    \includegraphics[width=\smallimage]{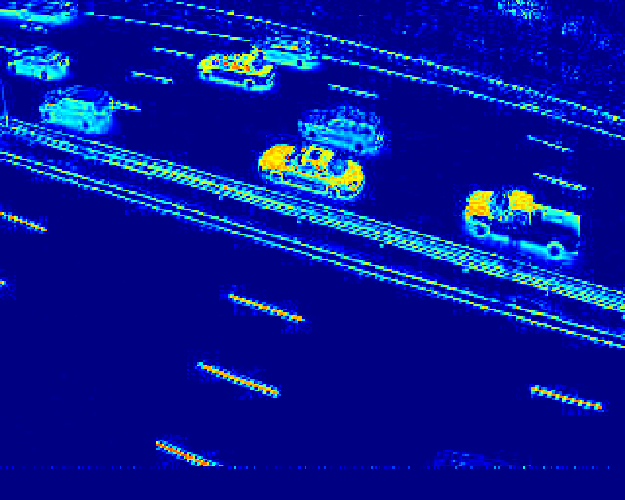} &
    \includegraphics[width=\smallimage]{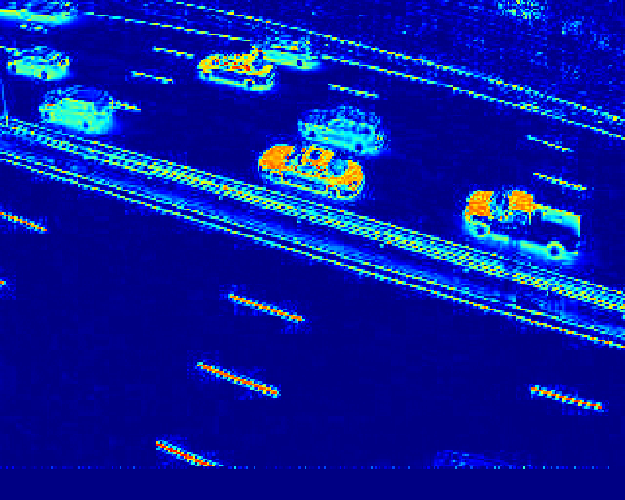} &
    \includegraphics[width=\smallimage]{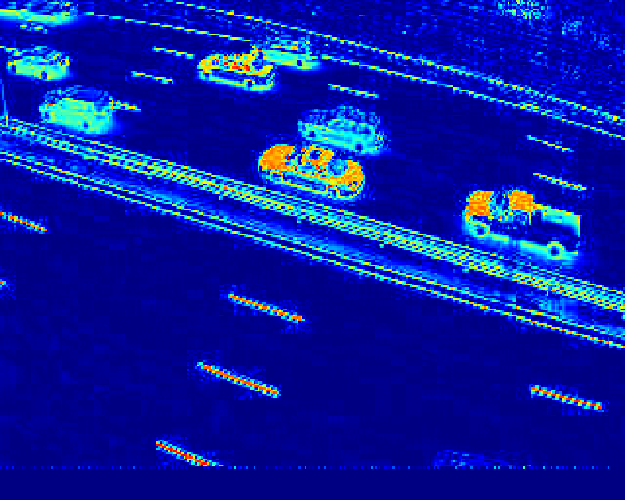} &
    \includegraphics[width=\smallimage]{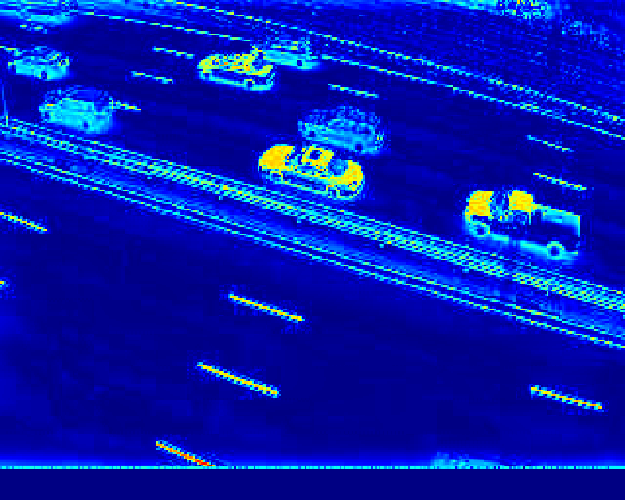} &
    \includegraphics[width=\smallimage]{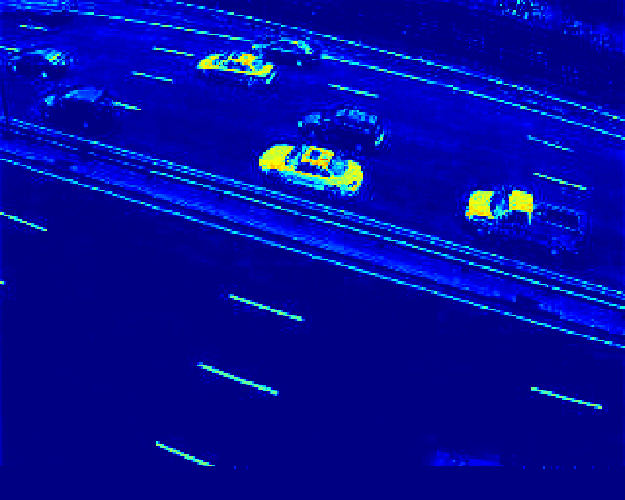} &
    \includegraphics[width=\smallimage]{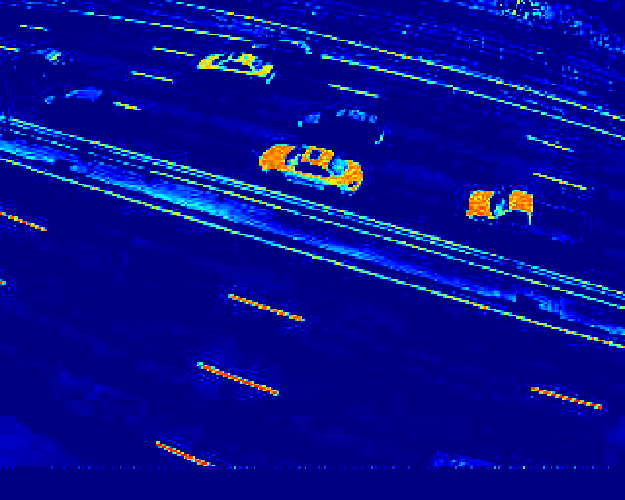} &
    \includegraphics[width=\smallimage]{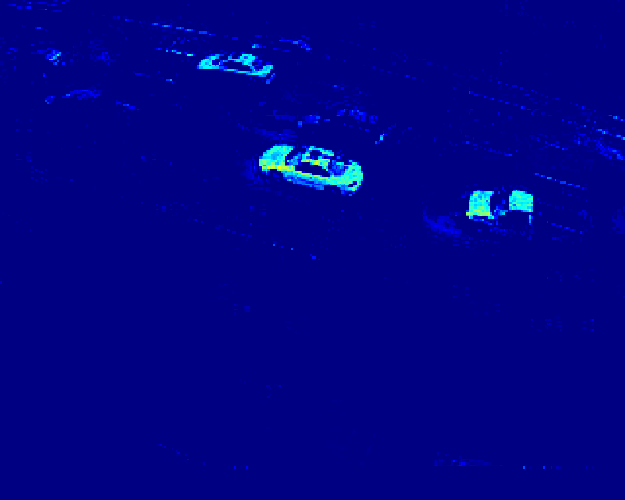} &
    \includegraphics[width=\smallimage]{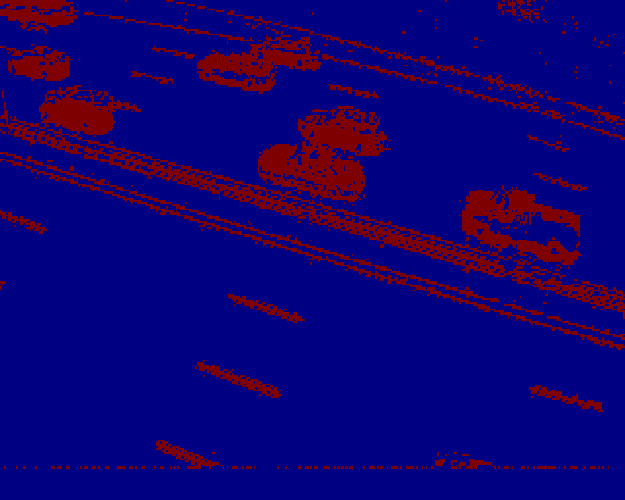} &
    \includegraphics[width=\smallimage]{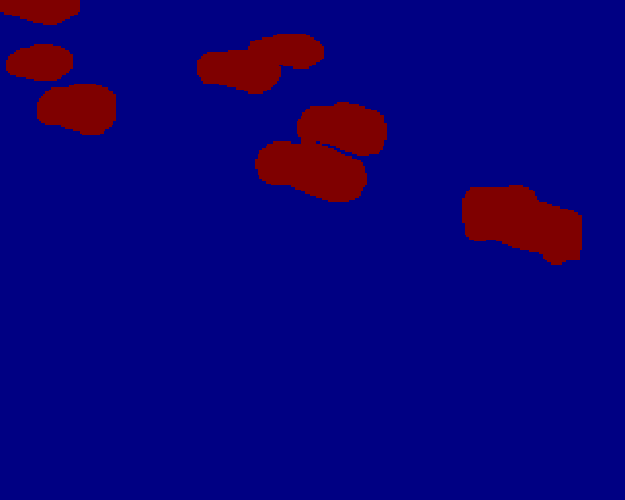}
    \\
    \includegraphics[width=\smallimage]{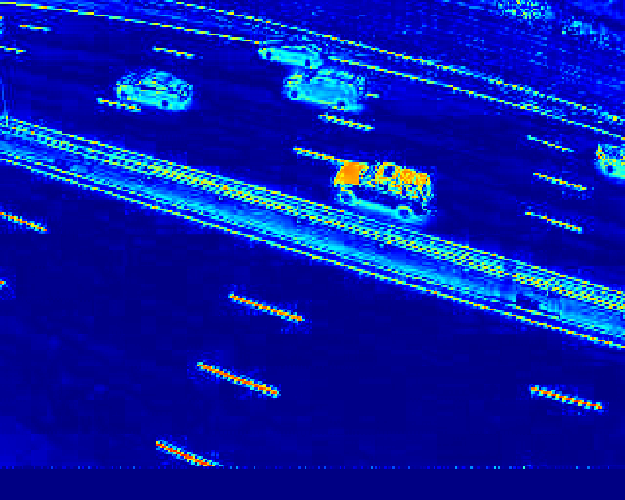} &
    \includegraphics[width=\smallimage]{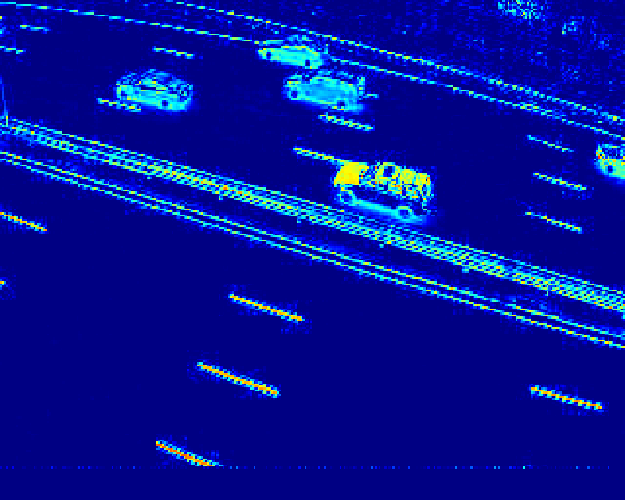} &
    \includegraphics[width=\smallimage]{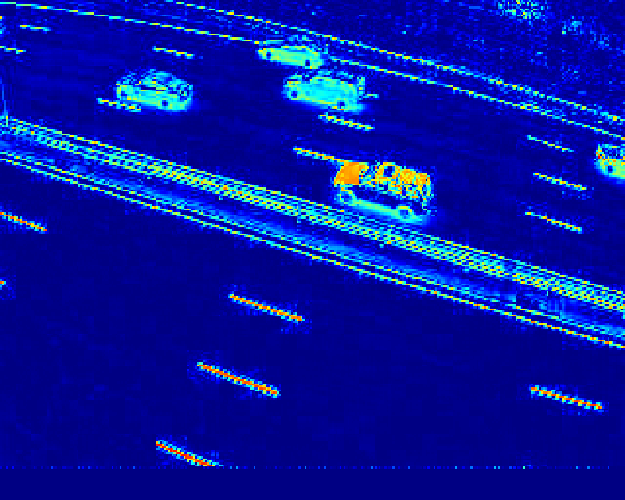} &
    \includegraphics[width=\smallimage]{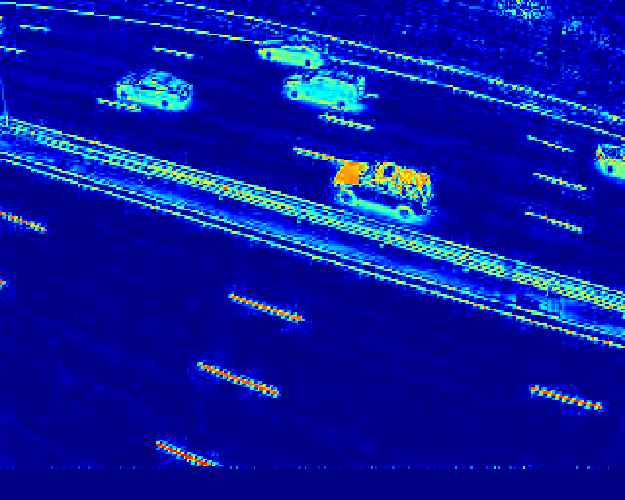} &
    \includegraphics[width=\smallimage]{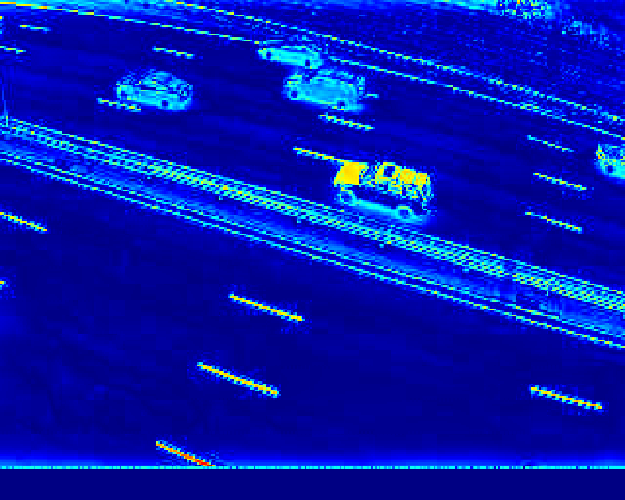} &
    \includegraphics[width=\smallimage]{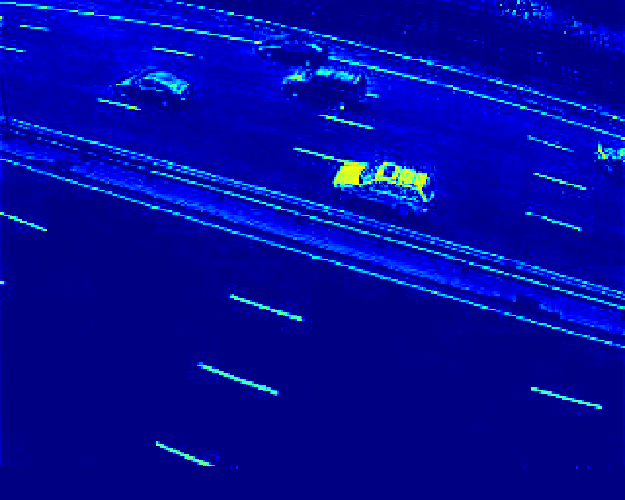} &
    \includegraphics[width=\smallimage]{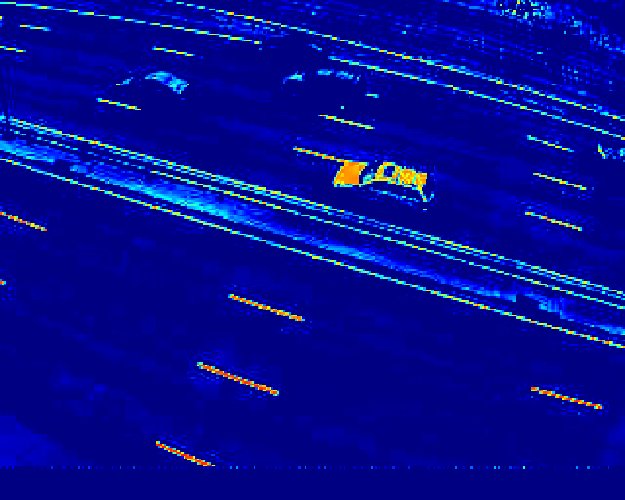} &
    \includegraphics[width=\smallimage]{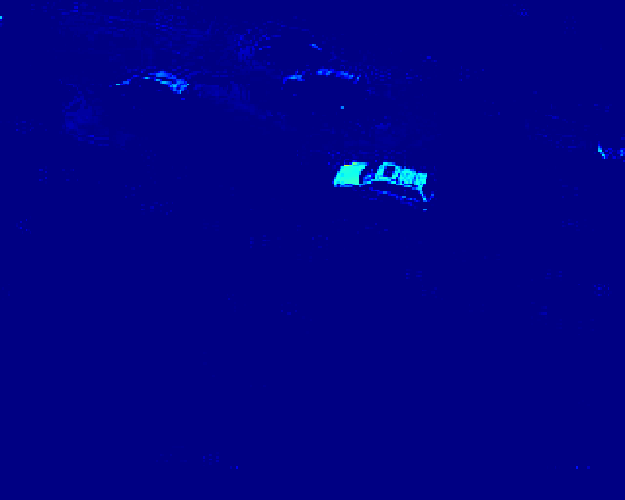} &
    \includegraphics[width=\smallimage]{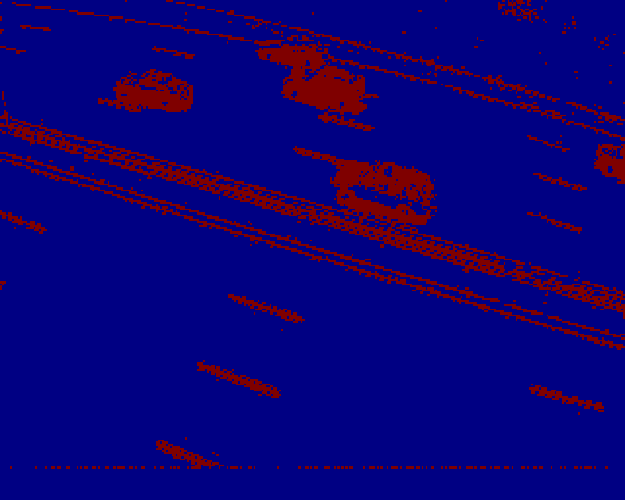} &
    \includegraphics[width=\smallimage]{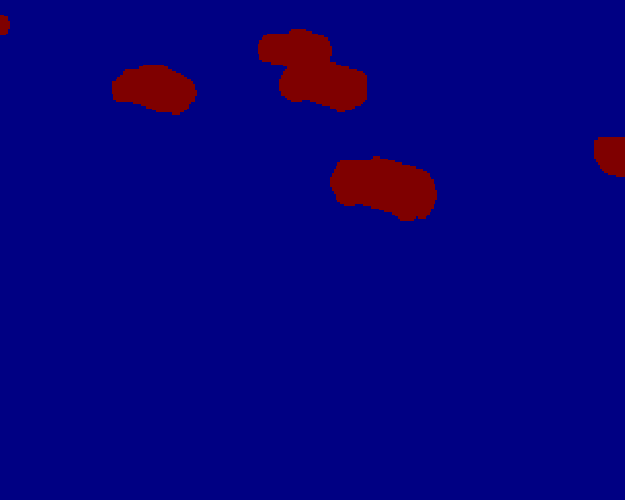}
    \\
    \includegraphics[width=\smallimage]{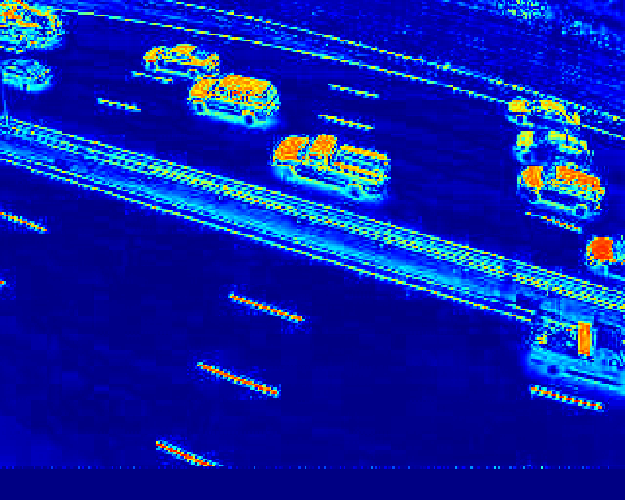} &
    \includegraphics[width=\smallimage]{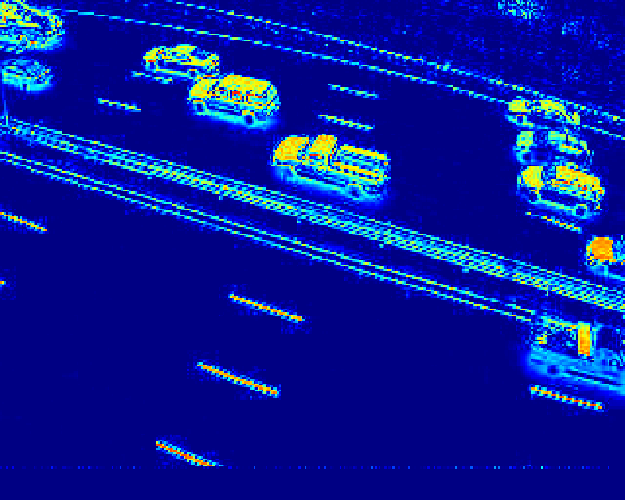} &
    \includegraphics[width=\smallimage]{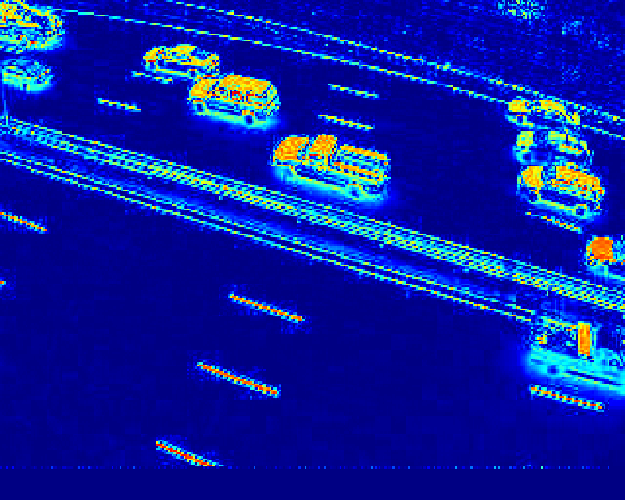} &
    \includegraphics[width=\smallimage]{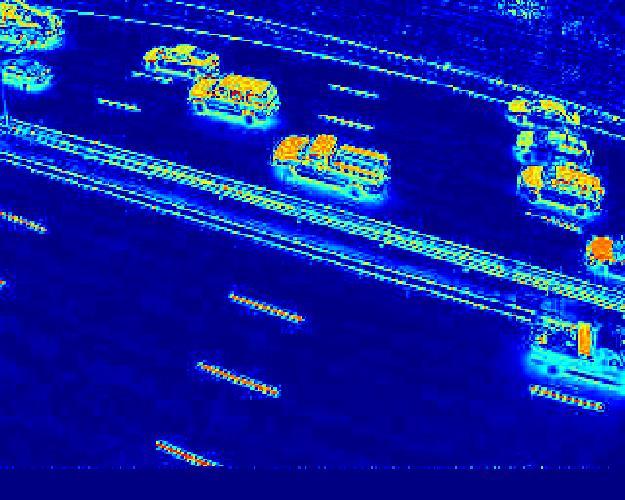} &
    \includegraphics[width=\smallimage]{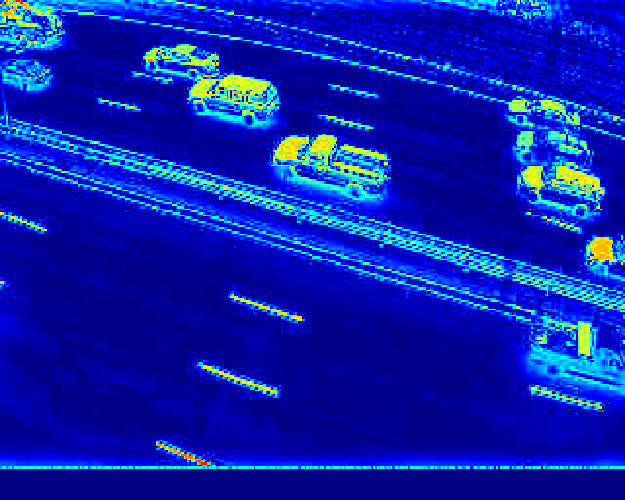} &
    \includegraphics[width=\smallimage]{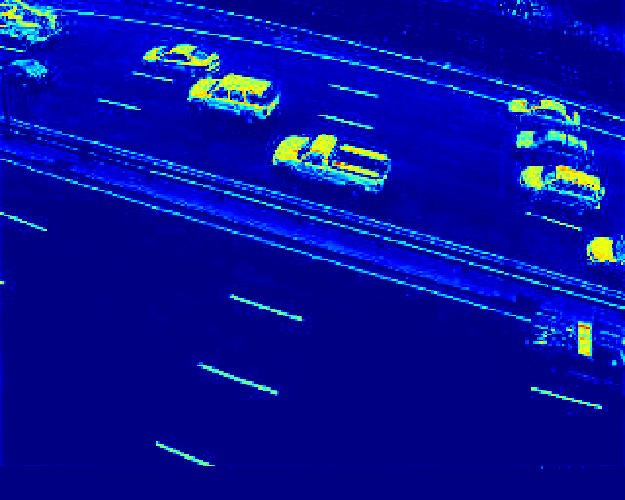} &
    \includegraphics[width=\smallimage]{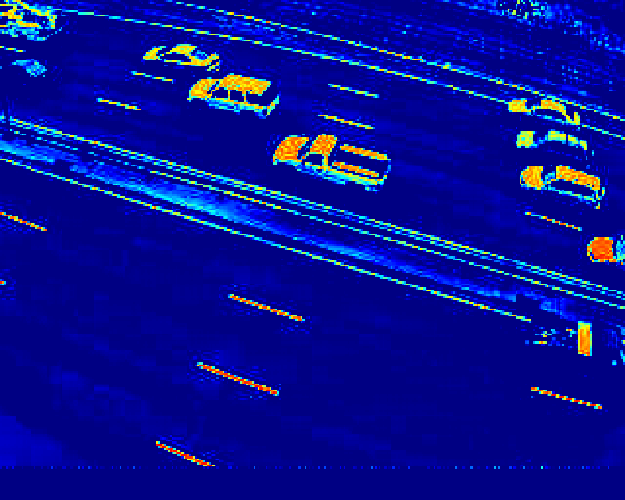} &
    \includegraphics[width=\smallimage]{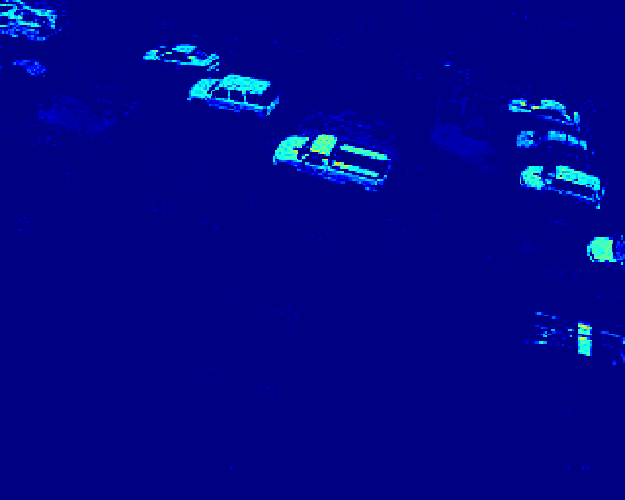} &
    \includegraphics[width=\smallimage]{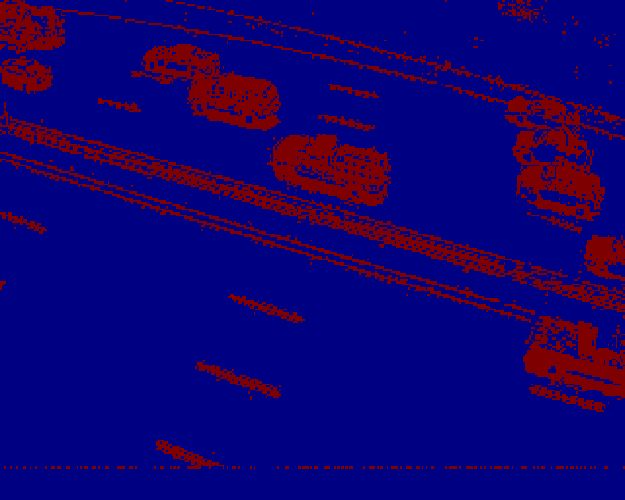} &
    \includegraphics[width=\smallimage]{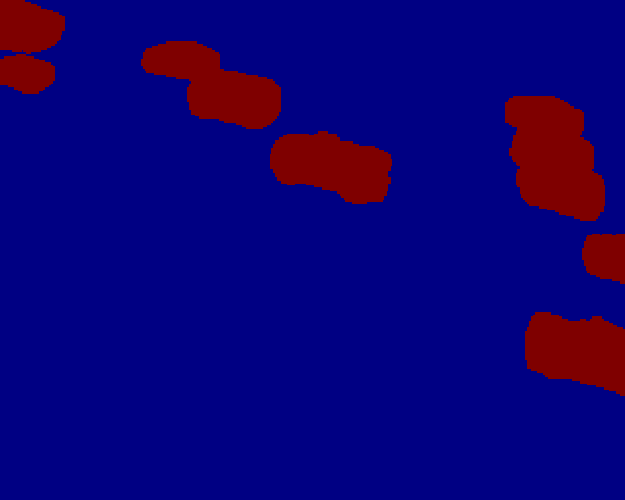}
    \\
    \multicolumn{1}{c}{\footnotesize PCP} &
    \multicolumn{1}{c}{\footnotesize KBR} &
    \multicolumn{1}{c}{\footnotesize TRPCA} &
    \multicolumn{1}{c}{\footnotesize ETRPCA} &
    \multicolumn{1}{c}{\footnotesize t-CTV} &
    \multicolumn{1}{c}{\footnotesize MTTD} &
    \multicolumn{1}{c}{\footnotesize LRTFR} &
    \multicolumn{1}{c}{\footnotesize ONTRPCA} &
    \multicolumn{1}{c}{\footnotesize BCP-RPCC} &
    \multicolumn{1}{c}{\footnotesize Ground truth}
  \end{tabular}
  \caption{Foreground extraction for the Turnpike
    dataset. Row 1: Frame 10. Row 2: Frame 25. Row 3: Frame 50.}
  \label{fig:FMTurnpike}
\end{figure*}

\begin{figure*}[t]
  \centering
  \setlength{\tabcolsep}{0.1mm}
  \begin{tabular}{cccccccccc}
    \includegraphics[width=\smallimage]{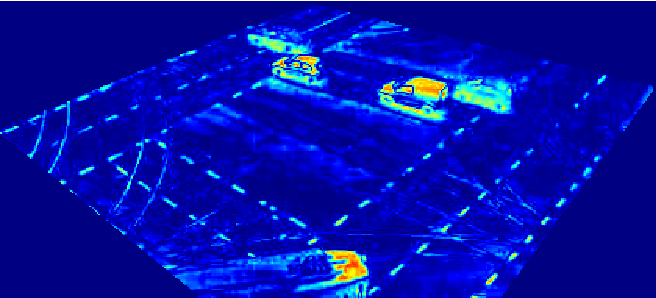} &
    \includegraphics[width=\smallimage]{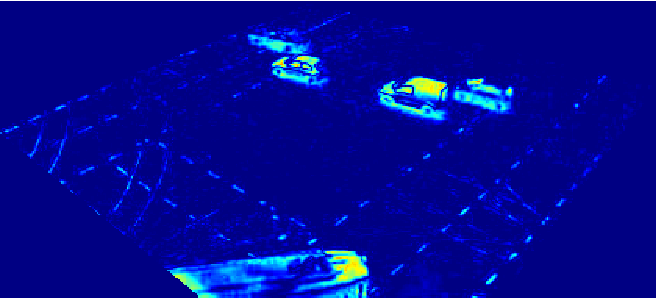} &
    \includegraphics[width=\smallimage]{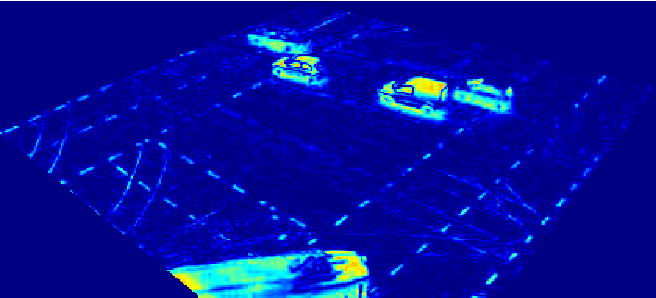} &
    \includegraphics[width=\smallimage]{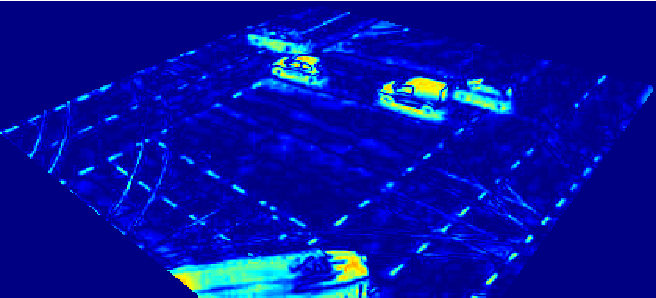} &
    \includegraphics[width=\smallimage]{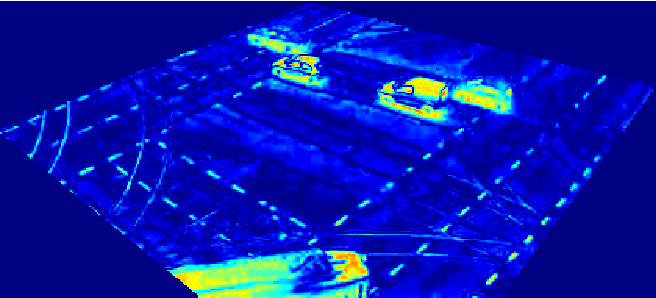} &
    \includegraphics[width=\smallimage]{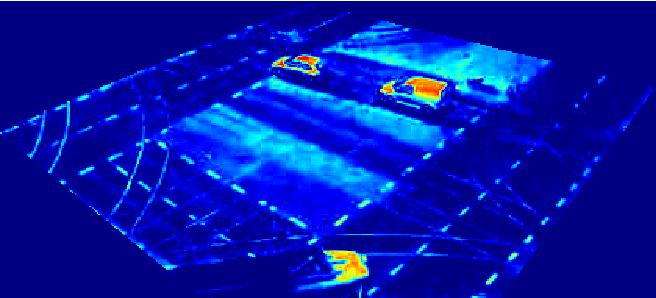} &
    \includegraphics[width=\smallimage]{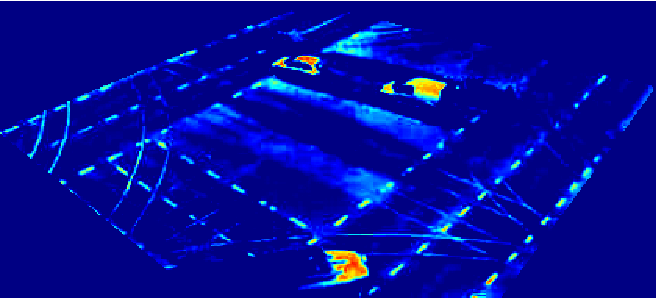} &
    \includegraphics[width=\smallimage]{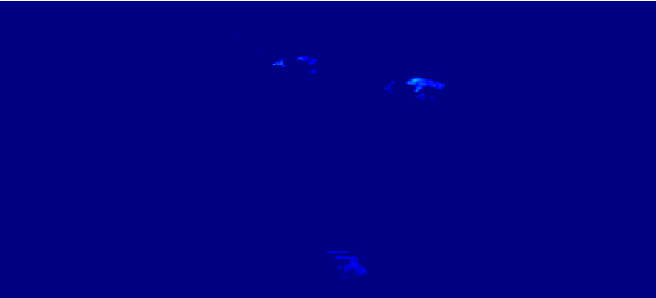} &
    \includegraphics[width=\smallimage]{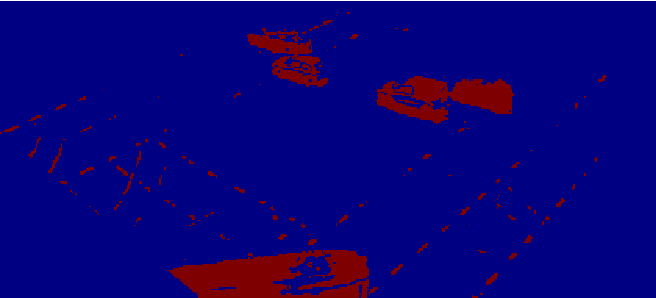} &
    \includegraphics[width=\smallimage]{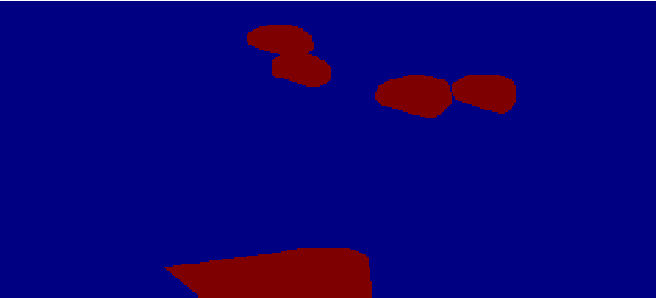}
    \\
    \includegraphics[width=\smallimage]{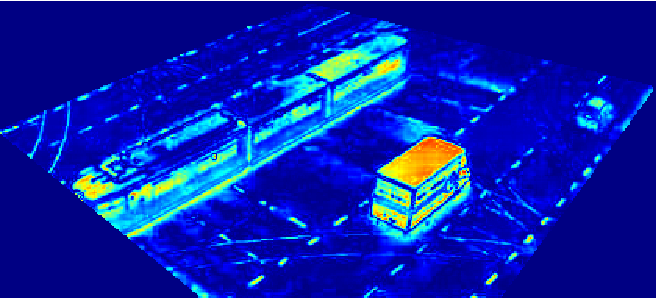} &
    \includegraphics[width=\smallimage]{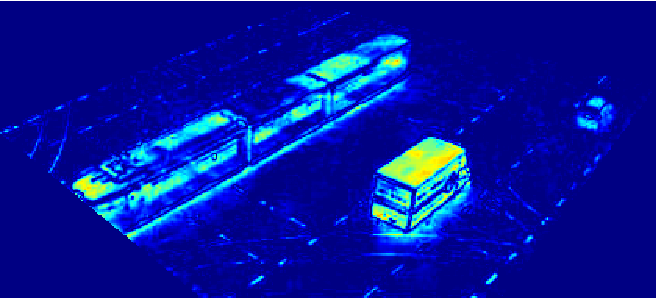} &
    \includegraphics[width=\smallimage]{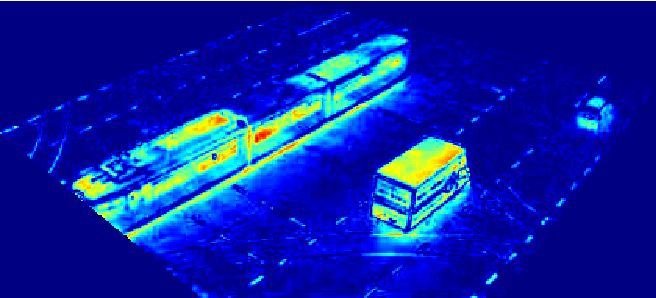} &
    \includegraphics[width=\smallimage]{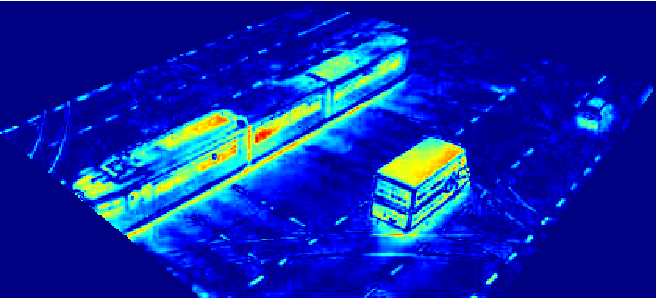} &
    \includegraphics[width=\smallimage]{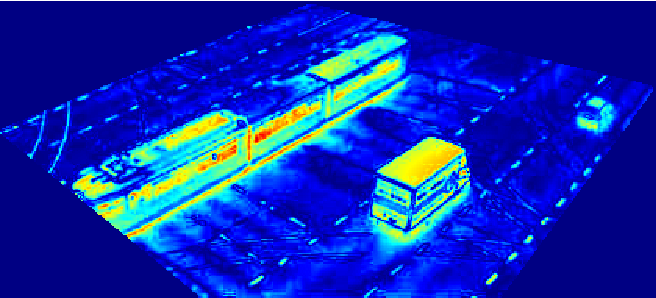} &
    \includegraphics[width=\smallimage]{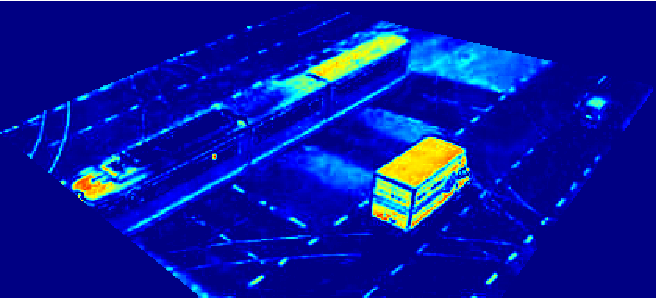} &
    \includegraphics[width=\smallimage]{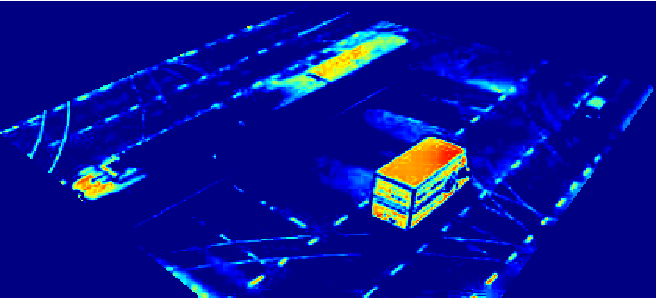} &
    \includegraphics[width=\smallimage]{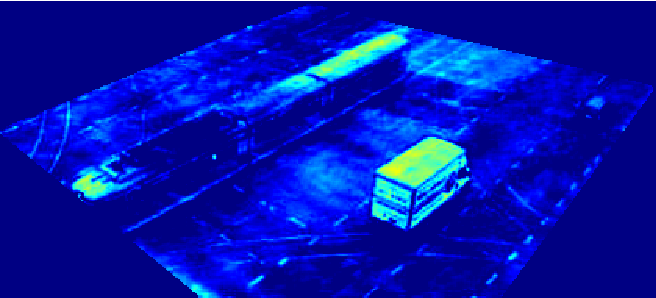} &
    \includegraphics[width=\smallimage]{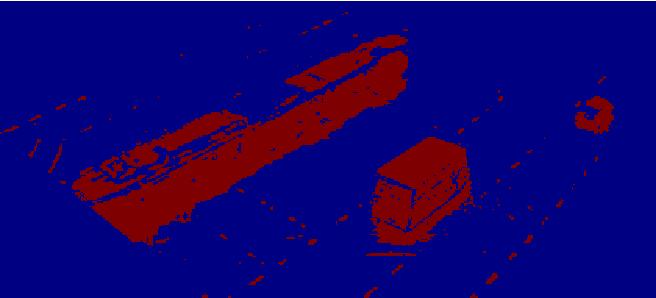} &
    \includegraphics[width=\smallimage]{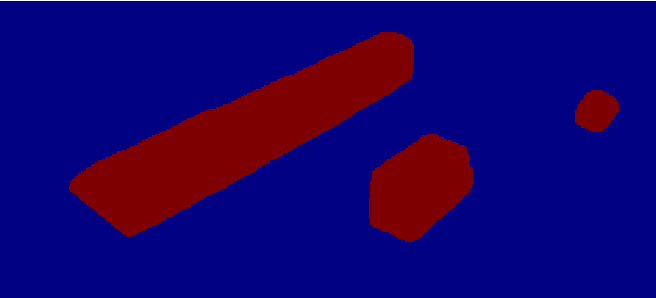}
    \\
    \includegraphics[width=\smallimage]{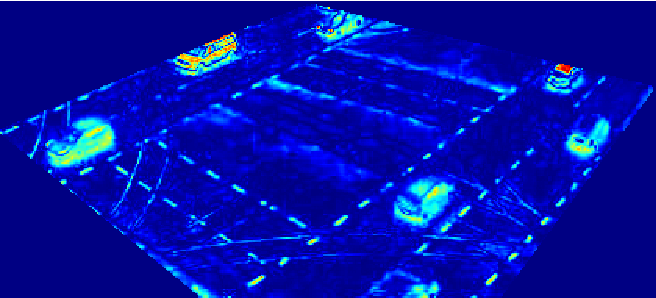} &
    \includegraphics[width=\smallimage]{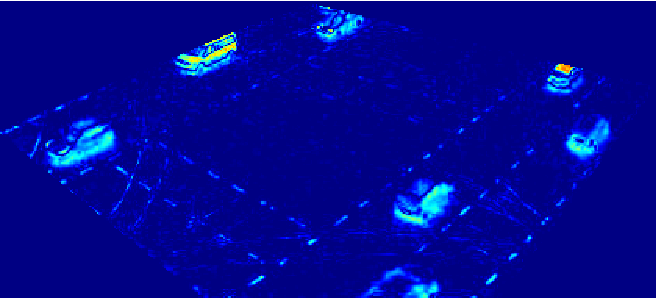} &
    \includegraphics[width=\smallimage]{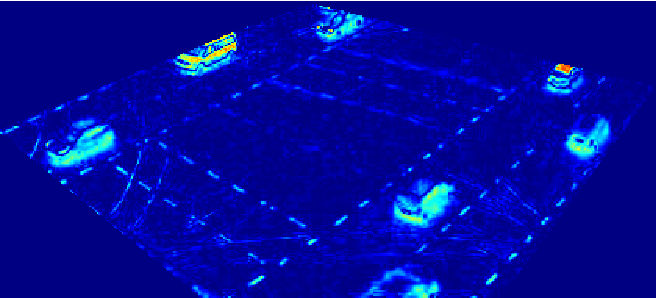} &
    \includegraphics[width=\smallimage]{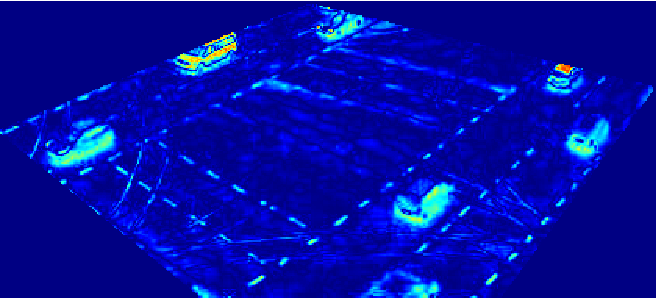} &
    \includegraphics[width=\smallimage]{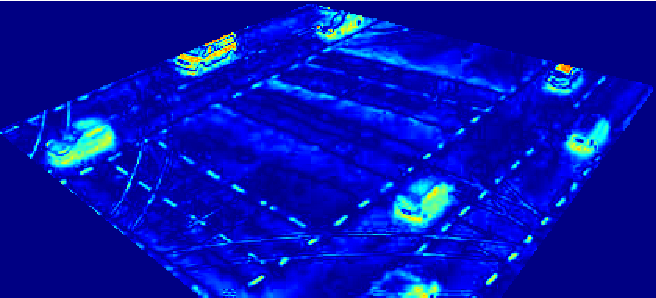} &
    \includegraphics[width=\smallimage]{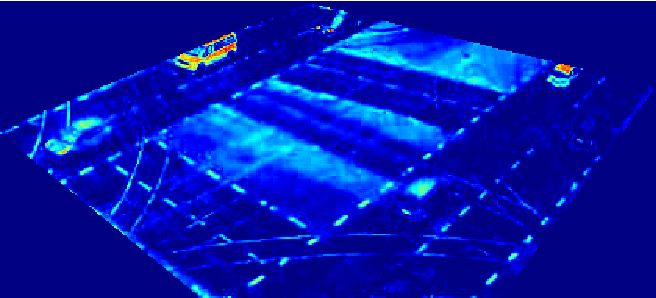} &
    \includegraphics[width=\smallimage]{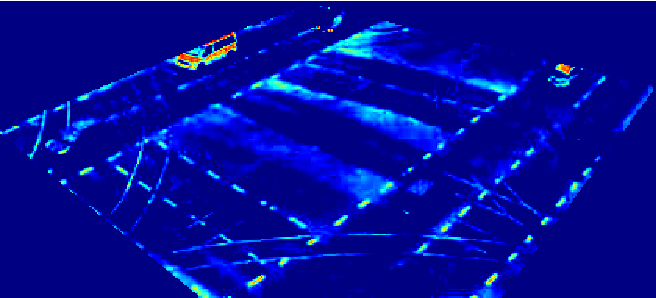} &
    \includegraphics[width=\smallimage]{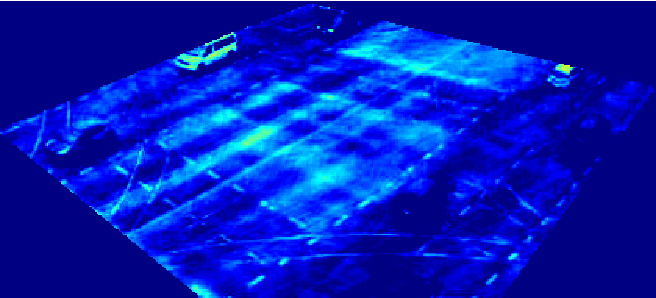} &
    \includegraphics[width=\smallimage]{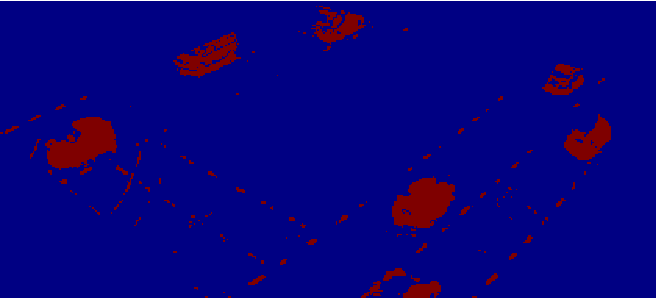} &
    \includegraphics[width=\smallimage]{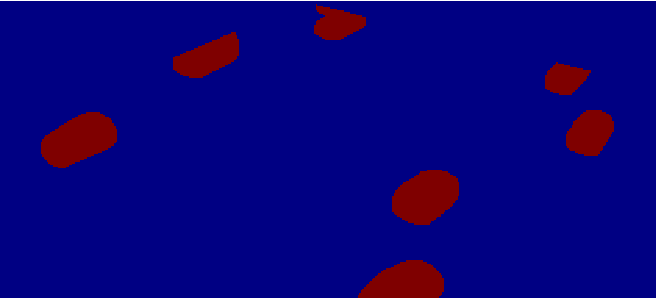}
    \\
    \multicolumn{1}{c}{\footnotesize PCP} &
    \multicolumn{1}{c}{\footnotesize KBR} &
    \multicolumn{1}{c}{\footnotesize TRPCA} &
    \multicolumn{1}{c}{\footnotesize ETRPCA} &
    \multicolumn{1}{c}{\footnotesize t-CTV} &
    \multicolumn{1}{c}{\footnotesize MTTD} &
    \multicolumn{1}{c}{\footnotesize LRTFR} &
    \multicolumn{1}{c}{\footnotesize ONTRPCA} &
    \multicolumn{1}{c}{\footnotesize BCP-RPCC} &
    \multicolumn{1}{c}{\footnotesize Ground truth}
  \end{tabular}
  \caption{Foreground extraction for the Crossroad
    dataset. Row 1: Frame 10. Row 2: Frame 45. Row 3: Frame 50.}
  \label{fig:FMCrossroad}
\end{figure*}

\begin{figure*}[t]
  \centering
  \setlength{\tabcolsep}{0.1mm}
  \begin{tabular}{ccccccccccc}
    \includegraphics[width=\smallimage]{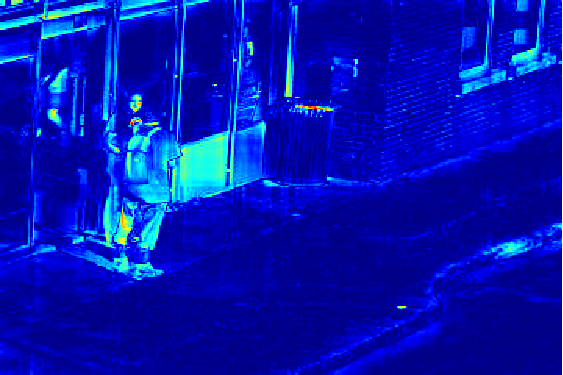} &
    \includegraphics[width=\smallimage]{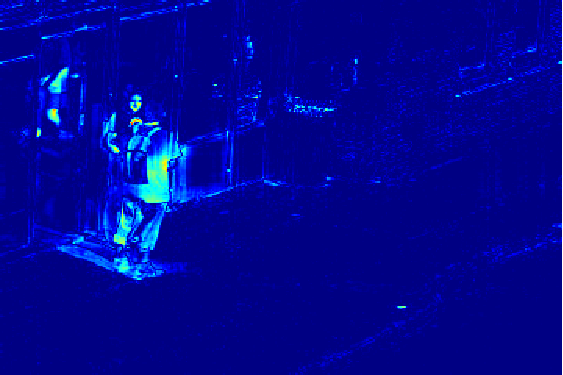} &
    \includegraphics[width=\smallimage]{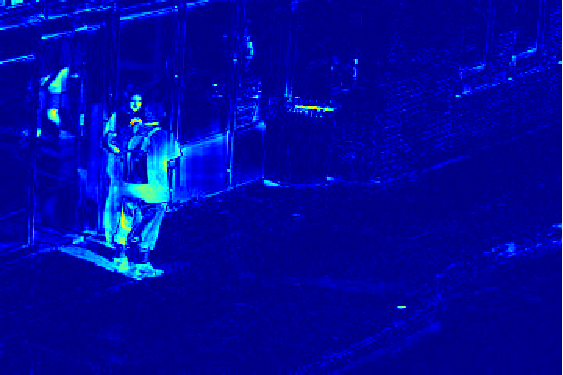} &
    \includegraphics[width=\smallimage]{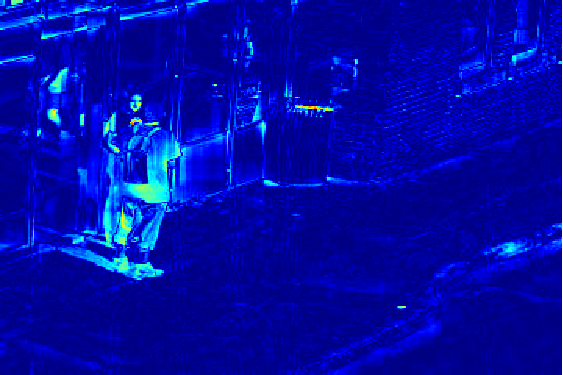} &
    \includegraphics[width=\smallimage]{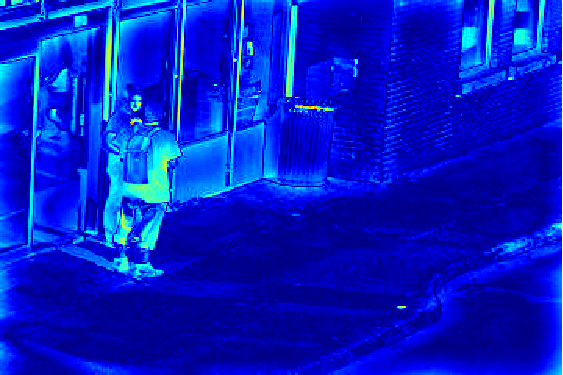} &
    \includegraphics[width=\smallimage]{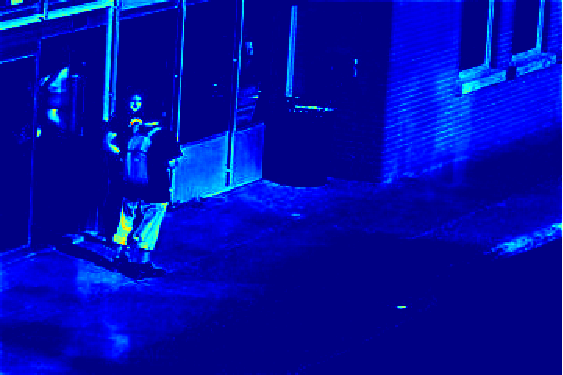} &
    \includegraphics[width=\smallimage]{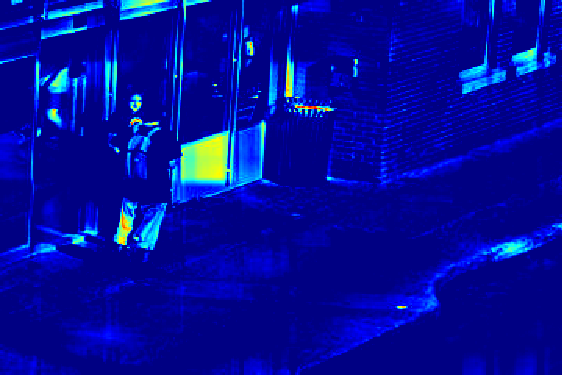} &
    \includegraphics[width=\smallimage]{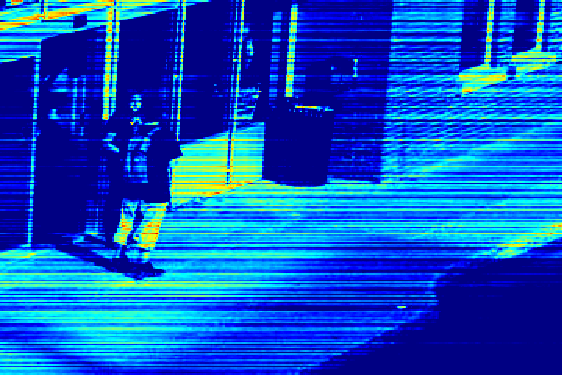} &
    \includegraphics[width=\smallimage]{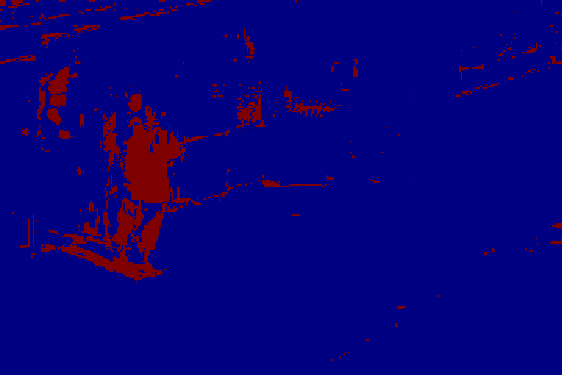} &
    \includegraphics[width=\smallimage]{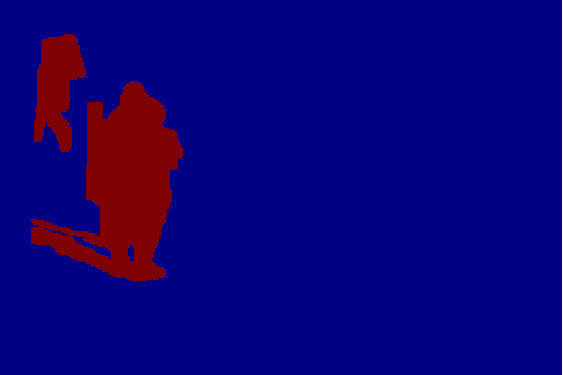}
    \\
    \includegraphics[width=\smallimage]{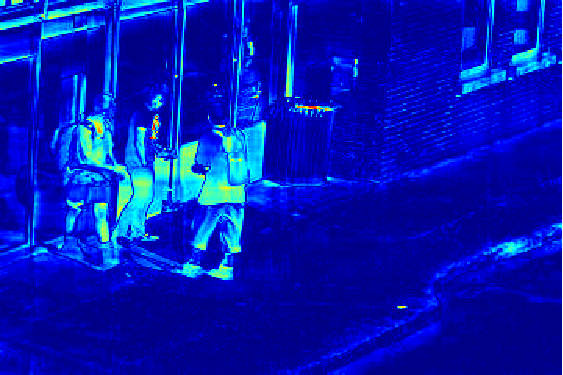} &
    \includegraphics[width=\smallimage]{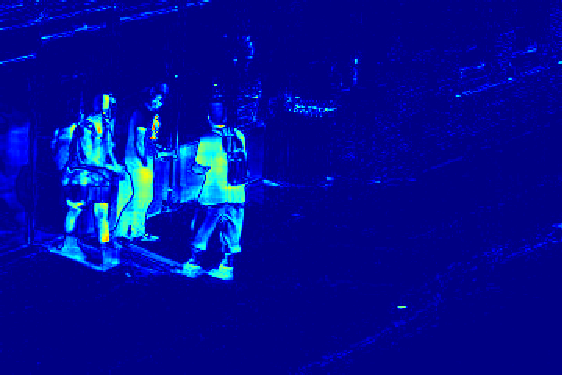} &
    \includegraphics[width=\smallimage]{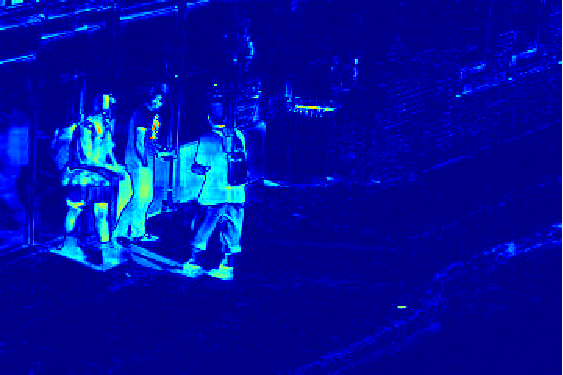} &
    \includegraphics[width=\smallimage]{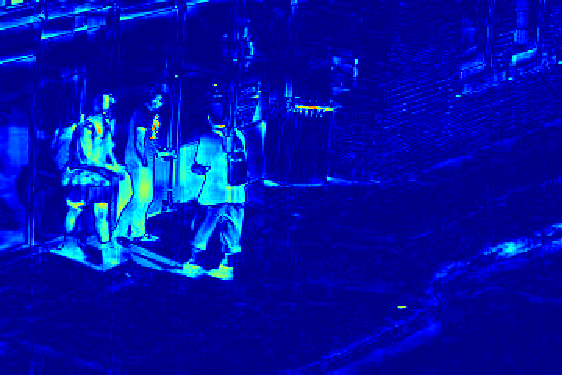} &
    \includegraphics[width=\smallimage]{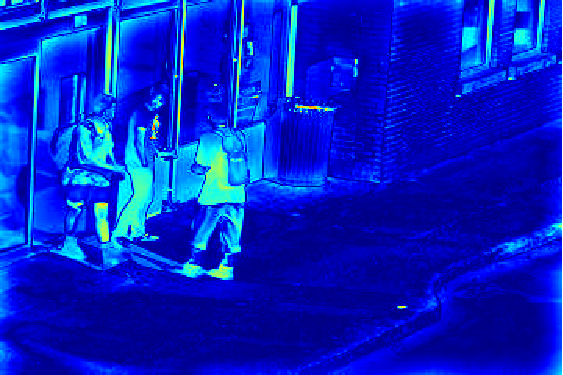} &
    \includegraphics[width=\smallimage]{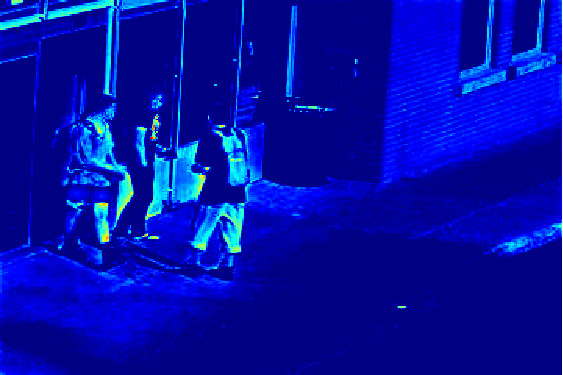} &
    \includegraphics[width=\smallimage]{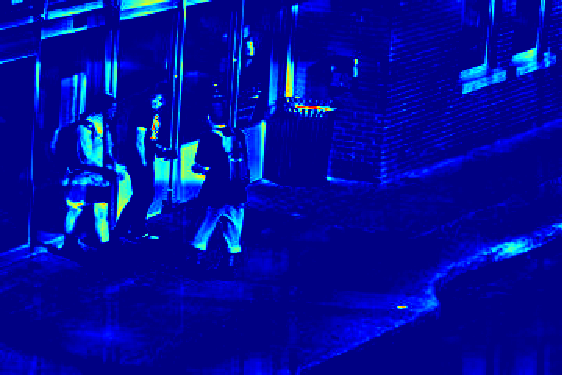} &
    \includegraphics[width=\smallimage]{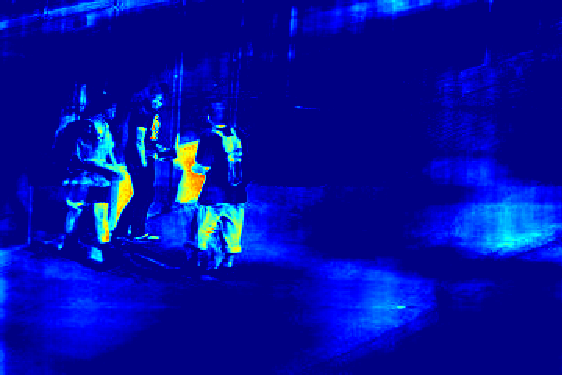} &
    \includegraphics[width=\smallimage]{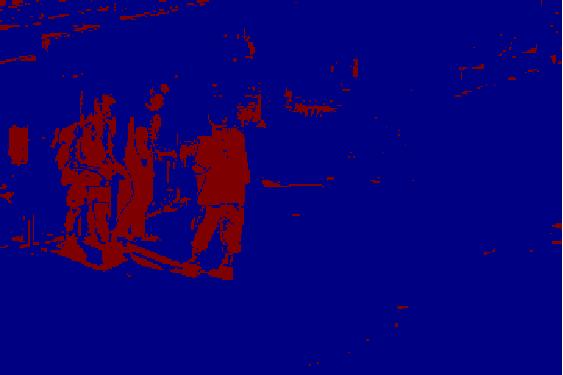} &
    \includegraphics[width=\smallimage]{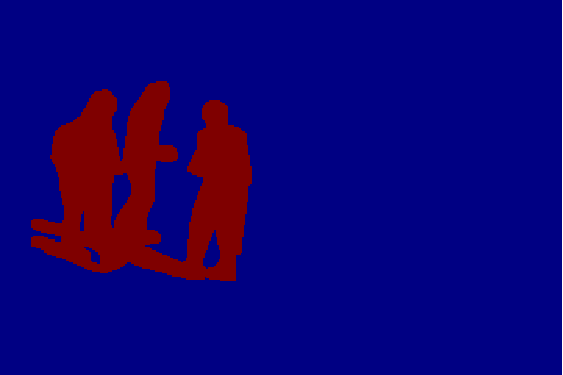}
    \\
    \includegraphics[width=\smallimage]{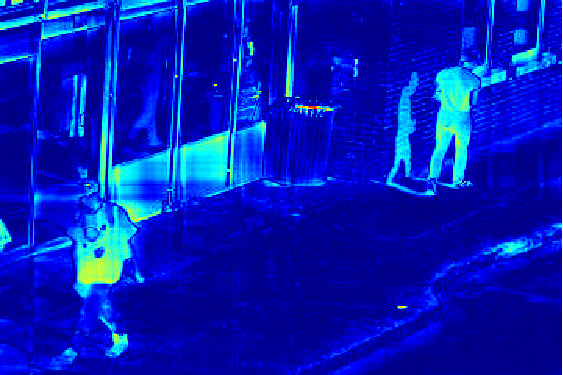} &
    \includegraphics[width=\smallimage]{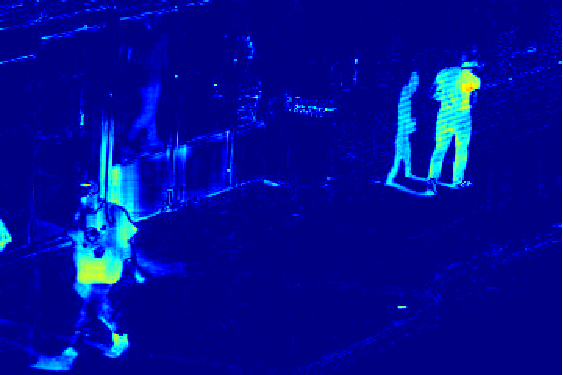} &
    \includegraphics[width=\smallimage]{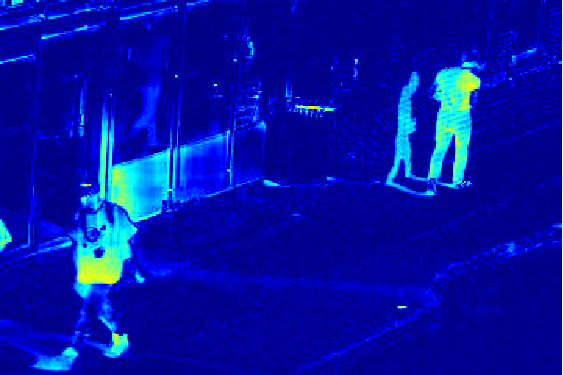} &
    \includegraphics[width=\smallimage]{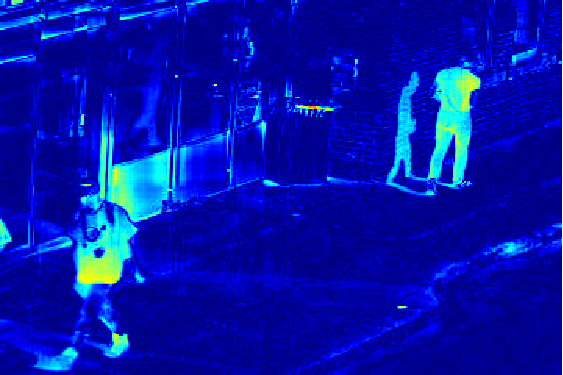} &
    \includegraphics[width=\smallimage]{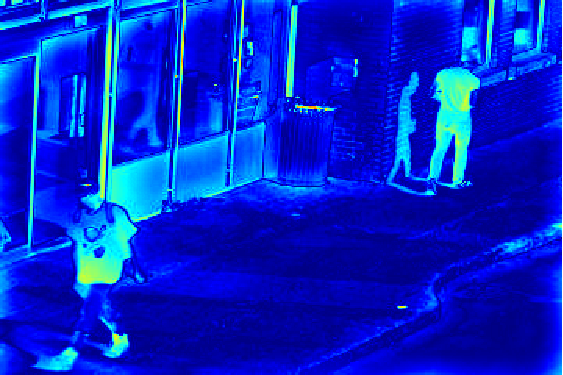} &
    \includegraphics[width=\smallimage]{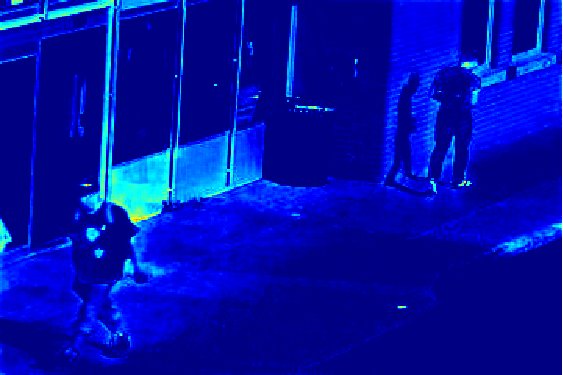} &
    \includegraphics[width=\smallimage]{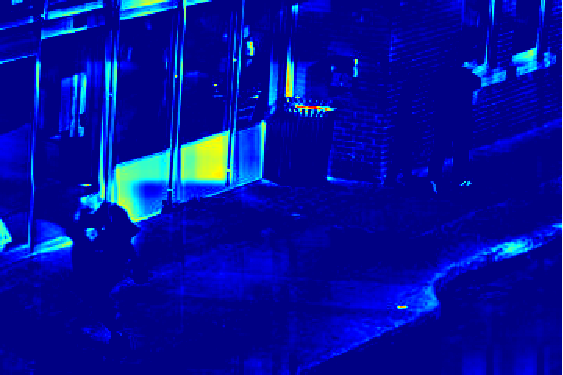} &
    \includegraphics[width=\smallimage]{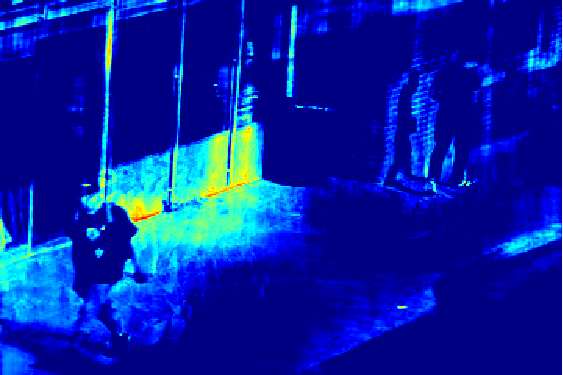} &
    \includegraphics[width=\smallimage]{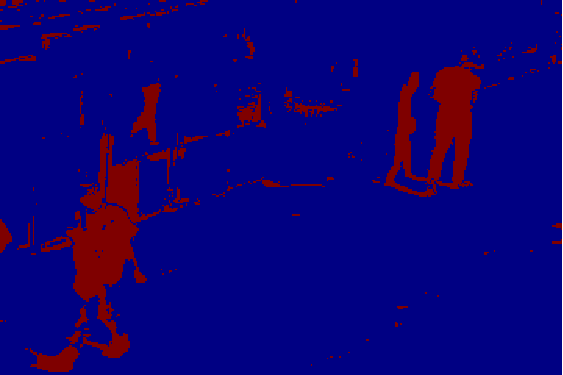} &
    \includegraphics[width=\smallimage]{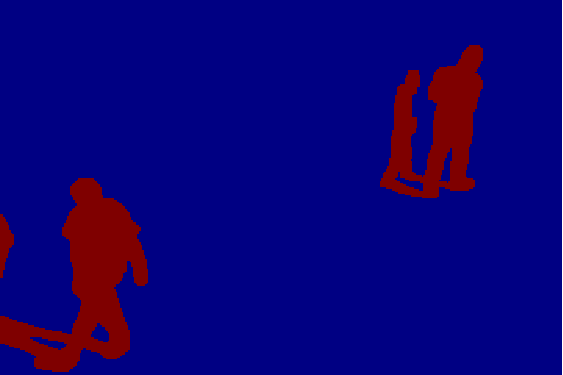}
    \\
    \multicolumn{1}{c}{\footnotesize PCP} &
    \multicolumn{1}{c}{\footnotesize KBR} &
    \multicolumn{1}{c}{\footnotesize TRPCA} &
    \multicolumn{1}{c}{\footnotesize ETRPCA} &
    \multicolumn{1}{c}{\footnotesize t-CTV} &
    \multicolumn{1}{c}{\footnotesize MTTD} &
    \multicolumn{1}{c}{\footnotesize LRTFR} &
    \multicolumn{1}{c}{\footnotesize ONTRPCA} &
    \multicolumn{1}{c}{\footnotesize BCP-RPCC} &
    \multicolumn{1}{c}{\footnotesize Ground truth}
  \end{tabular}
  \caption{Foreground extraction for the Busstation
    dataset. Row 1: Frame 1. Row 2: Frame 25. Row 3: Frame 50.}
  \label{fig:FMBusstation}
\end{figure*}

\begin{table}[t]
  \centering
  \renewcommand{\arraystretch}{1.25}
  \caption{Foreground Extraction for Highway. Best results are in boldface.}
  \begin{tabular}{c|ccc}
    \toprule
    \hline
    \multirow{2}{*}{\textbf{Methods}} &
    \multicolumn{3}{c}{\textbf{Metrics}} \\ \cline{2-4}
    & \multicolumn{1}{c|}{$\text{AUC}_{\operatorname{F1}}\uparrow$} &
    \multicolumn{1}{c|}{$\text{AUC}_{\operatorname{IoU}}\uparrow$} &
    $\text{AUC}_{\operatorname{Fa}}\downarrow$ \\ \hline
    PCP
    & \multicolumn{1}{c|}{0.2395} & \multicolumn{1}{c|}{0.1543} & 0.0292 \\
    KBR
    & \multicolumn{1}{c|}{0.3919} & \multicolumn{1}{c|}{0.2862} & \textbf{0.0060} \\
    TRPCA
    & \multicolumn{1}{c|}{0.3905} & \multicolumn{1}{c|}{0.2779} & 0.0138 \\
    ETRPCA
    & \multicolumn{1}{c|}{0.3607} & \multicolumn{1}{c|}{0.2498} & 0.0197 \\
    t-CTV
    & \multicolumn{1}{c|}{0.3470} & \multicolumn{1}{c|}{0.2310} & 0.0455 \\
    MTTD
    & \multicolumn{1}{c|}{0.1642} & \multicolumn{1}{c|}{0.0957} & 0.0290 \\
    LRTFR
    & \multicolumn{1}{c|}{0.1105} & \multicolumn{1}{c|}{0.0611} & 0.0185 \\
    ONTRPCA
    & \multicolumn{1}{c|}{0.1069} & \multicolumn{1}{c|}{0.0633} & 0.0279 \\
    BCP-RPCC
    & \multicolumn{1}{c|}{\textbf{0.7247}} & \multicolumn{1}{c|}{\textbf{0.5682}} & 0.0174 \\
    \hline
    \bottomrule
  \end{tabular}%
  \label{tab:FMHighway}
\end{table}

\begin{table}[t]
  \centering
  \renewcommand{\arraystretch}{1.25}
  \caption{Foreground Extraction for Turnpike. Best results are in boldface.}
  \begin{tabular}{c|ccc}
    \toprule
    \hline
    \multirow{2}{*}{\textbf{Methods}} &
    \multicolumn{3}{c}{\textbf{Metrics}} \\ \cline{2-4}
    & \multicolumn{1}{c|}{$\text{AUC}_{\operatorname{F1}}\uparrow$} &
    \multicolumn{1}{c|}{$\text{AUC}_{\operatorname{IoU}}\uparrow$} &
    $\text{AUC}_{\operatorname{Fa}}\downarrow$ \\ \hline
    PCP&
    \multicolumn{1}{c|}{0.2501} & \multicolumn{1}{c|}{0.1527} & 0.0713 \\ \hline
    KBR&
    \multicolumn{1}{c|}{0.2466} & \multicolumn{1}{c|}{0.1561} & 0.0394 \\ \hline
    TRPCA&
    \multicolumn{1}{c|}{0.2791} & \multicolumn{1}{c|}{0.1765} & 0.0591 \\\hline
    ETRPCA&
    \multicolumn{1}{c|}{0.2758} & \multicolumn{1}{c|}{0.1736} & 0.0648 \\\hline
    t-CTV
    & \multicolumn{1}{c|}{0.2253} & \multicolumn{1}{c|}{0.1389} & 0.0706 \\\hline
    MTTD
    & \multicolumn{1}{c|}{0.1820} & \multicolumn{1}{c|}{0.1082} & \textbf{0.0347} \\\hline
    LRTFR
    &  \multicolumn{1}{c|}{0.1836} & \multicolumn{1}{c|}{0.1044} & 0.0356 \\\hline
    ONTRPCA
    & \multicolumn{1}{c|}{0.1247} & \multicolumn{1}{c|}{0.0723} & 0.0114 \\\hline
    BCP-RPCC
      &  \multicolumn{1}{c|}{\textbf{0.5813}} & \multicolumn{1}{c|}{\textbf{0.4097}} & 0.0774 \\ \hline\bottomrule
    \end{tabular}%
  \label{tab:FMTurnpike}
\end{table}

\begin{table}[t]
  \centering
  \renewcommand{\arraystretch}{1.25}
  \caption{Foreground Extraction for Crossroad. Best results are in boldface.}
  \begin{tabular}{c|ccc}
    \toprule
    \hline
    \multirow{2}{*}{\textbf{Methods}} &
    \multicolumn{3}{c}{\textbf{Metrics}} \\ \cline{2-4}
    & \multicolumn{1}{c|}{$\text{AUC}_{\operatorname{F1}}\uparrow$} &
    \multicolumn{1}{c|}{$\text{AUC}_{\operatorname{IoU}}\uparrow$} &
    $\text{AUC}_{\operatorname{Fa}}\downarrow$ \\ \hline
    PCP
    &  \multicolumn{1}{c|}{0.2202} & \multicolumn{1}{c|}{0.1361} & 0.0412 \\ \hline KBR
    &  \multicolumn{1}{c|}{0.2319} & \multicolumn{1}{c|}{0.1554} & \textbf{0.0095} \\ \hline
    TRPCA
    &  \multicolumn{1}{c|}{0.2854} & \multicolumn{1}{c|}{0.1949} & 0.0214 \\\hline
    ETRPCA
    &  \multicolumn{1}{c|}{0.2657} & \multicolumn{1}{c|}{0.1752} & 0.0328 \\\hline t
    -CTV
    &  \multicolumn{1}{c|}{0.2807} & \multicolumn{1}{c|}{0.1870} & 0.0510 \\\hline
    MTTD
    &  \multicolumn{1}{c|}{0.1361} & \multicolumn{1}{c|}{0.0760} & 0.0535 \\\hline
    LRTFR
    &  \multicolumn{1}{c|}{0.1476} & \multicolumn{1}{c|}{0.0822} & 0.0326 \\\hline
    ONTRPCA
    &  \multicolumn{1}{c|}{0.0674} & \multicolumn{1}{c|}{0.0363} & 0.0628 \\\hline
    BCP-RPCC
    &  \multicolumn{1}{c|}{\textbf{0.7077}} & \multicolumn{1}{c|}{\textbf{0.5476}} & 0.0126  \\
    \hline
    \bottomrule
    \end{tabular}%
  \label{tab:FMCrossroad}
\end{table}

\begin{table}[t]
  \centering
  \renewcommand{\arraystretch}{1.25}
  \caption{Foreground Extraction for Crossroad. Best results are in boldface.}
  \begin{tabular}{c|ccc}
    \toprule
    \hline
    \multirow{2}{*}{\textbf{Methods}} &
    \multicolumn{3}{c}{\textbf{Metrics}} \\ \cline{2-4}
    & \multicolumn{1}{c|}{$\text{AUC}_{\operatorname{F1}}\uparrow$} &
    \multicolumn{1}{c|}{$\text{AUC}_{\operatorname{IoU}}\uparrow$} &
    $\text{AUC}_{\operatorname{Fa}}\downarrow$ \\ \hline
      PCP
      &  \multicolumn{1}{c|}{0.1658} & \multicolumn{1}{c|}{0.1013} & 0.0631 \\ \hline KBR
      &  \multicolumn{1}{c|}{0.2105} & \multicolumn{1}{c|}{0.1404} & \textbf{0.0124} \\ \hline
      TRPCA
      &  \multicolumn{1}{c|}{0.2031} & \multicolumn{1}{c|}{0.1335} & 0.0265 \\\hline
      ETRPCA
      &  \multicolumn{1}{c|}{0.1923} & \multicolumn{1}{c|}{0.1229} & 0.0391 \\\hline
      t-CTV
      &  \multicolumn{1}{c|}{0.1461} & \multicolumn{1}{c|}{0.0874} & 0.0925 \\\hline
      MTTD
      &  \multicolumn{1}{c|}{0.0833} & \multicolumn{1}{c|}{0.0460} & 0.0476 \\\hline
      LRTFR
      &  \multicolumn{1}{c|}{0.0735} & \multicolumn{1}{c|}{0.0398} & 0.0403 \\\hline
      ONTRPCA
      &  \multicolumn{1}{c|}{0.0836} & \multicolumn{1}{c|}{0.0458} & 0.0370 \\\hline
      BCP-RPCC
      &  \multicolumn{1}{c|}{\textbf{0.6448}} & \multicolumn{1}{c|}{\textbf{0.4757}} & 0.0238 \\
      \hline
      \bottomrule
    \end{tabular}%
  \label{tab:FMBusstation}
\end{table}

\begin{figure}[t]
  \centering
  \setlength{\tabcolsep}{0.15mm}
  \begin{tabular}{ccc}
    \includegraphics[width=\smallplot]{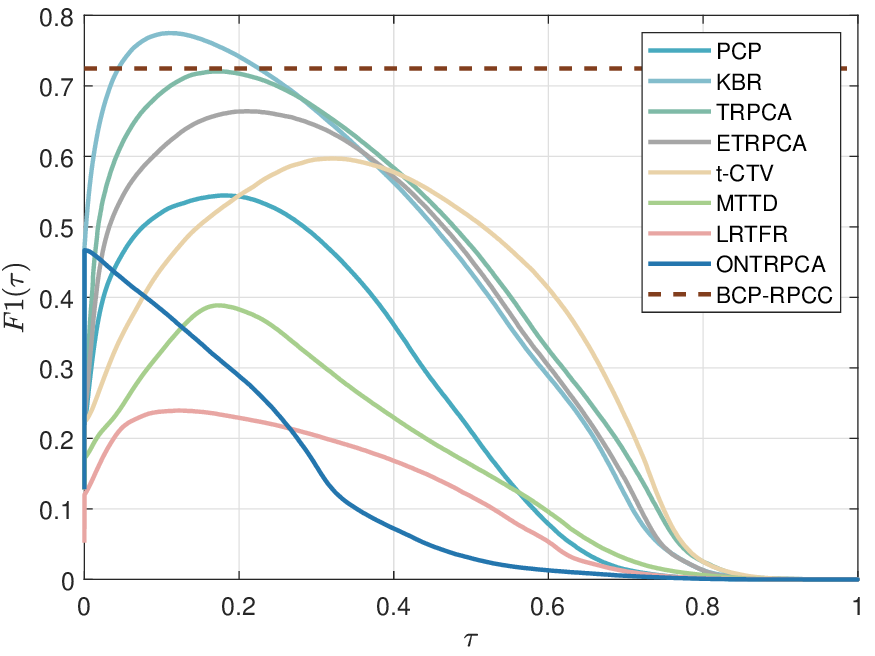} &
    \includegraphics[width=\smallplot]{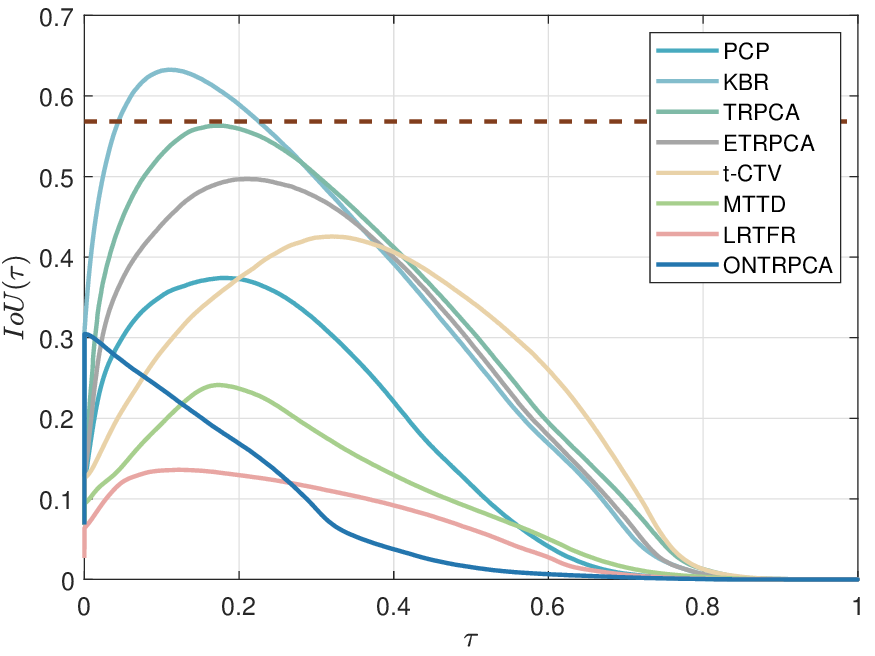} &
    \includegraphics[width=\smallplot]{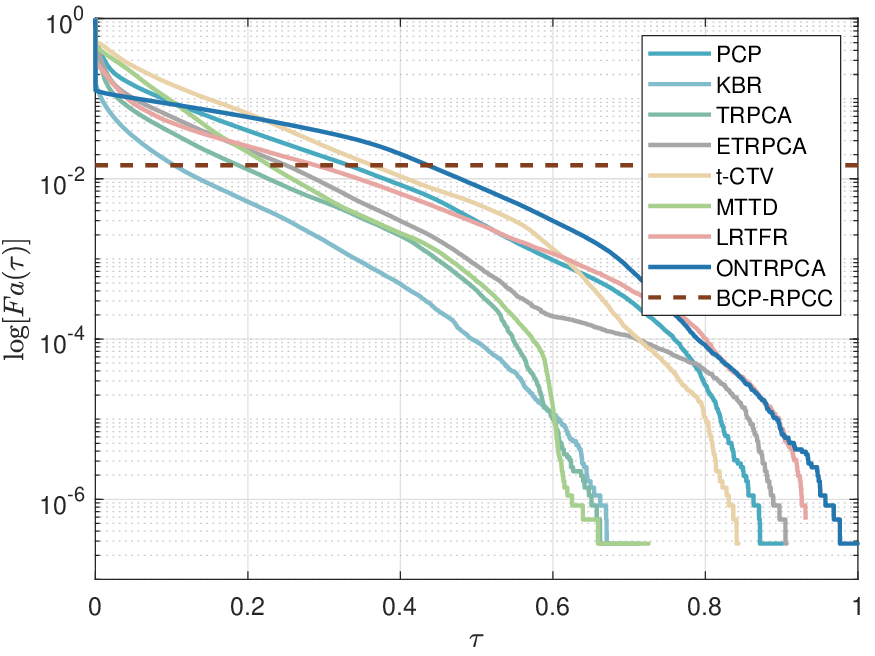}
    \\
    \includegraphics[width=\smallplot]{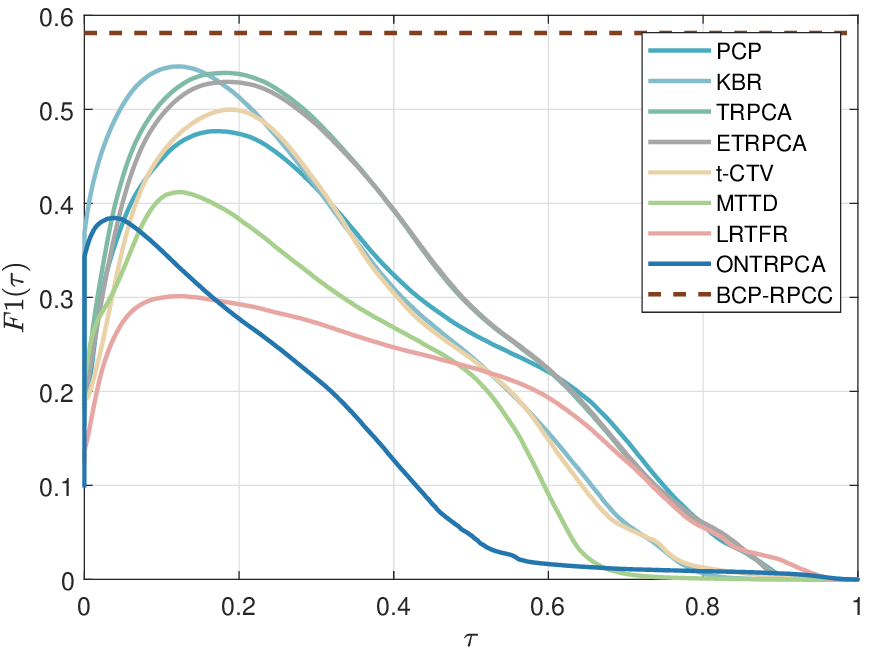} &
    \includegraphics[width=\smallplot]{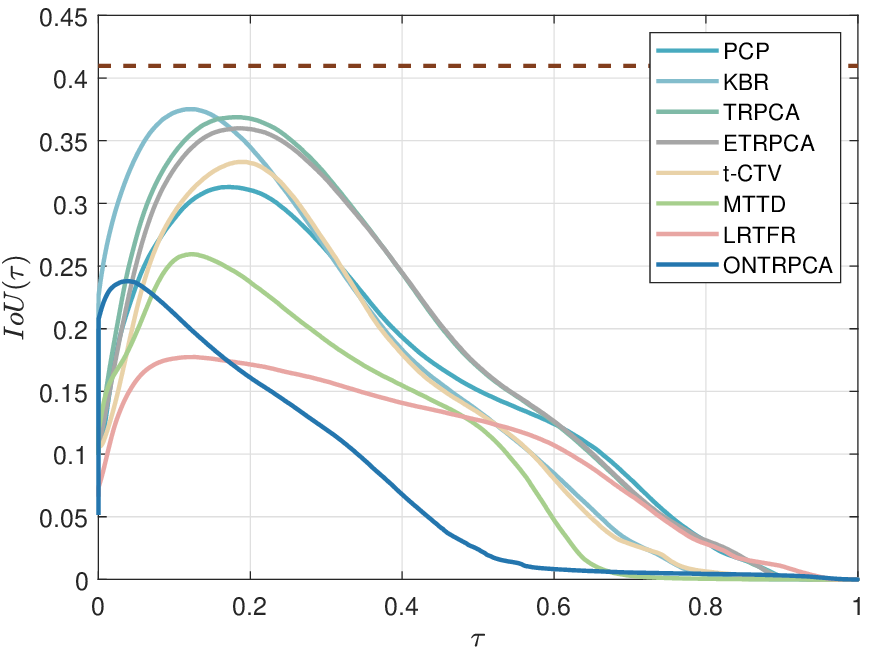} &
    \includegraphics[width=\smallplot]{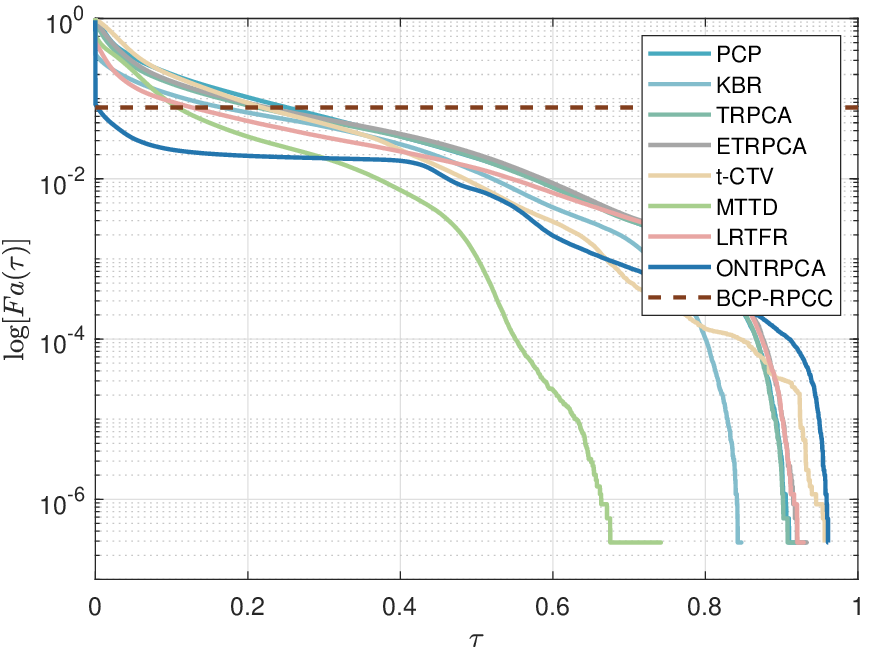}
    \\
    \includegraphics[width=\smallplot]{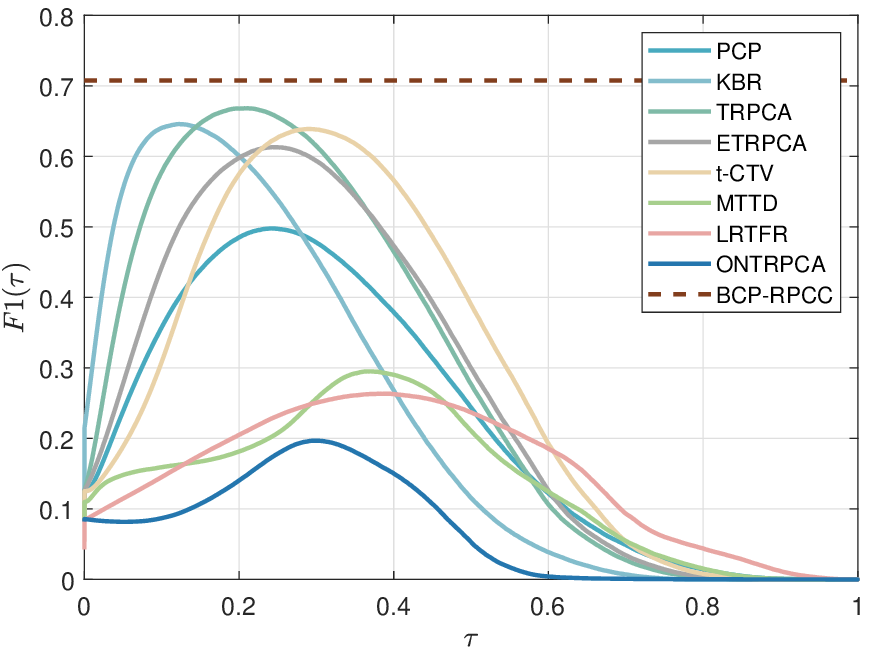} &
    \includegraphics[width=\smallplot]{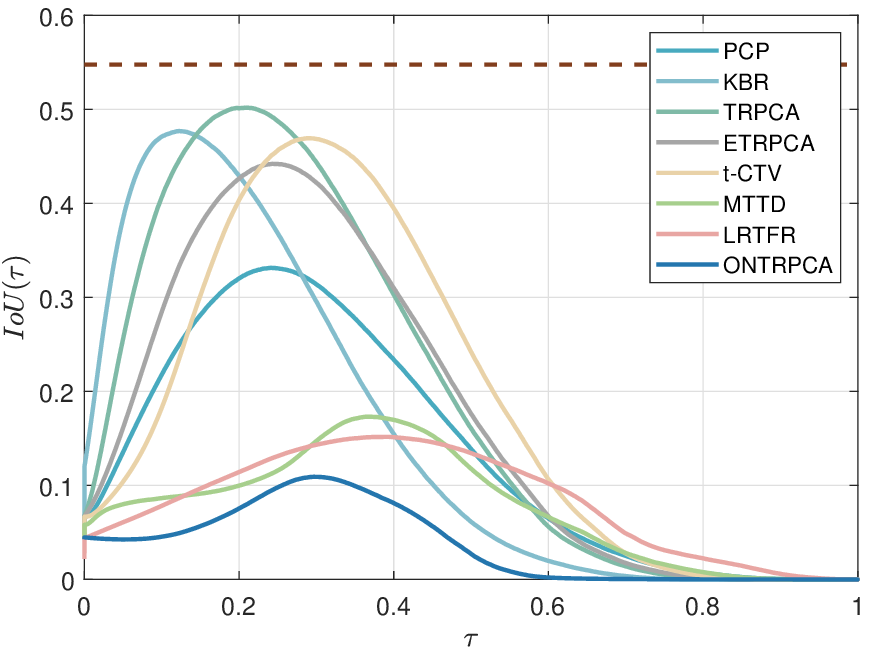} &
    \includegraphics[width=\smallplot]{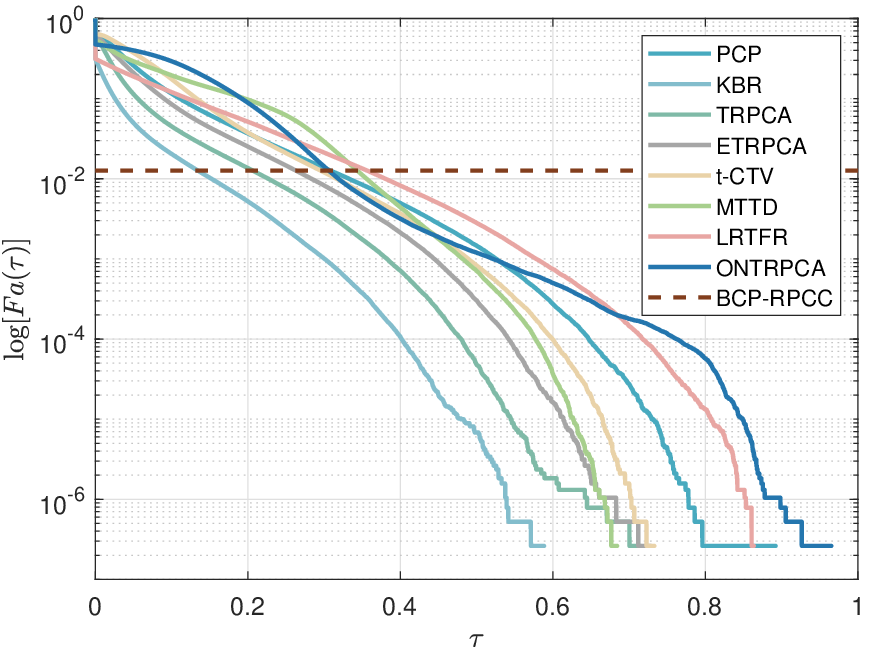}
    \\
    \includegraphics[width=\smallplot]{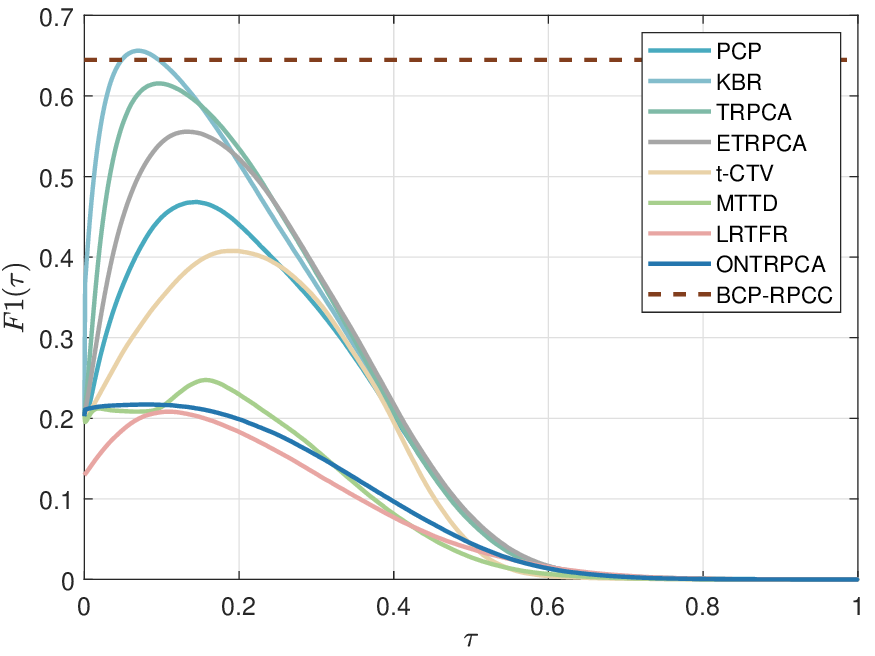} &
    \includegraphics[width=\smallplot]{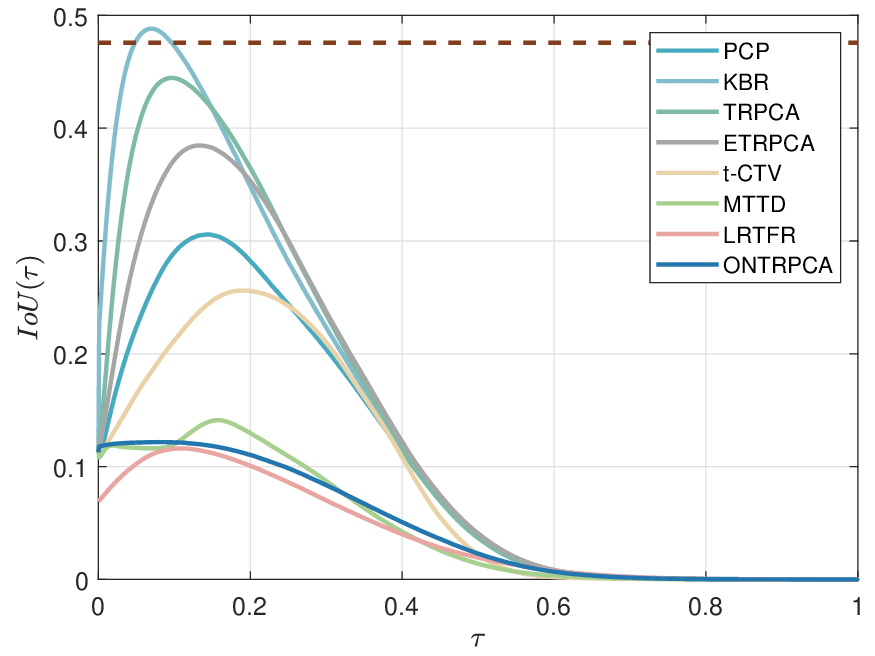} &
    \includegraphics[width=\smallplot]{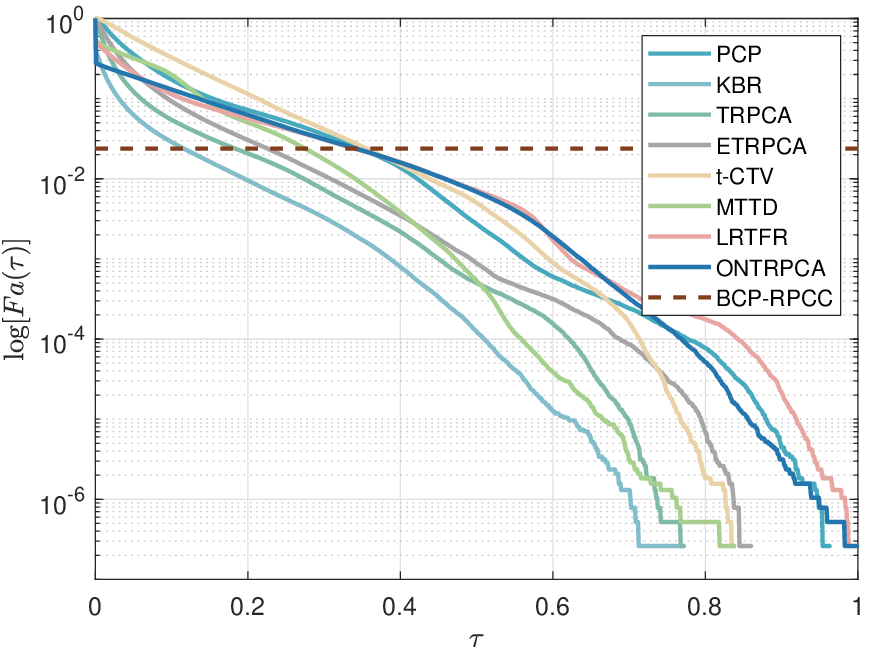}
  \end{tabular}
  \caption{Foreground-extraction performance in terms of
    $\operatorname{F1}(\tau)$, $\operatorname{IoU}(\tau)$, and
    $\operatorname{Fa}(\tau)$: Highway (row~1), Turnpike (row~2),
    Crossroad (row~3), and Busstation (row~4).}
  \label{fig:FMcurves}
\end{figure}

The foreground-extraction problem considers a 
video composed of a set of sparsely-supported
foreground objects moving against a relatively static
background. In the context of RPCC,
the desired foreground is $\mathcal{S}$, the low-rank video background
is $\mathcal{L}$, the mask separating foreground from background
is
$\mathtt{\Omega}$, and the observed video is
$\mathcal{Y} = \mathscr{P}_{\mathtt{\Omega}^{\bot}}\left[\mathcal{L}\right] + \mathcal{S}$. We then employ the BCP-RPCC of Alg.~\ref{alg:BCP-RPCC}
to estimate $\widehat{\mathcal{S}}$.

To conduct empirical results on real-world video, we use four color
videos from the CDnet dataset \cite{WJP2014} which are illustrated in
Fig.~\ref{fig:FgMdata}.  For each video, we draw 50 frames to create a
$50\times H\times W\times 3$ tensor, where $H$ and $W$ correspond to
the height and width of each frame, respectively, and $3$ corresponds
to RGB channels.  Each original video in the CDnet dataset is
accompanied by a manually segmented ground-truth foreground mask along
with a region of interest (ROI) indicating the possible spatial range of the
foreground objects. For the results here, $\mathcal{S}$ is forcibly confined
to lie with the specified ROI for each video.

We compare to several RPCA-driven methods from prior literature.
These include the original matrix-based PCP solution to RPCA presented
in \cite{CLM2011}, along with a number of the tensor-based approaches
discussed previously in Sec.~\ref{sec:backgroundrpca}---namely, KBR
\cite{XZM2017}, TRPCA \cite{LFC2020}, ETRPCA \cite{GZX2021}, t-CTV
\cite{WPQ2023}, MTTD \cite{FZL2023}, ONTRPCA \cite{FLL2025}, and LRTFR
\cite{LZL2024a}.

We note that, while BCP-RPCC implements a hard classifier as discussed
in Sec.~\ref{sec:convergence}, the RPCA-based methods to which we
compare are effectively soft classifiers which require a
subsequent threholding to definitively estimate the target foreground
component $\mathcal{S}$.  To the best of our knowledge, simultaneously
evaluating both hard and soft classifiers is a challenge that
has been largely unaddressed in the past. Consequently, 
we are forced to propose new measures to this end.
Specifically,
we first choose three common measures for hard classifiers, namely
the F1-score,
\begin{equation}
  \operatorname{F1}=\frac{2\times \text{Precision}\times
    \text{Recall}}{\text{Precision}+ \text{Recall}},
\end{equation}
the IoU from \eqref{eq:iou},
and the false-alarm rate,
\begin{equation}
  \operatorname{Fa}=\frac{\text{False
      Positive}}{\text{False Positive} + \text{True Negative}}.
\end{equation}
Then, for the soft classifiers, these three measures become functions
of the threshold $\tau$, which are denoted as
$\operatorname{F1}(\tau)$, $\operatorname{IoU}(\tau)$, and
$\operatorname{Fa}(\tau)$, respectively, and we measure
area under the
curve (AUC) for each, i.e.,
\begin{align}
    \text{AUC}_{\operatorname{F1}}&=\int_{0}^{1}\operatorname{F1}(\tau)\,d\tau,\\
    \text{AUC}_{\operatorname{IoU}}&=\int_{0}^{1}\operatorname{IoU}(\tau)\,d\tau,\\
    \text{AUC}_{\operatorname{Fa}}&=\int_{0}^{1}\operatorname{Fa}(\tau)\,d\tau .
\end{align}
We also note that, when comparing BCP-RPCC to the RPCA-based methods,
$\operatorname{F1}(\tau)$, $\operatorname{IoU}(\tau)$, and
$\operatorname{Fa}(\tau)$ are constant.


For BCP-RPCC, the main hyperparameters are the block sizes
$\left\{J_n\right\}_{n=1}^4$, the block numbers
$\left\{K_n\right\}_{n=1}^4$, the CP rank $R$, and the variance
$\sigma$ generating the noise $\mathcal{E}$.  For the
foreground-extraction problem, we cast each pixel as one
block. Accordingly, we have $J_1=J_2=J_3=1$ and $J_4=3$, which makes
$K_1=50$, $K_2=H$, $K_3=W$, and $K_4=1$. Prop.~\ref{Prop:HardClass}
concludes that a smaller $\sigma$ brings BCP-RPCC closer to a hard
classifier. However, its proof reveals that, if $\sigma=o\Bigl(E\bigl[
  \bigl(\widehat{\mathbf{y}}^{[K]}_k - \mathbf{l}^{[K]}_k\bigr)^T
  \bigl(\widehat{\mathbf{y}}^{[K]}_k - \mathbf{l}^{[K]}_k\bigr)
  \bigr]\Bigr)$, the model will classify all blocks as foreground,
thus resulting in complete failure. Hence, both $R$ and $\sigma$ must
be tuned carefully, which is visualized in
Fig.~\ref{fig:FgMTuning}. Holding $R=25$, we see that both F1 and IoU
reach peak performance when $\sigma$ is around $10^{-3}$. For this value, Fig.~\ref{fig:FgMTuning}
indicates monotonically improving performance with increasing
$R$ for Highway, Turnpike and Busstation. Accordingly, to balance solution quality with computational
cost while simultaneously considering the performance drop on Crossroad when $R>30$, we set $R=30$.


Performance of foreground extraction for all the methods under
consideration is depicted in
Figs.~\ref{fig:FMHighway}--\ref{fig:FMBusstation}. Apparent is that
only the proposed BCP-RPCC is provides binary predictions, which sets
it apart from the RPCA-driven methods. Additionally, the foreground
region detected by BCP-RPCC achieves the largest intersection with the
ground-truth map. Notably, in Figs.~\ref{fig:FMHighway} and
\ref{fig:FMTurnpike}, the MTTD, LRTFR, and ONTRPCA methods struggle to
identify car windows, while dark clothes in
Fig.~\ref{fig:FMBusstation} present similar difficulties for these
same techniques.  BCP-RPCC nevertheless offers clear foreground
separation in both cases.

Quantitative performance is tabulated in
Tables~\ref{tab:FMHighway}--\ref{tab:FMBusstation} which consistently
show BCP-RPCC outperforming the other techniques for the
$\operatorname{AUC}_{\operatorname{F1}}$ and
$\operatorname{AUC}_{\operatorname{IoU}}$ measures. Furthermore,
Fig.~\ref{fig:FMcurves} plots the $\operatorname{F1}(\tau)$,
$\operatorname{IoU}(\tau)$, and $\operatorname{Fa}(\tau)$
curves. Therein, we see that the constant BCP-RPCC curve is larger
than most of the other curves, which all feature a relatively sharp
peak. The tightly-peaked nature of these curves indicates that it is
exceedingly difficult to choose an appropriate soft-classifier
threshold for these techniques in practice since real-world scenarios
would not permit access to ground truth.  Thus, while BCP-RPCC is
marginally outperformed in terms of false-alarm rate, this pales in
comparison to its substantial advantages in
$\operatorname{AUC}_{\operatorname{F1}}$ and
$\operatorname{AUC}_{\operatorname{IoU}}$ performance.

\subsection{Hyperspectral Anomaly Detection}
\label{sec:resultsanomaly}
\begin{figure}[t]
  \centering
  \setlength{\tabcolsep}{0.1mm}
  \begin{tabular}{cccc}
    \multicolumn{1}{c}{\makecell{\footnotesize Belcher\\\footnotesize{$150\times150\times150$}}} &
    \multicolumn{1}{c}{\makecell{\footnotesize Urban\\\footnotesize{$100\times100\times204$}}} &
    \multicolumn{1}{c}{\makecell{\footnotesize Beach\\\footnotesize{$150\times150\times102$}}} &
    \multicolumn{1}{c}{\makecell{\footnotesize Salinas\\\footnotesize{$100\times100\times204$}}}
    \\
    \includegraphics[width=\mediumimage]{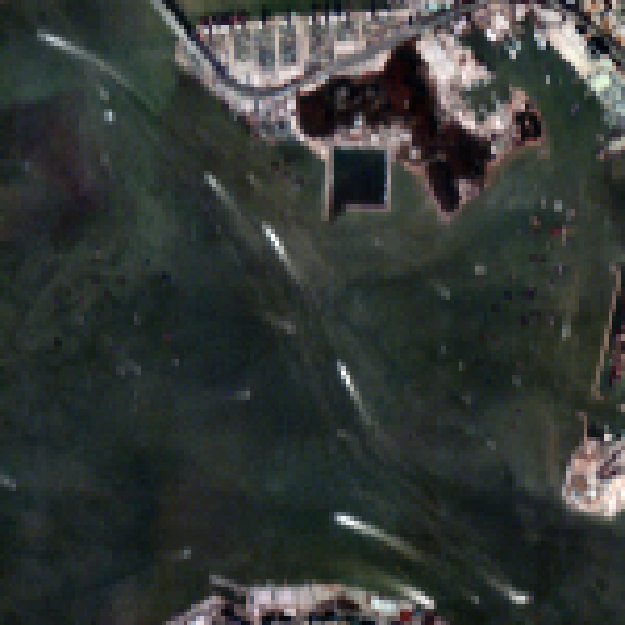} &
    \includegraphics[width=\mediumimage]{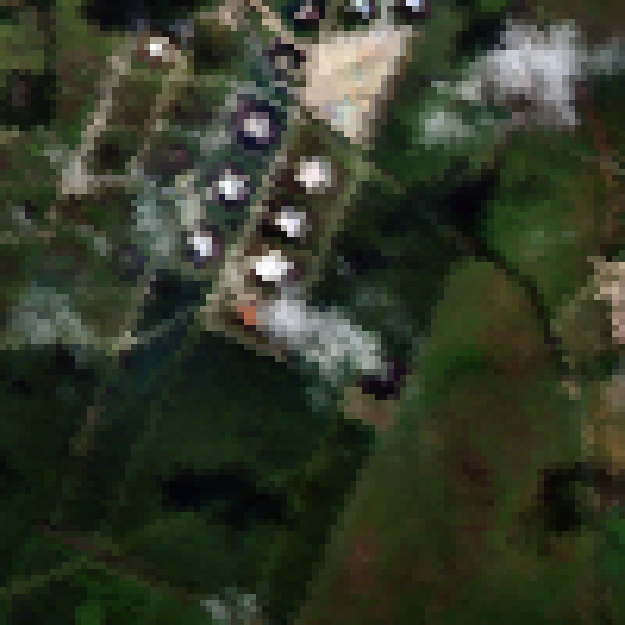} &
    \includegraphics[width=\mediumimage]{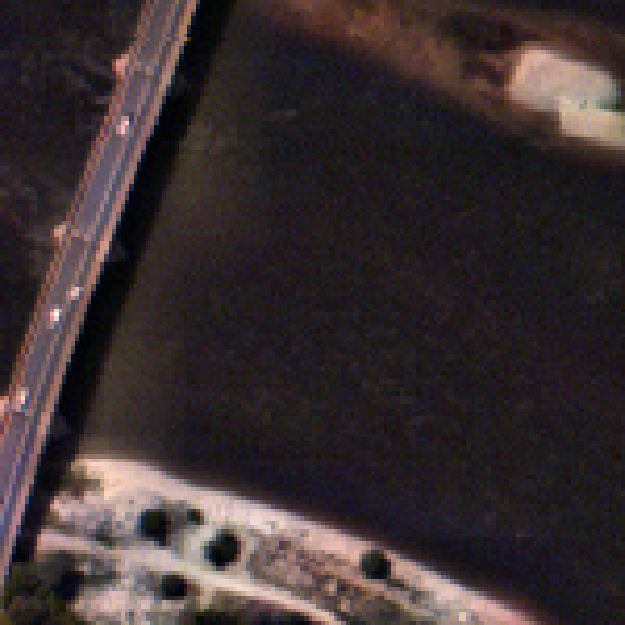} &
    \includegraphics[width=\mediumimage]{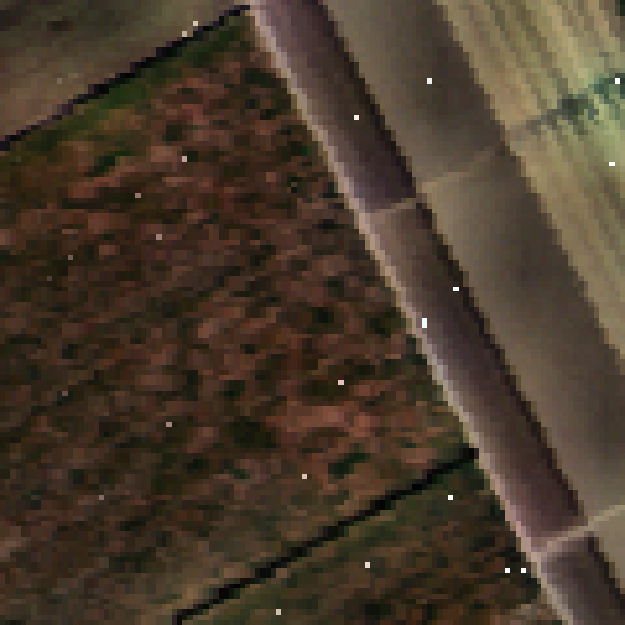}\\
    \includegraphics[width=\mediumimage]{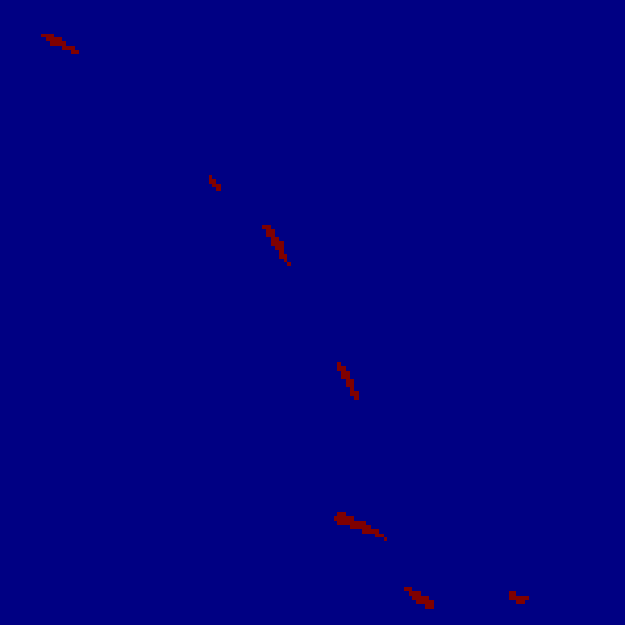} &
    \includegraphics[width=\mediumimage]{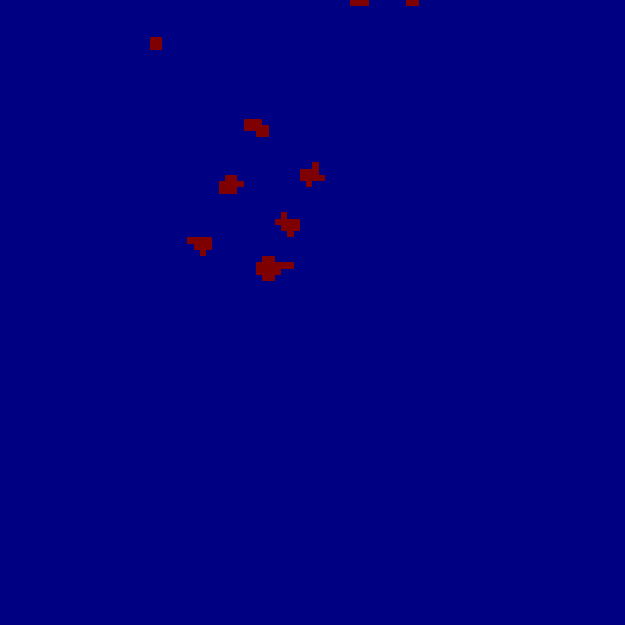} &
    \includegraphics[width=\mediumimage]{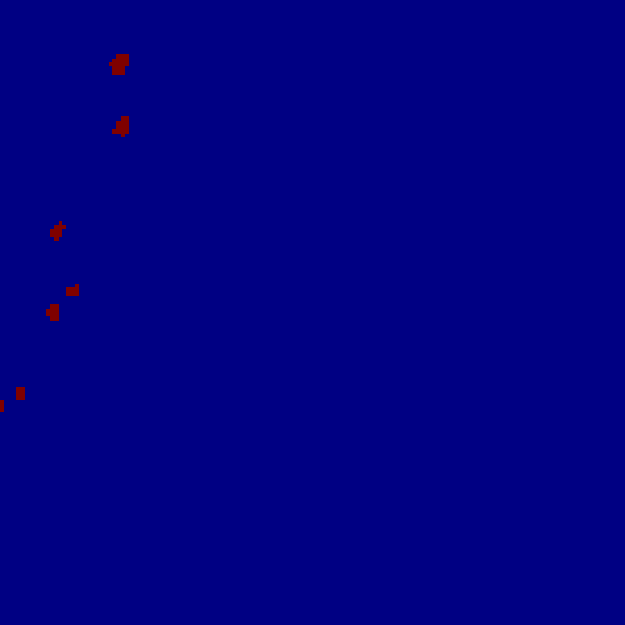} &
    \includegraphics[width=\mediumimage]{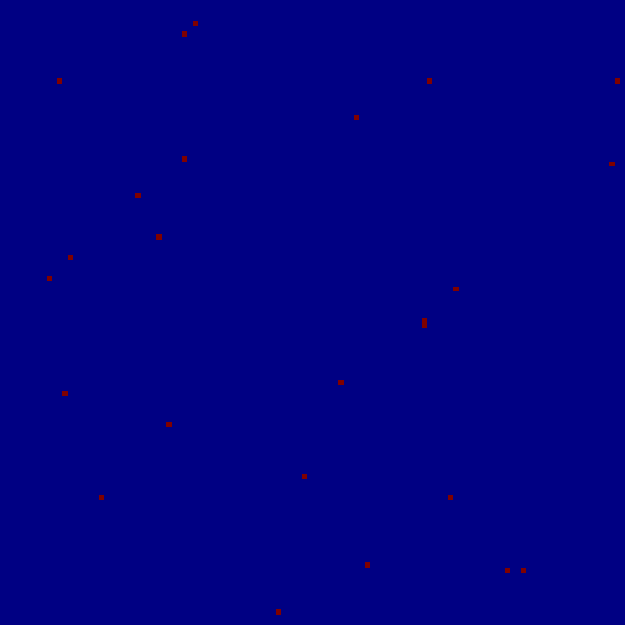}
    \\
    \multicolumn{4}{c}{\includegraphics[width=3\mediumimage]{Figures/FM/Colorbar.pdf}}
  \end{tabular}
  \caption{Hyperspectral datasets for anomaly detection; second row is ground truth.}
  \label{fig:HADdata}
\end{figure}

\begin{figure}[t]
  \centering
  \setlength{\tabcolsep}{0.2mm}
  \begin{tabular}{cc}
    \includegraphics[width=\mediumplot]{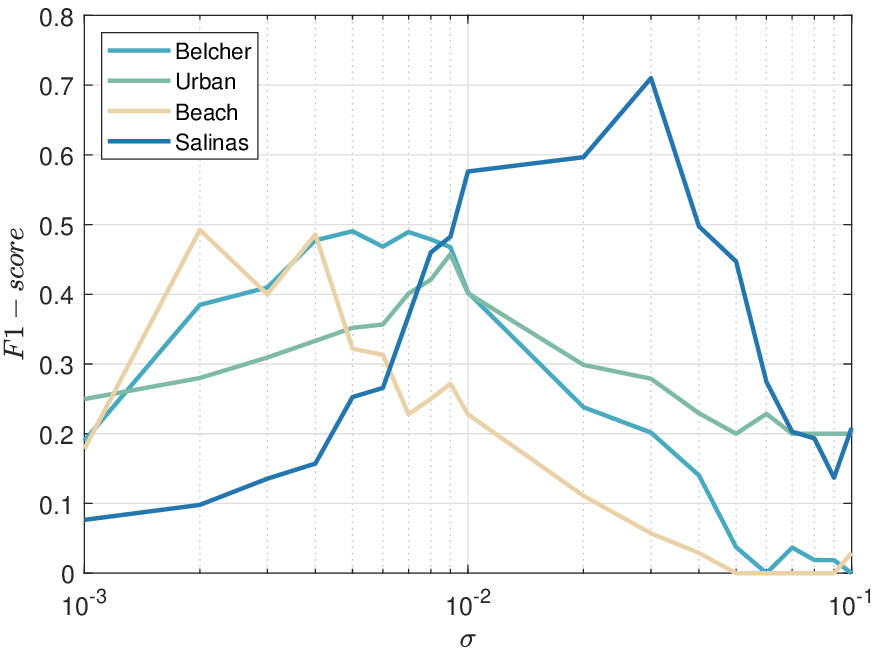} &
    \includegraphics[width=\mediumplot]{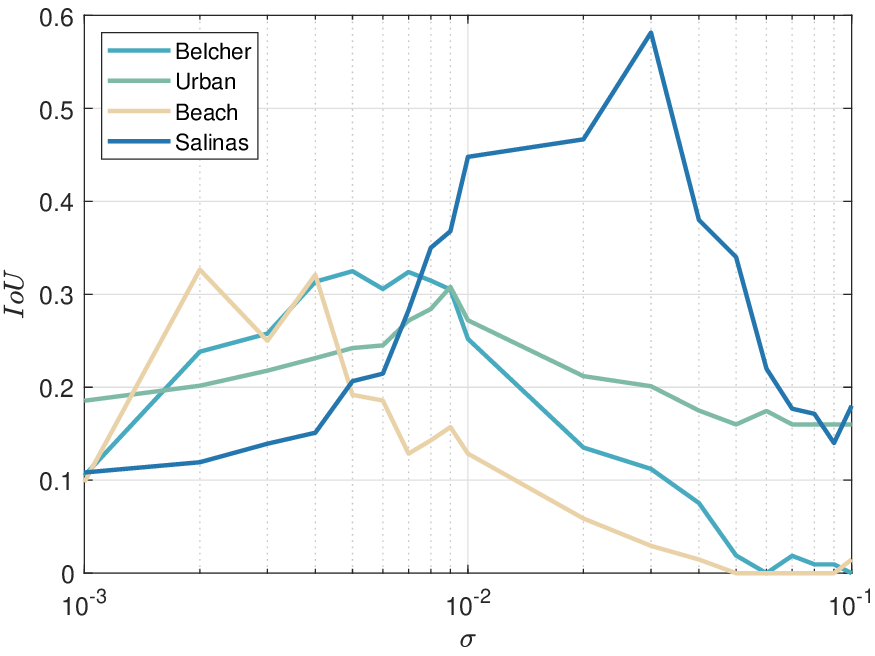}
    \\
    \includegraphics[width=\mediumplot]{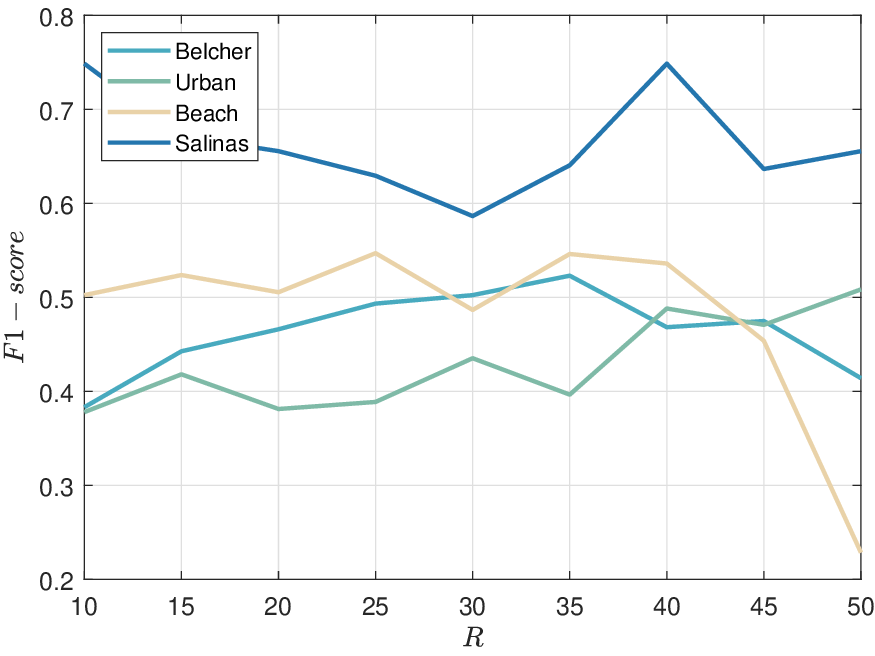} &
    \includegraphics[width=\mediumplot]{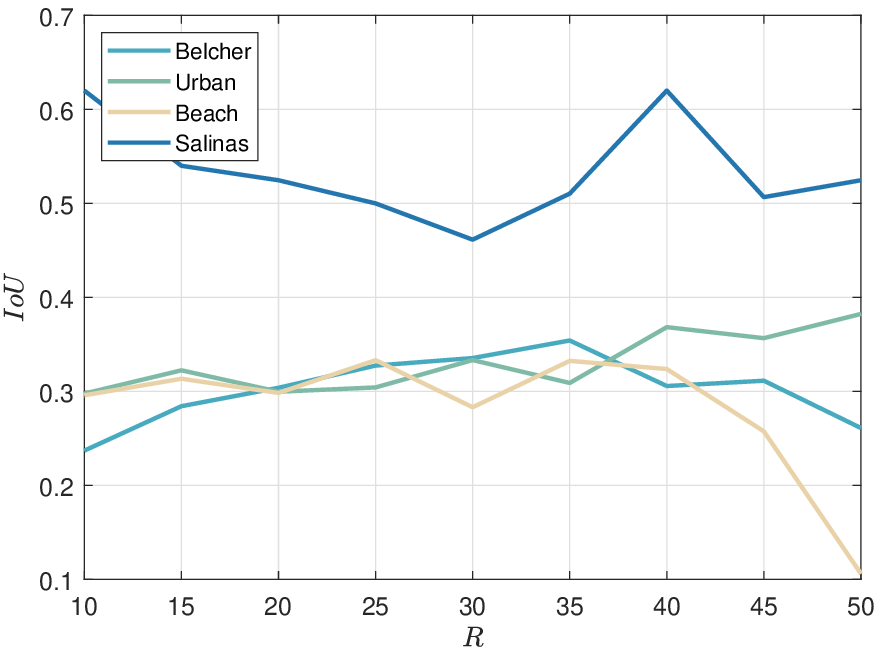}
  \end{tabular}
  \caption{Hyperparameter tuning for hyperspectral anomaly
    detection. Row 1: Tuning $\sigma$ when $R=25$. Row 2: Tuning $R$
    when $\sigma=4\times10^{-3}$ for Belcher and Beach, $\sigma=9\times10^{-3}$ for Urban and $\sigma=3\times10^{-2}$ for Salinas.}
  \label{fig:HADTuning}
\end{figure} 

\begin{table}[t]
  \centering
  \renewcommand{\arraystretch}{1.25}
  \caption{Anomaly Detection for Belcher. Best results are in boldface.}
  \begin{tabular}{c|ccc}
    \toprule
    \hline
    \multirow{2}{*}{\textbf{Methods}} &
    \multicolumn{3}{c}{\textbf{Metrics}} \\ \cline{2-4}
    & \multicolumn{1}{c|}{$\text{AUC}_{\operatorname{F1}}\uparrow$} &
    \multicolumn{1}{c|}{$\text{AUC}_{\operatorname{IoU}}\uparrow$} &
    $\text{AUC}_{\operatorname{Fa}}\downarrow$ \\ \hline
    PCP &
    \multicolumn{1}{c|}{0.1717}  &
    \multicolumn{1}{c|}{0.0992}  &
    0.0101
    \\ \hline
    KBR &
    \multicolumn{1}{c|}{0.2135}  &
    \multicolumn{1}{c|}{0.1297}  &
    0.0088
    \\ \hline
    TRPCA &
    \multicolumn{1}{c|}{0.2084}  &
    \multicolumn{1}{c|}{0.1218}  &
    0.0329
    \\\hline
    LRTFR &
    \multicolumn{1}{c|}{0.2785}  &
    \multicolumn{1}{c|}{0.1751}  &
    0.0167
    \\\hline
    PTA &
    \multicolumn{1}{c|}{0.0921}  &
    \multicolumn{1}{c|}{0.0490}  &
    0.1093
    \\\hline
    LSDM-MoG &
    \multicolumn{1}{c|}{0.2410}  &
    \multicolumn{1}{c|}{0.1466}  &
    0.0056
    \\\hline
    TLRSR &
    \multicolumn{1}{c|}{0.2802}  &
    \multicolumn{1}{c|}{0.1747}  &
    0.0144
    \\\hline
    BTD &
    \multicolumn{1}{c|}{0.2419}  &
    \multicolumn{1}{c|}{0.1482}  &
    0.0623
    \\\hline
    BCP-RPCC &
    \multicolumn{1}{c|}{\textbf{0.5308}}  &
    \multicolumn{1}{c|}{\textbf{0.3613}} &
    \textbf{0.0022}
    \\ \hline
    \bottomrule
  \end{tabular}%
  \label{tab:HADBelcher}
\end{table}

\begin{table}[t]
  \centering
  \renewcommand{\arraystretch}{1.25}
  \caption{Anomaly Detection for Urban. Best results are in boldface.}
  \begin{tabular}{c|ccc}
    \toprule
    \hline
    \multirow{2}{*}{\textbf{Methods}} &
    \multicolumn{3}{c}{\textbf{Metrics}} \\ \cline{2-4}
    & \multicolumn{1}{c|}{$\text{AUC}_{\operatorname{F1}}\uparrow$} &
    \multicolumn{1}{c|}{$\text{AUC}_{\operatorname{IoU}}\uparrow$} &
    $\text{AUC}_{\operatorname{Fa}}\downarrow$ \\ \hline
    PCP  &
    \multicolumn{1}{c|}{0.1656}  &
    \multicolumn{1}{c|}{0.0931}  &
    0.1365
    \\ \hline
    KBR  &
    \multicolumn{1}{c|}{0.2672}  &
    \multicolumn{1}{c|}{0.1715} & 0.1052
    \\ \hline
    TRPCA  &
    \multicolumn{1}{c|}{0.2987}  &
    \multicolumn{1}{c|}{0.1956} & 0.1504
    \\\hline
    LRTFR  &
    \multicolumn{1}{c|}{0.2354}  &
    \multicolumn{1}{c|}{0.1453} & 0.1058
    \\\hline
    PTA  &
    \multicolumn{1}{c|}{0.1403}  &
    \multicolumn{1}{c|}{0.0792} & 0.3321
    \\\hline
    LSDM-MoG  &
    \multicolumn{1}{c|}{0.3353}  &
    \multicolumn{1}{c|}{0.2195}  &
    0.0344
    \\\hline
    TLRSR  &
    \multicolumn{1}{c|}{0.2765}  &
    \multicolumn{1}{c|}{0.1748}  &
    0.0935
    \\\hline
    BTD  &
    \multicolumn{1}{c|}{0.3186}  &
    \multicolumn{1}{c|}{0.2104}  &
    0.1278
    \\\hline
    BCP-RPCC  &
    \multicolumn{1}{c|}{\textbf{0.4686}}  &
    \multicolumn{1}{c|}{\textbf{0.3023}} & \textbf{0.0019}
    \\ \hline\bottomrule
  \end{tabular}%
  \label{tab:HADUrban}
\end{table}

\begin{table}[t]
  \centering
  \renewcommand{\arraystretch}{1.25}
  \caption{Anomaly Detection for Beach. Best results are in boldface.}
  \begin{tabular}{c|ccc}
    \toprule
    \hline
    \multirow{2}{*}{\textbf{Methods}} &
    \multicolumn{3}{c}{\textbf{Metrics}} \\ \cline{2-4}
    & \multicolumn{1}{c|}{$\text{AUC}_{\operatorname{F1}}\uparrow$} &
    \multicolumn{1}{c|}{$\text{AUC}_{\operatorname{IoU}}\uparrow$} &
    $\text{AUC}_{\operatorname{Fa}}\downarrow$ \\ \hline
    PCP      &
    \multicolumn{1}{c|}{0.0960} &
    \multicolumn{1}{c|}{0.0520} &
    0.0030
    \\ \hline
    KBR      &
    \multicolumn{1}{c|}{0.1391} &
    \multicolumn{1}{c|}{0.0770} &
    0.0284
    \\ \hline
    TRPCA      &
    \multicolumn{1}{c|}{0.1285} &
    \multicolumn{1}{c|}{0.0714} &
    0.0411
    \\ \hline
    LRTFR      &
    \multicolumn{1}{c|}{0.2469} &
    \multicolumn{1}{c|}{0.1552} &
    0.0275
    \\ \hline
    PTA      &
    \multicolumn{1}{c|}{0.0703} &
    \multicolumn{1}{c|}{0.0378} &
    0.1100
    \\ \hline
    LSDM-MoG      &
    \multicolumn{1}{c|}{0.1549} &
    \multicolumn{1}{c|}{0.0864} &
    0.0602
    \\ \hline
    TLRSR      &
    \multicolumn{1}{c|}{0.2264} &
    \multicolumn{1}{c|}{0.1418} &
    0.0170
    \\ \hline
    BTD      &
    \multicolumn{1}{c|}{0.2022} &
    \multicolumn{1}{c|}{0.1267} &
    0.0429
    \\ \hline
    BCP-RPCC      &
    \multicolumn{1}{c|}{\textbf{0.4717}} &
    \multicolumn{1}{c|}{\textbf{0.3086}} &
    \textbf{0.0006}
    \\ \hline
    \bottomrule
  \end{tabular}%
  \label{tab:HADBeach}
\end{table}

\begin{table}[t]
  \centering
  \renewcommand{\arraystretch}{1.25}
  \caption{Anomaly Detection for Salinas. Best results are in boldface.}
  \begin{tabular}{c|ccc}
    \toprule
    \hline
    \multirow{2}{*}{\textbf{Methods}} &
    \multicolumn{3}{c}{\textbf{Metrics}} \\ \cline{2-4}
    & \multicolumn{1}{c|}{$\text{AUC}_{\operatorname{F1}}\uparrow$} &
    \multicolumn{1}{c|}{$\text{AUC}_{\operatorname{IoU}}\uparrow$} &
    $\text{AUC}_{\operatorname{Fa}}\downarrow$ \\ \hline
    PCP      &
    \multicolumn{1}{c|}{0.4369} &
    \multicolumn{1}{c|}{0.3056} &
    0.0090
    \\ \hline
    KBR      &
    \multicolumn{1}{c|}{0.5613} &
    \multicolumn{1}{c|}{0.4311} &
    0.0221
    \\ \hline
    TRPCA      &
    \multicolumn{1}{c|}{0.4733} &
    \multicolumn{1}{c|}{0.3537} &
    0.0583
    \\ \hline
    LRTFR      &
    \multicolumn{1}{c|}{0.5533} &
    \multicolumn{1}{c|}{0.4317} &
    0.0222
    \\ \hline
    PTA      &
    \multicolumn{1}{c|}{0.2946} &
    \multicolumn{1}{c|}{0.2094} &
    0.2352
    \\ \hline
    LSDM-MoG      &
    \multicolumn{1}{c|}{0.5046} &
    \multicolumn{1}{c|}{0.3735} &
    0.0162
    \\ \hline
    TLRSR      &
    \multicolumn{1}{c|}{0.5403} &
    \multicolumn{1}{c|}{0.4145} &
    0.0094
    \\ \hline
    BTD      &
    \multicolumn{1}{c|}{0.4953} &
    \multicolumn{1}{c|}{0.3741} &
    0.0330
    \\ \hline
    BCP-RPCC      &
    \multicolumn{1}{c|}{\textbf{0.7295}} &
    \multicolumn{1}{c|}{\textbf{0.5714}} & \textbf{0.0002}
    \\ \hline
    \bottomrule
  \end{tabular}%
  \label{tab:HADSalinas}
\end{table}

\begin{figure*}[t]
  \centering
  \setlength{\tabcolsep}{0.1mm}
  \begin{tabular}{cccccccccc}
    \includegraphics[width=\smallimage]{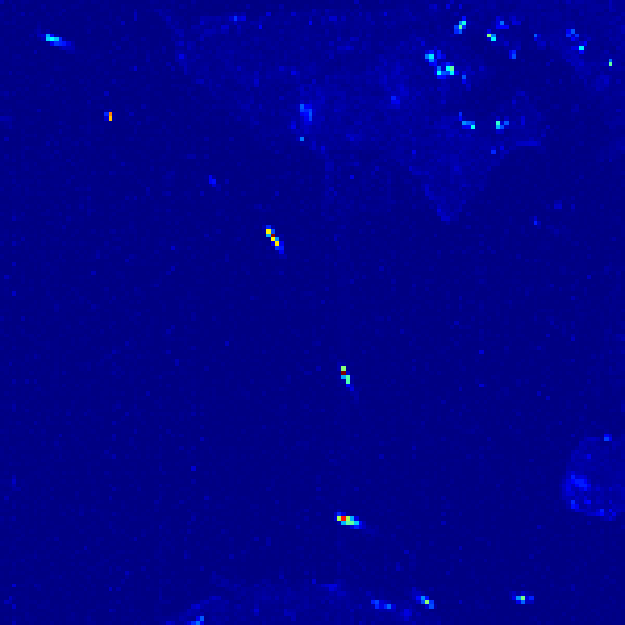} &
    \includegraphics[width=\smallimage]{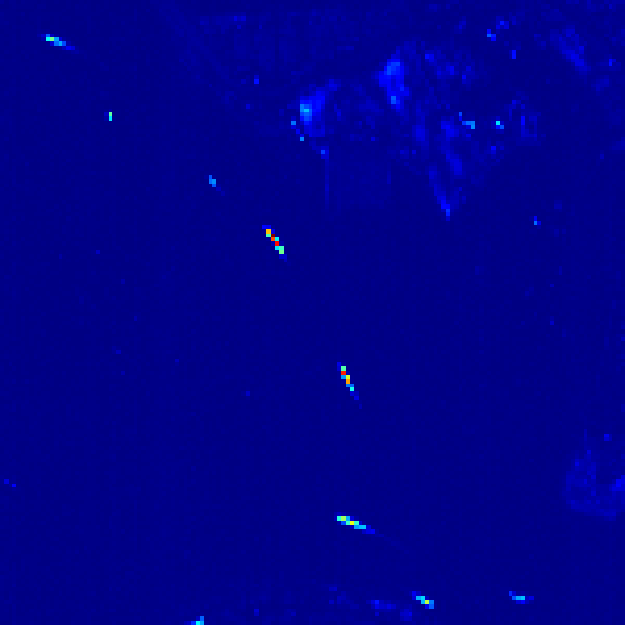} &
    \includegraphics[width=\smallimage]{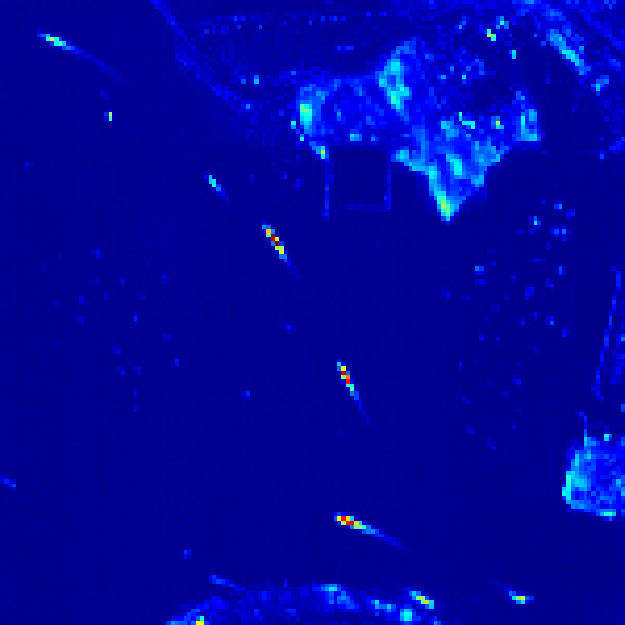} &
    \includegraphics[width=\smallimage]{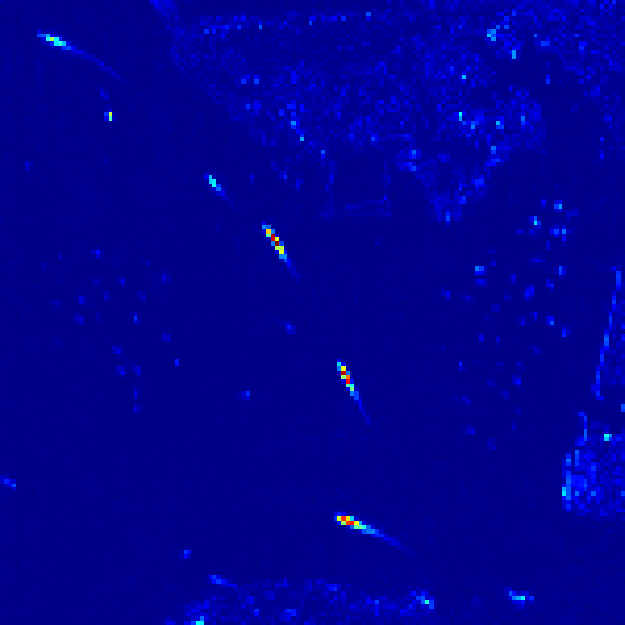} &
    \includegraphics[width=\smallimage]{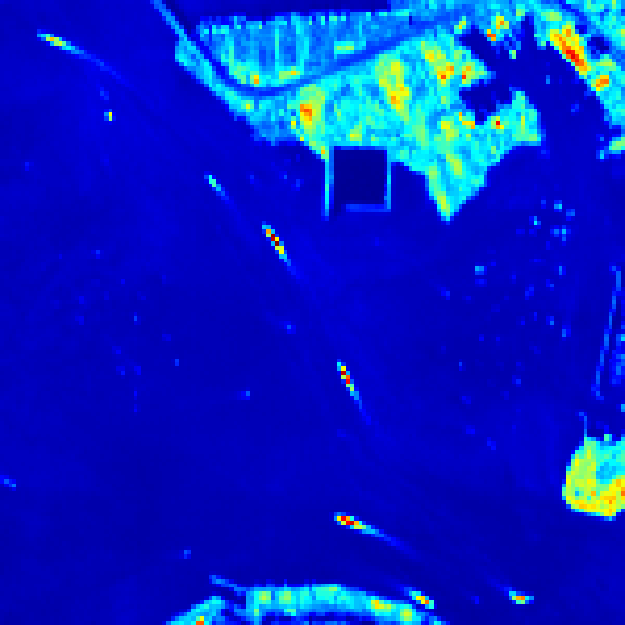} &
    \includegraphics[width=\smallimage]{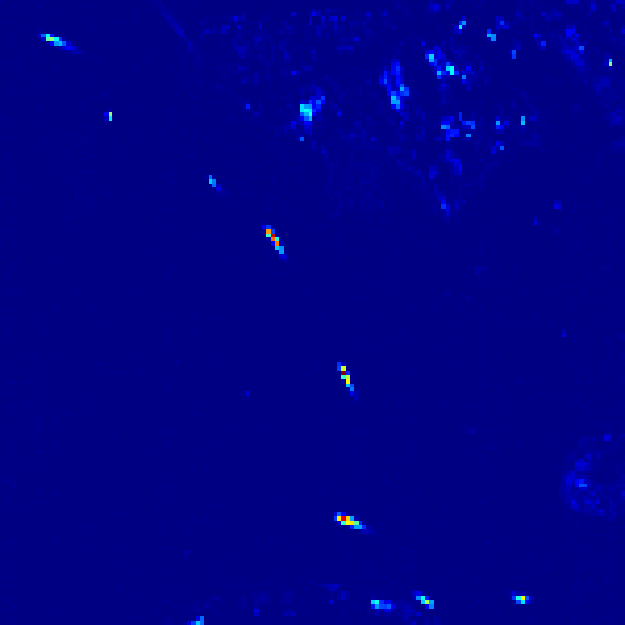} &
    \includegraphics[width=\smallimage]{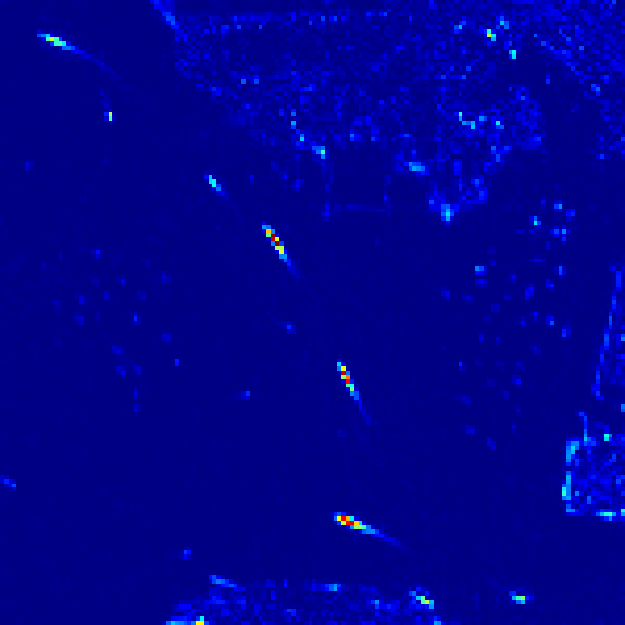} &
    \includegraphics[width=\smallimage]{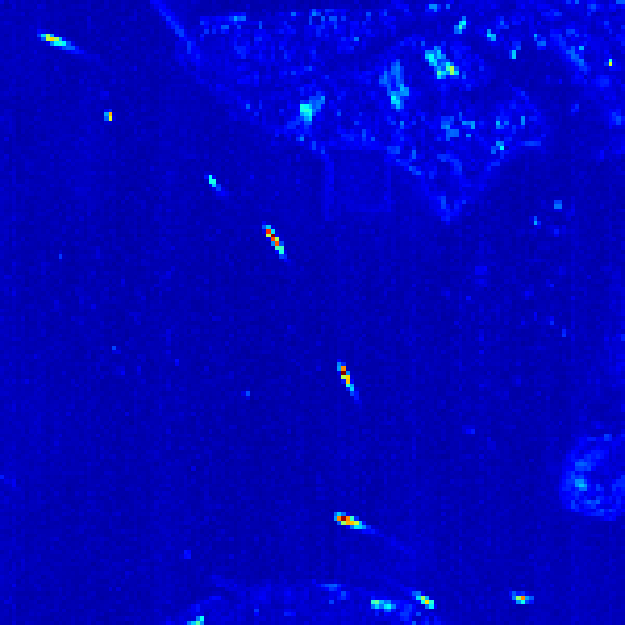} &
    \includegraphics[width=\smallimage]{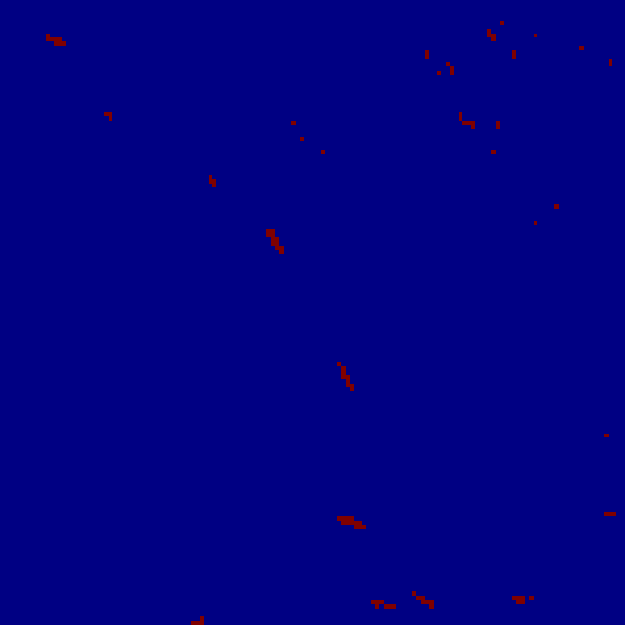} &
    \includegraphics[width=\smallimage]{Figures/HAD/Belcher/GT.eps}
    \\
    \includegraphics[width=\smallimage]{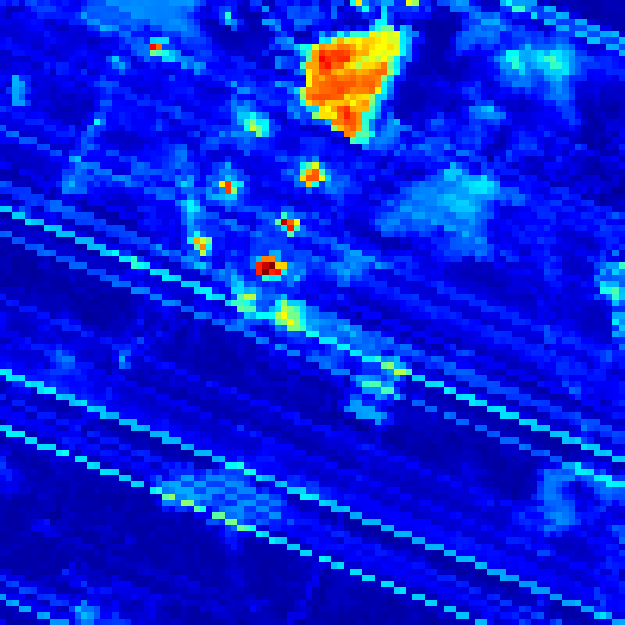} &
    \includegraphics[width=\smallimage]{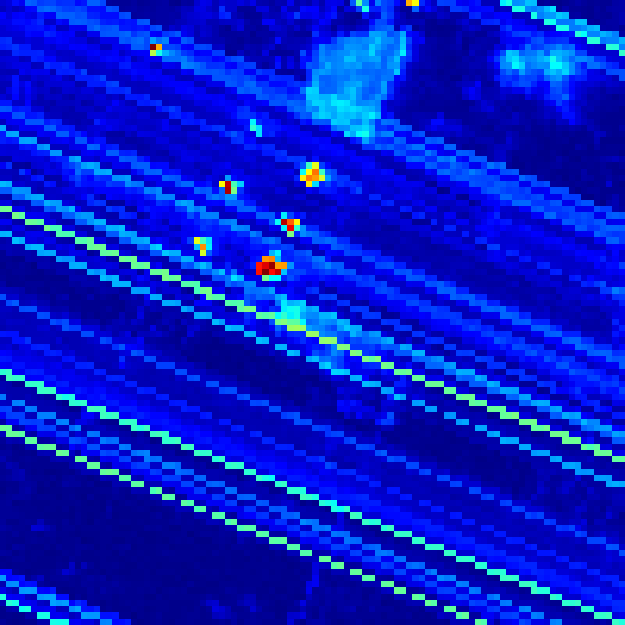} &
    \includegraphics[width=\smallimage]{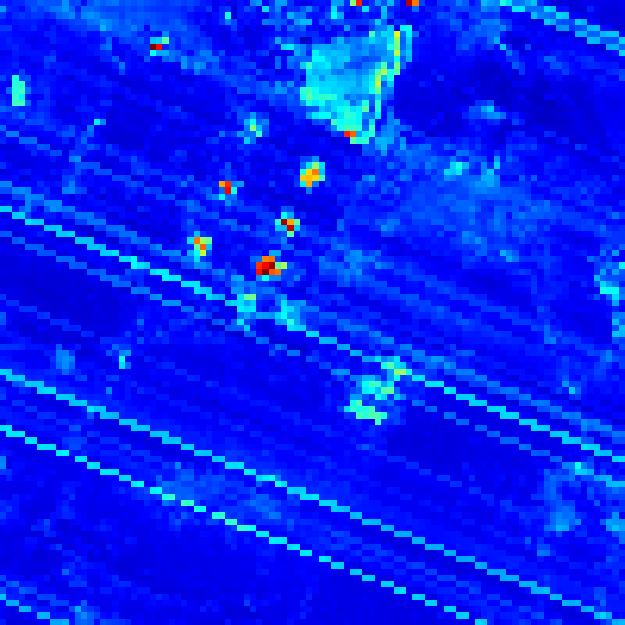} &
    \includegraphics[width=\smallimage]{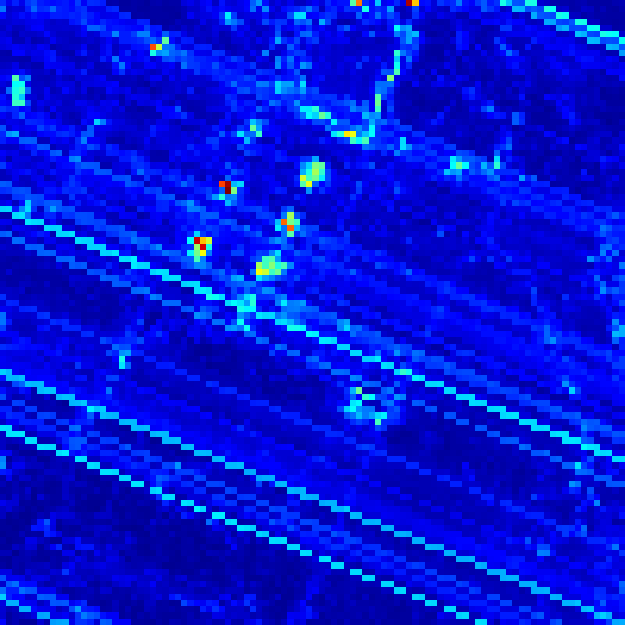} &
    \includegraphics[width=\smallimage]{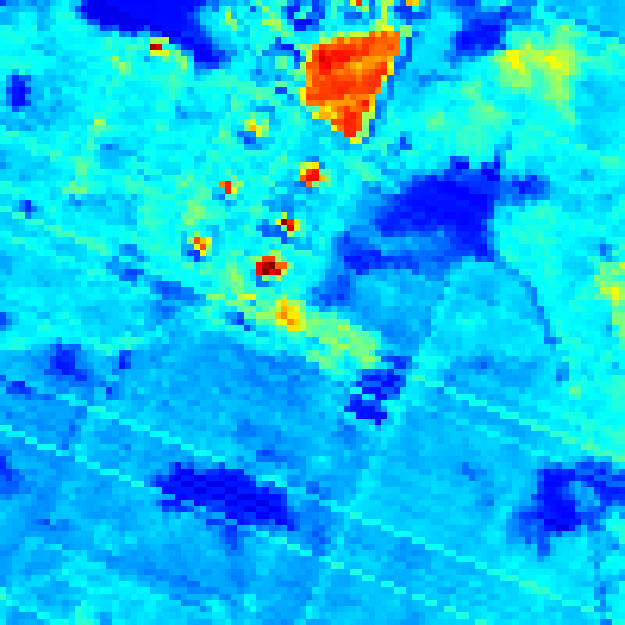} &
    \includegraphics[width=\smallimage]{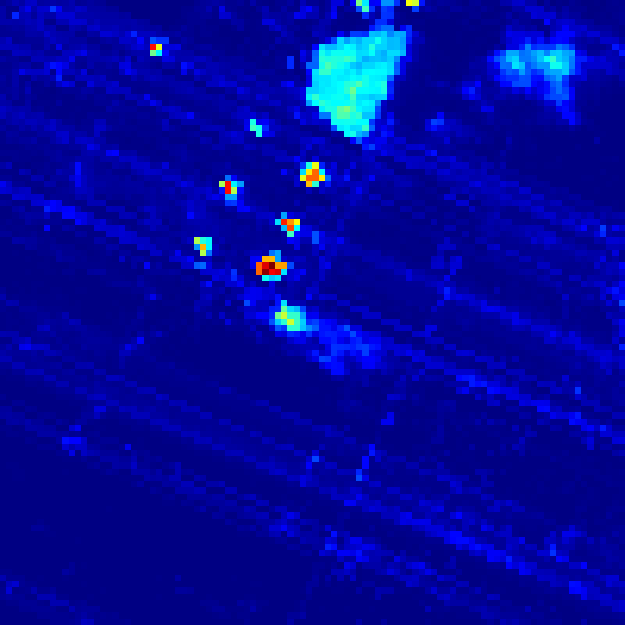} &
    \includegraphics[width=\smallimage]{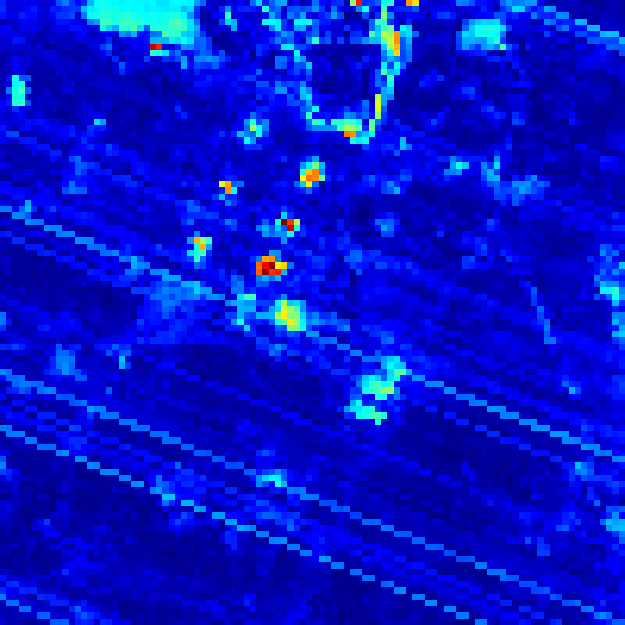} &
    \includegraphics[width=\smallimage]{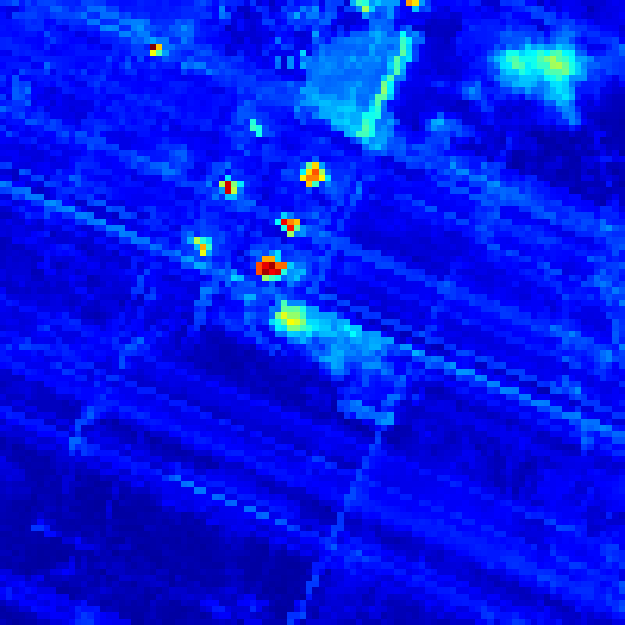} &
    \includegraphics[width=\smallimage]{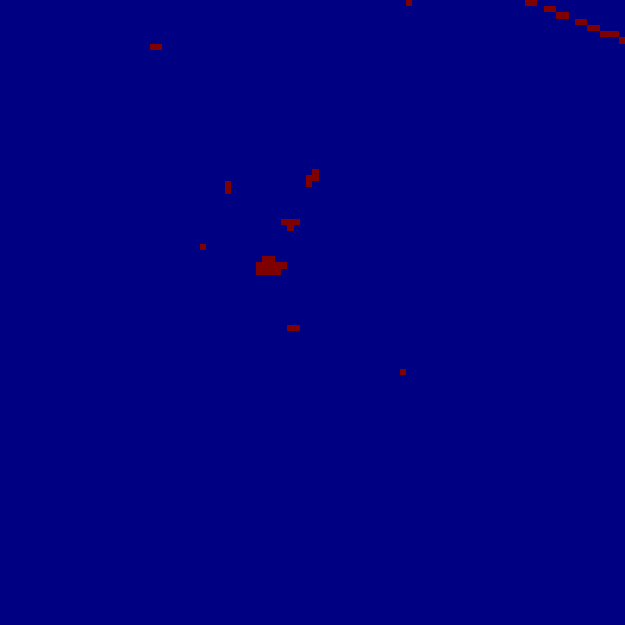} &
    \includegraphics[width=\smallimage]{Figures/HAD/Urban/GT.eps}
    \\
    \includegraphics[width=\smallimage]{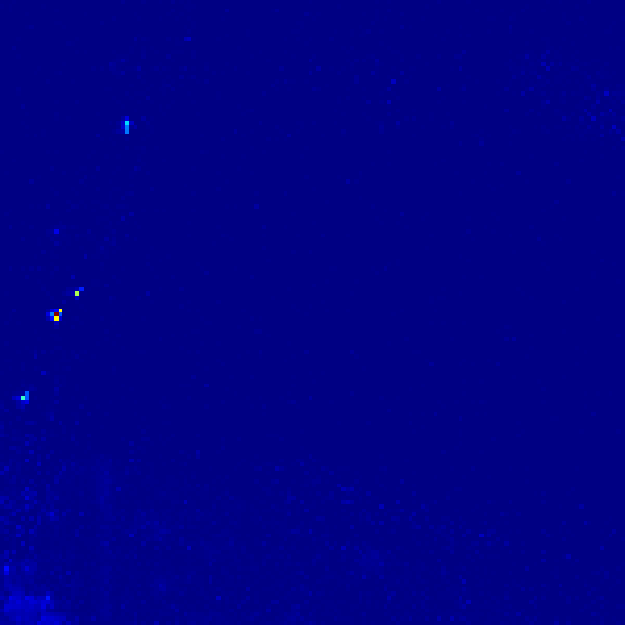} &
    \includegraphics[width=\smallimage]{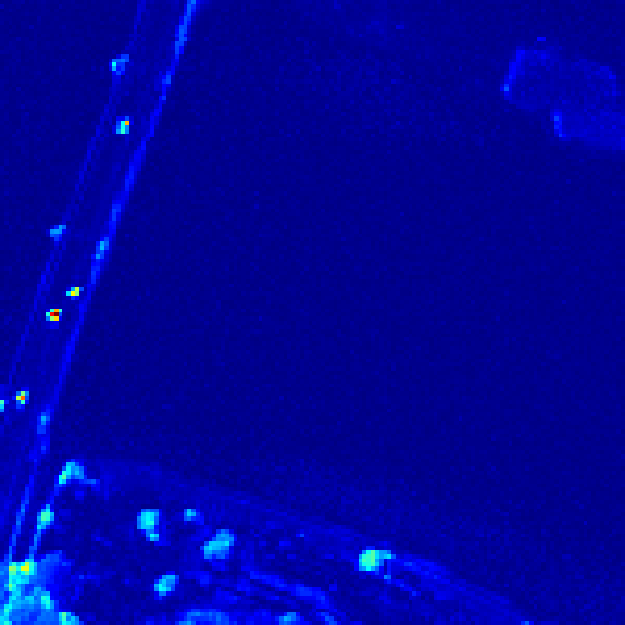} &
    \includegraphics[width=\smallimage]{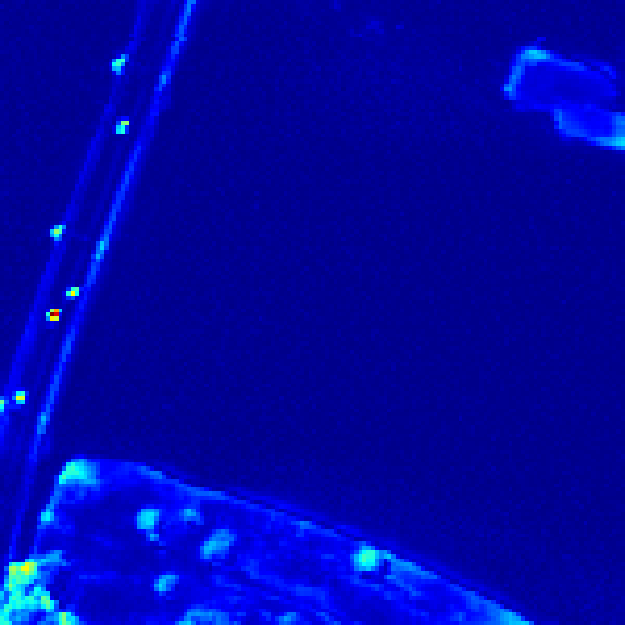} &
    \includegraphics[width=\smallimage]{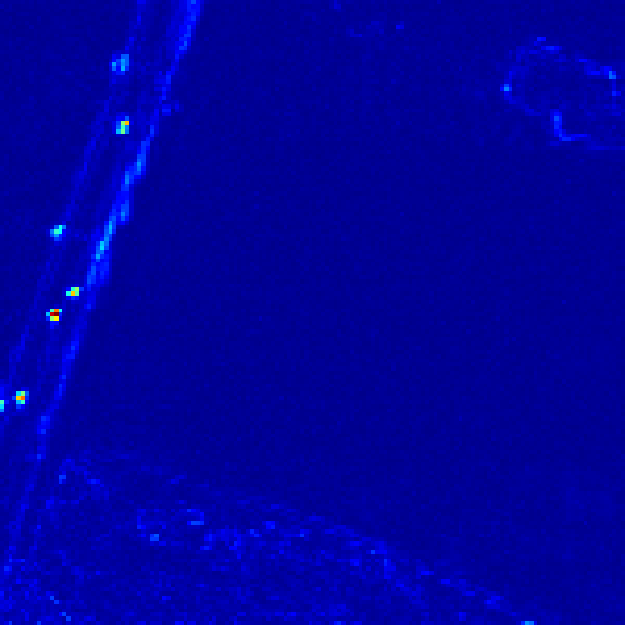} &
    \includegraphics[width=\smallimage]{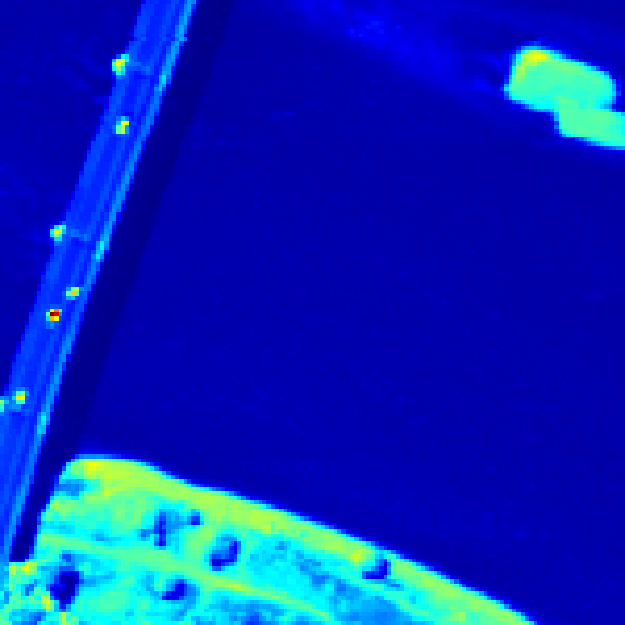} &
    \includegraphics[width=\smallimage]{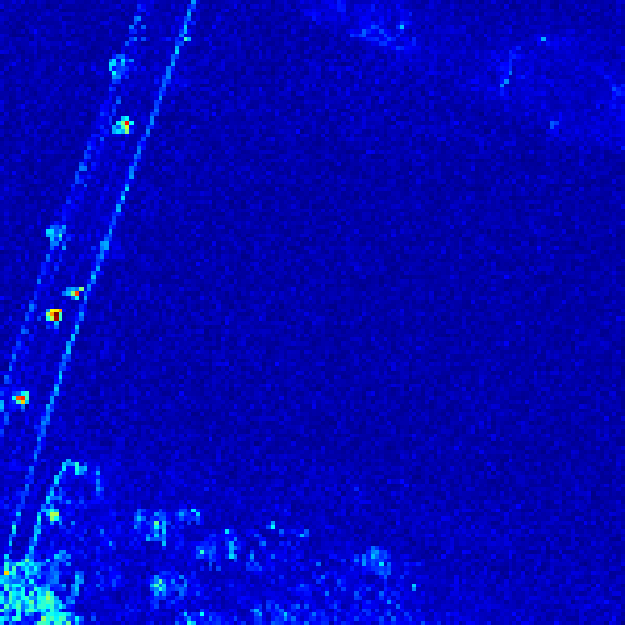} &
    \includegraphics[width=\smallimage]{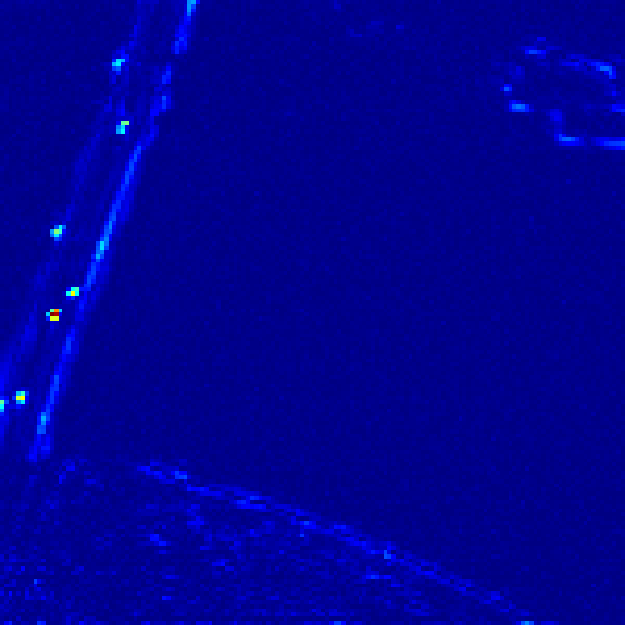} &
    \includegraphics[width=\smallimage]{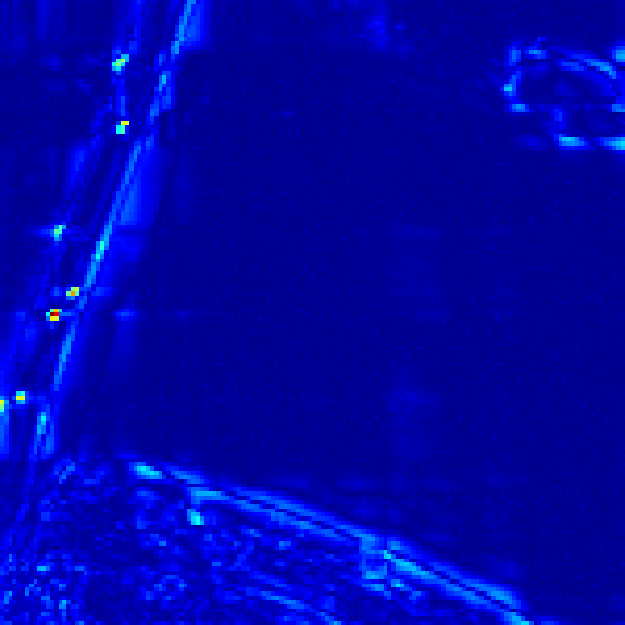} &
    \includegraphics[width=\smallimage]{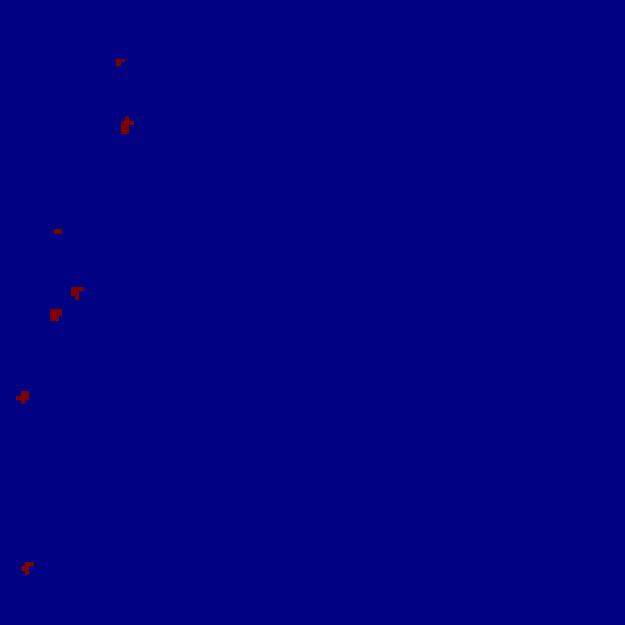} &
    \includegraphics[width=\smallimage]{Figures/HAD/Beach/GT.eps}
    \\
    \includegraphics[width=\smallimage]{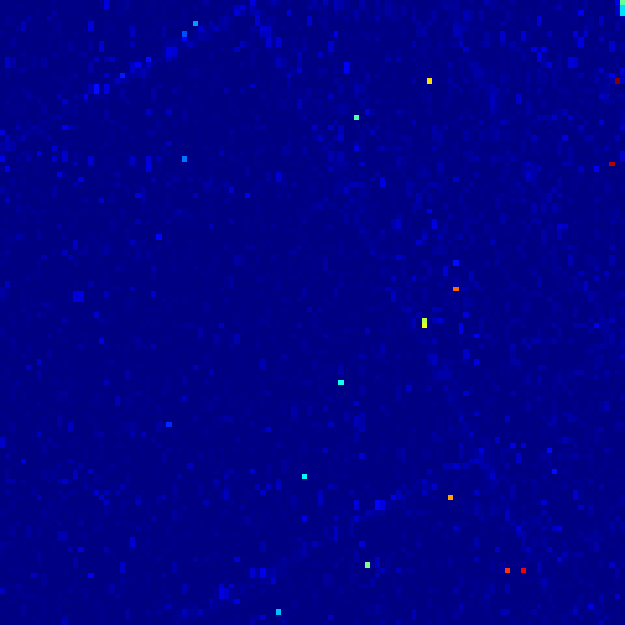} &
    \includegraphics[width=\smallimage]{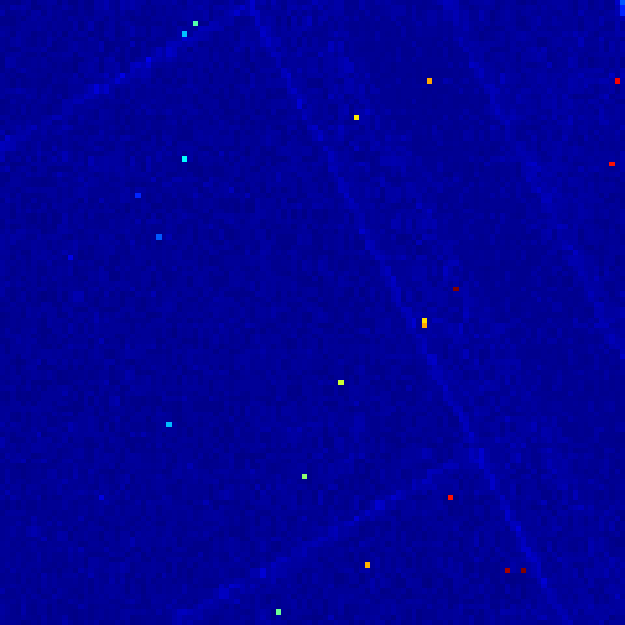} &
    \includegraphics[width=\smallimage]{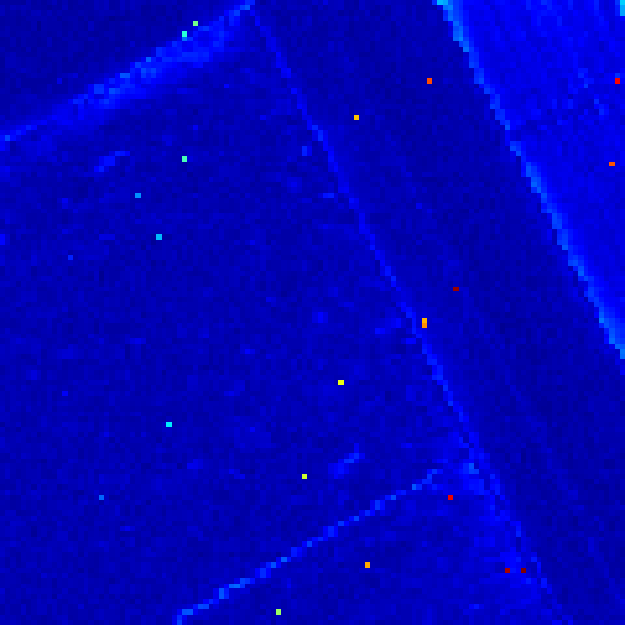} &
    \includegraphics[width=\smallimage]{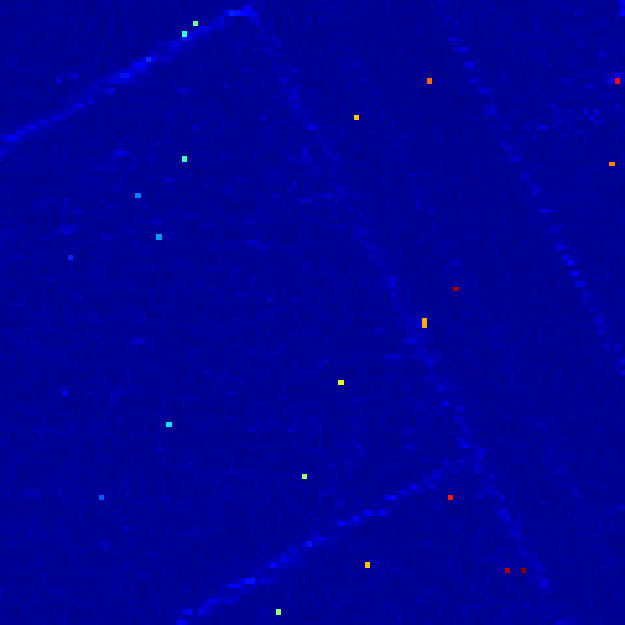} &
    \includegraphics[width=\smallimage]{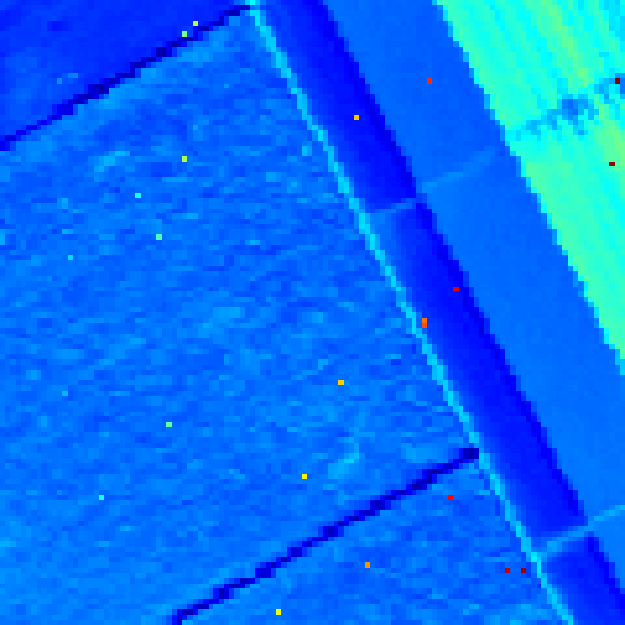} &
    \includegraphics[width=\smallimage]{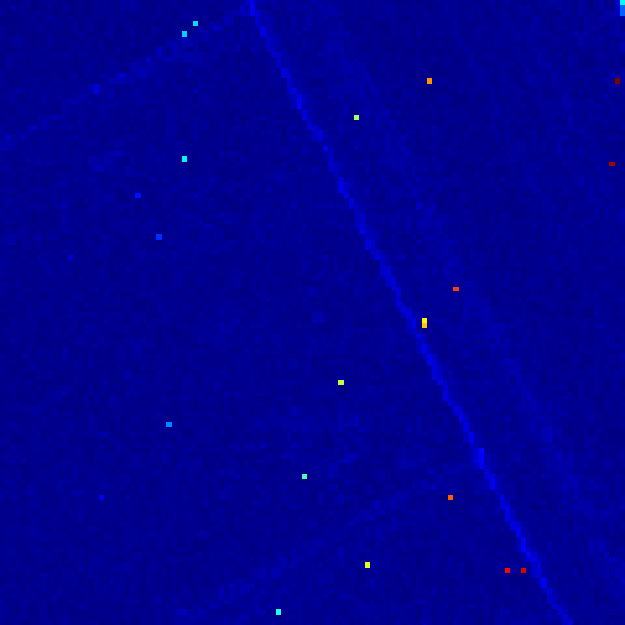} &
    \includegraphics[width=\smallimage]{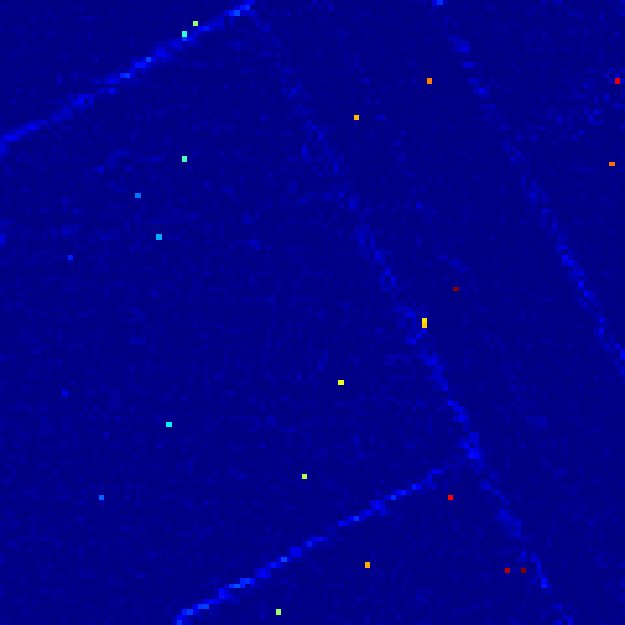} &
    \includegraphics[width=\smallimage]{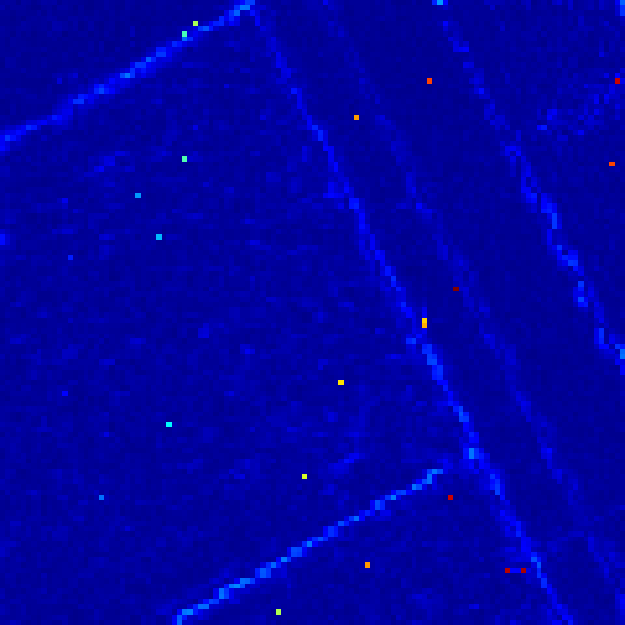} &
    \includegraphics[width=\smallimage]{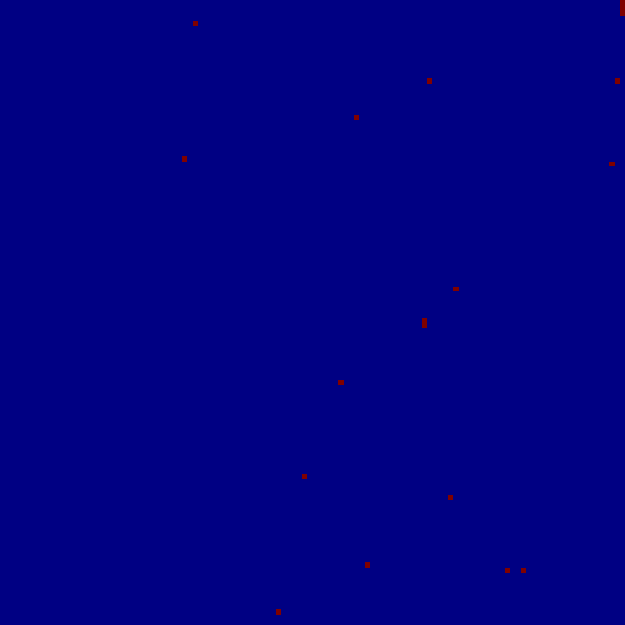} &
    \includegraphics[width=\smallimage]{Figures/HAD/Salinas/GT.eps}
    \\
    \multicolumn{1}{c}{\footnotesize PCP} &
    \multicolumn{1}{c}{\footnotesize KBR} &
    \multicolumn{1}{c}{\footnotesize TRPCA} &
    \multicolumn{1}{c}{\footnotesize LRTFR} &
    \multicolumn{1}{c}{\footnotesize PTA} &
    \multicolumn{1}{c}{\footnotesize LSDM-MoG} &
    \multicolumn{1}{c}{\footnotesize TLRSR} &
    \multicolumn{1}{c}{\footnotesize BTD} &
    \multicolumn{1}{c}{\footnotesize BCP-RPCC} &
    \multicolumn{1}{c}{\footnotesize Ground truth}
  \end{tabular}
  \caption{Anomaly detection on Belcher (row 1),
    Urban (row 2), Beach (row 3), and Salinas (row 4).}
  \label{fig:HADvisual}
\end{figure*}

\begin{figure}[t]
  \centering
  \setlength{\tabcolsep}{0.15mm}
  \begin{tabular}{ccc}
    \includegraphics[width=\smallplot]{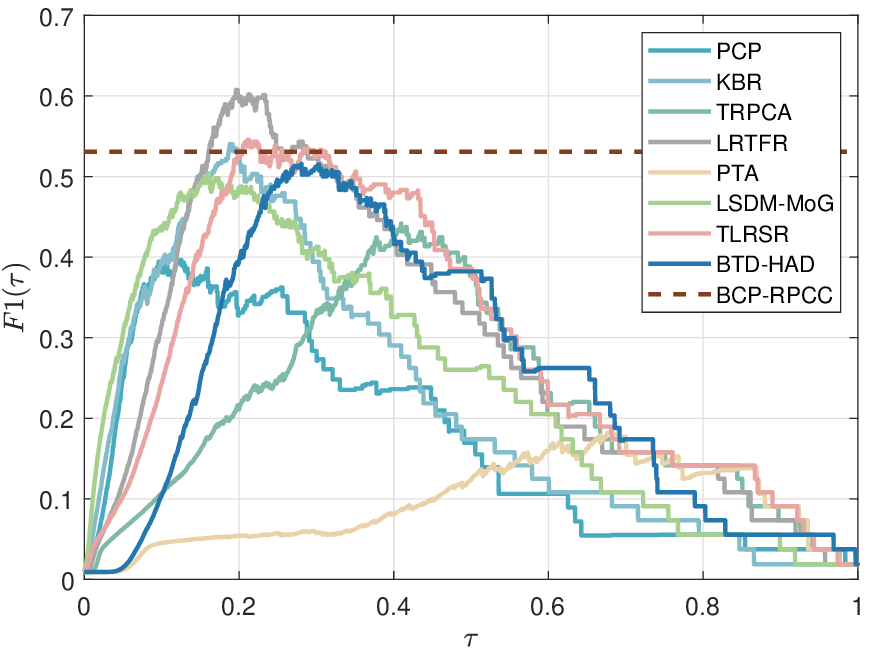} &
    \includegraphics[width=\smallplot]{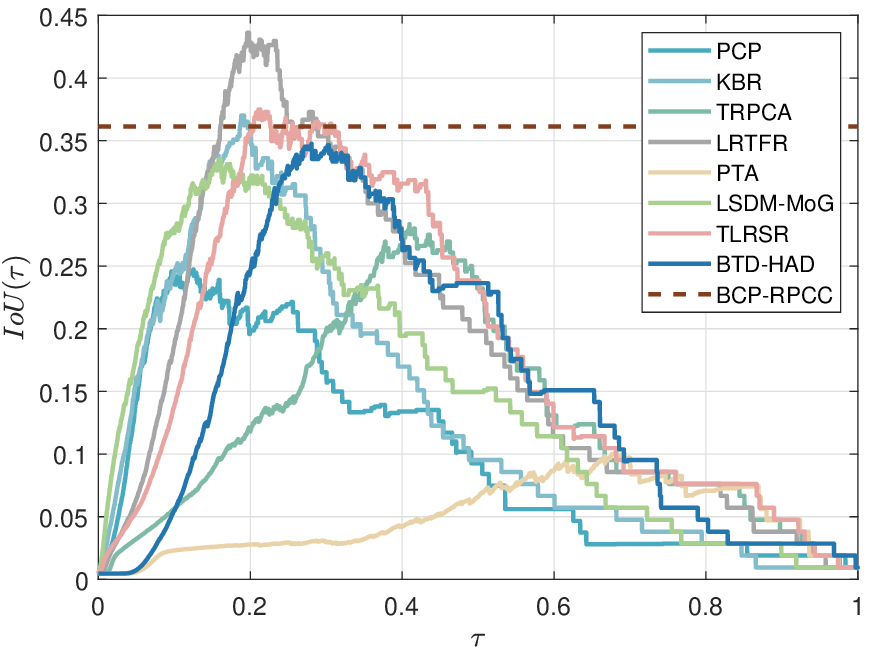} &
    \includegraphics[width=\smallplot]{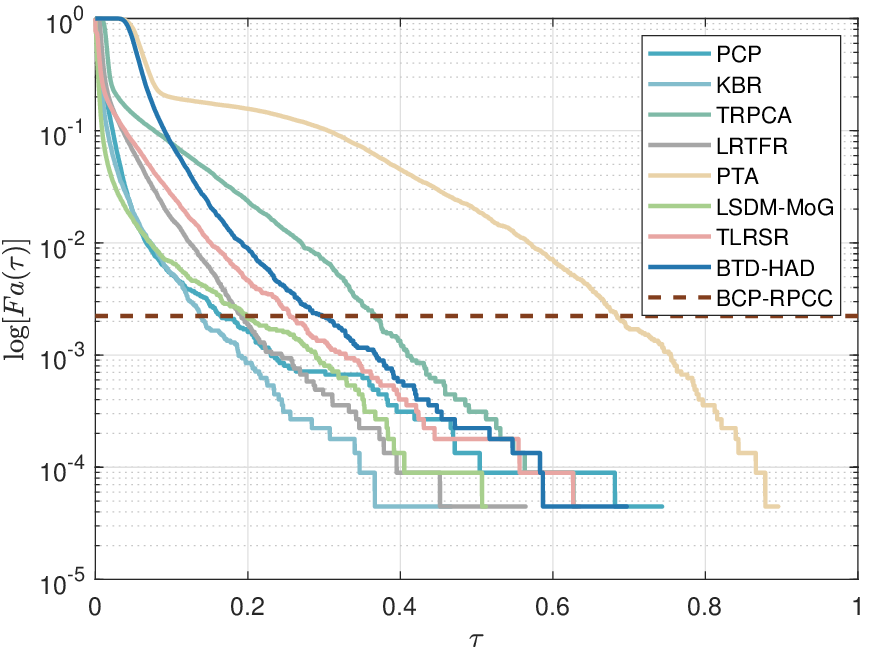}
    \\
    \includegraphics[width=\smallplot]{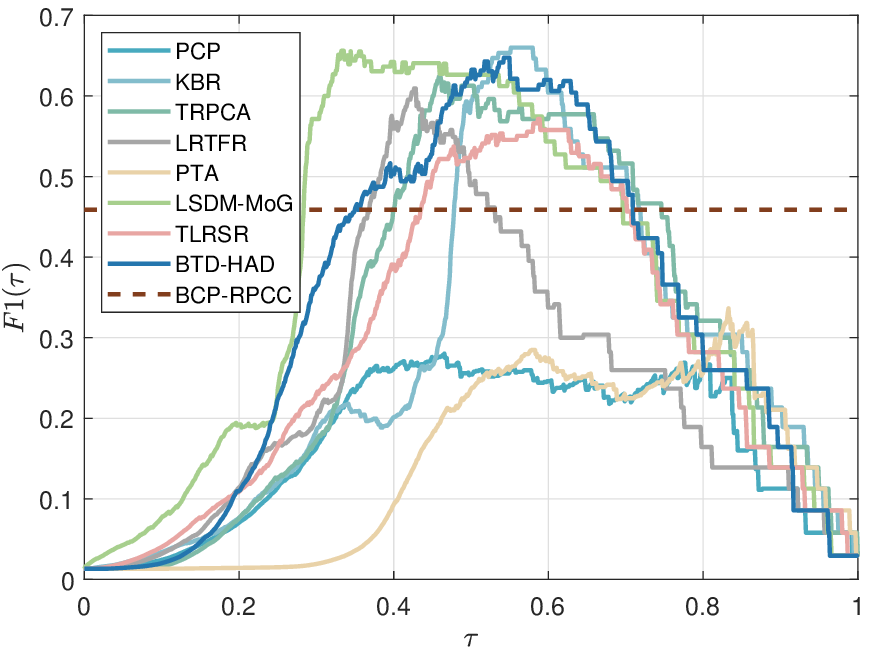}  &
    \includegraphics[width=\smallplot]{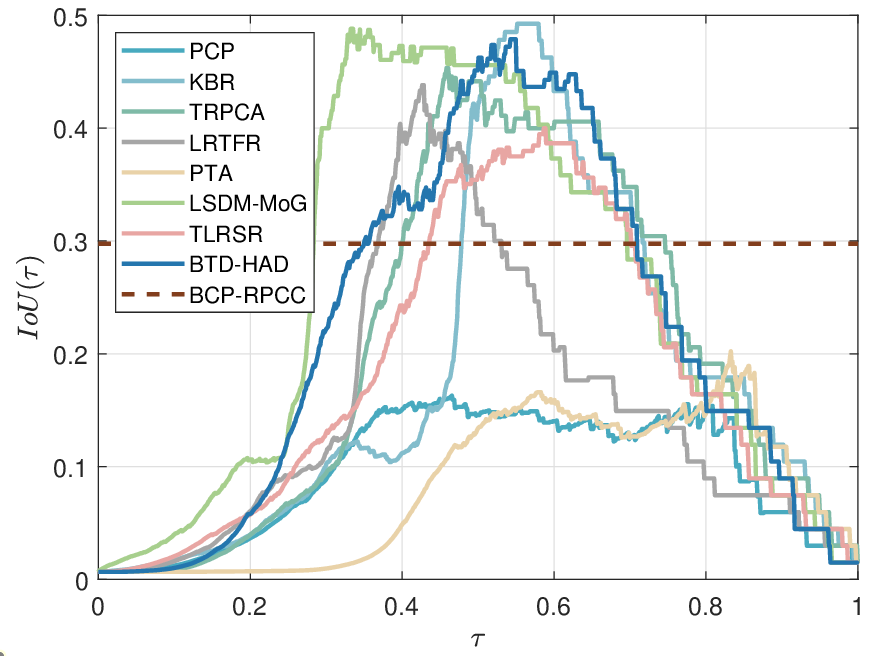} &
    \includegraphics[width=\smallplot]{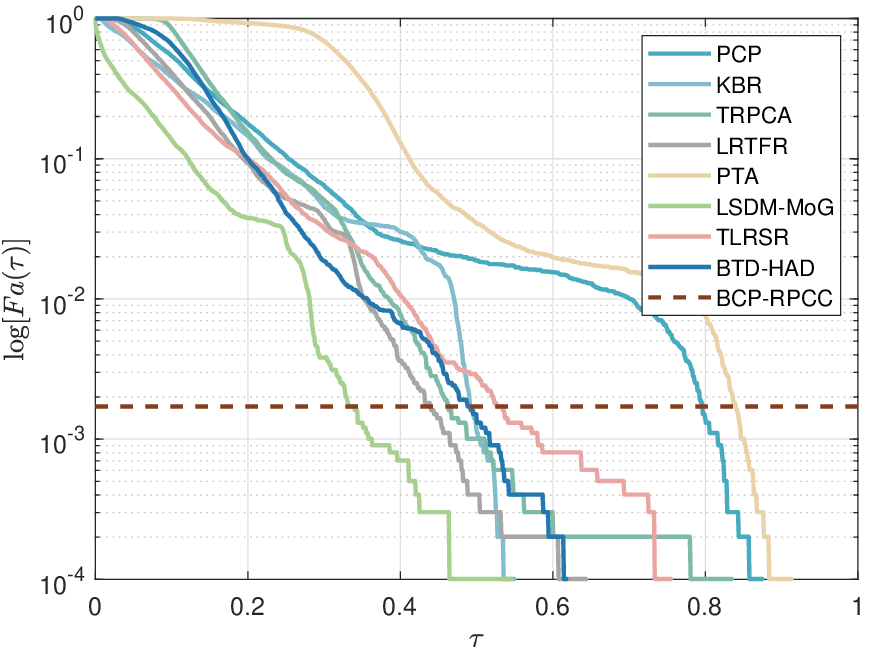}
    \\
    \includegraphics[width=\smallplot]{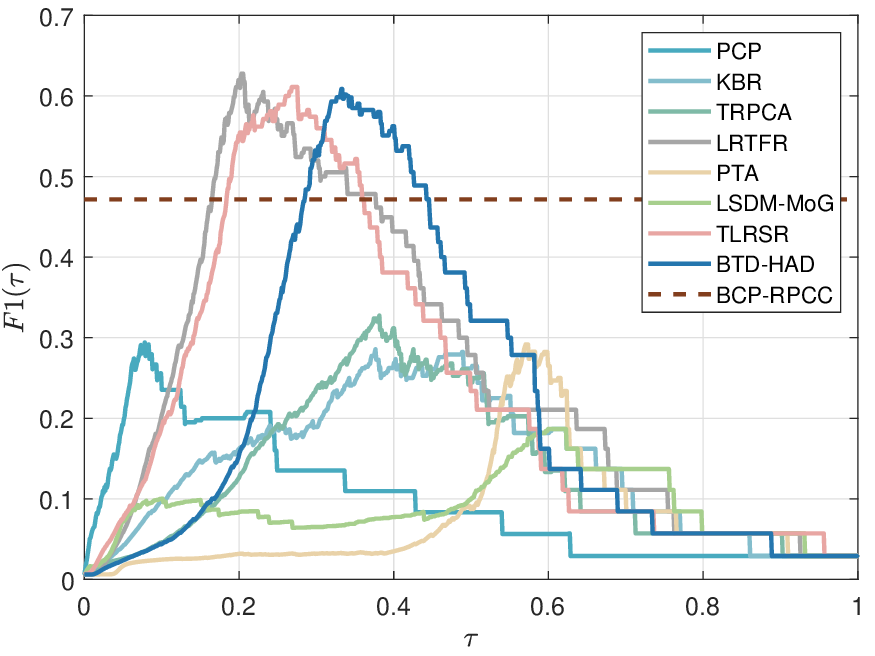} &
    \includegraphics[width=\smallplot]{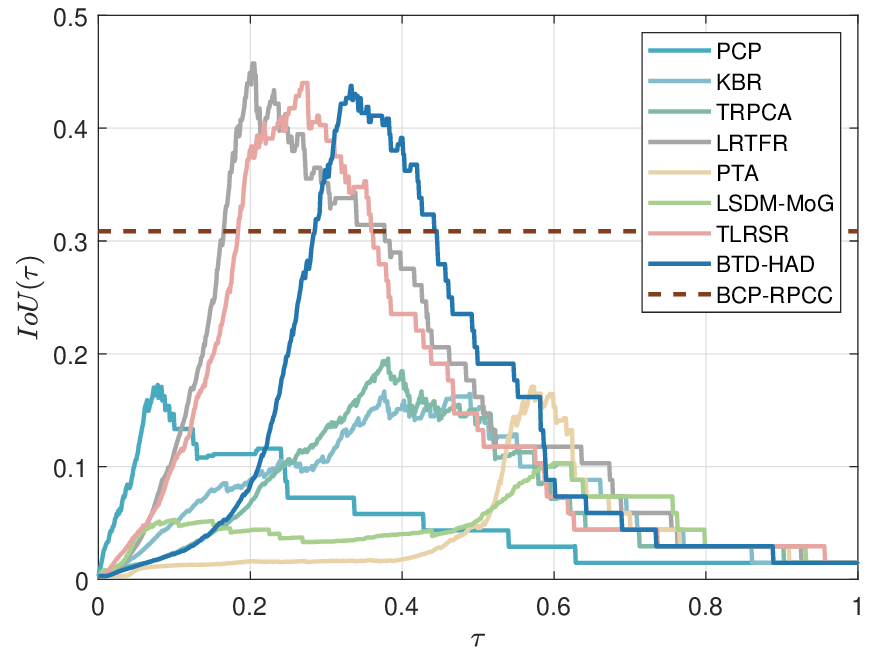} &
    \includegraphics[width=\smallplot]{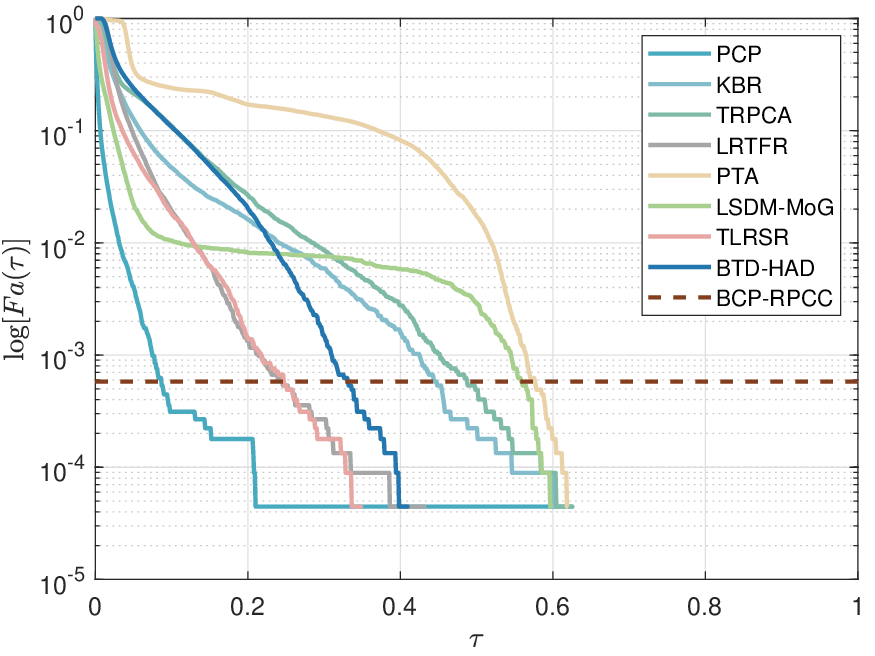}
    \\
    \includegraphics[width=\smallplot]{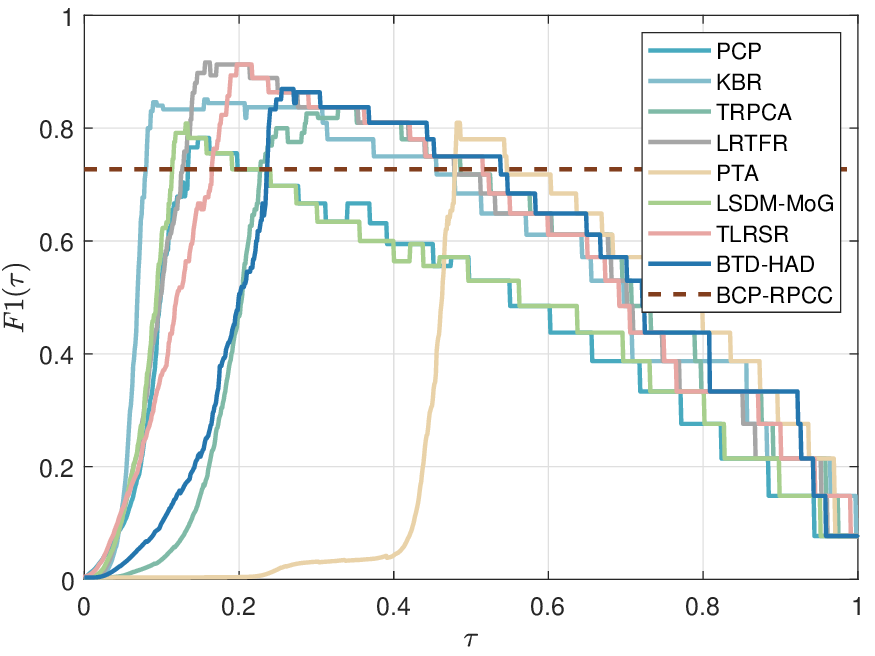} &
    \includegraphics[width=\smallplot]{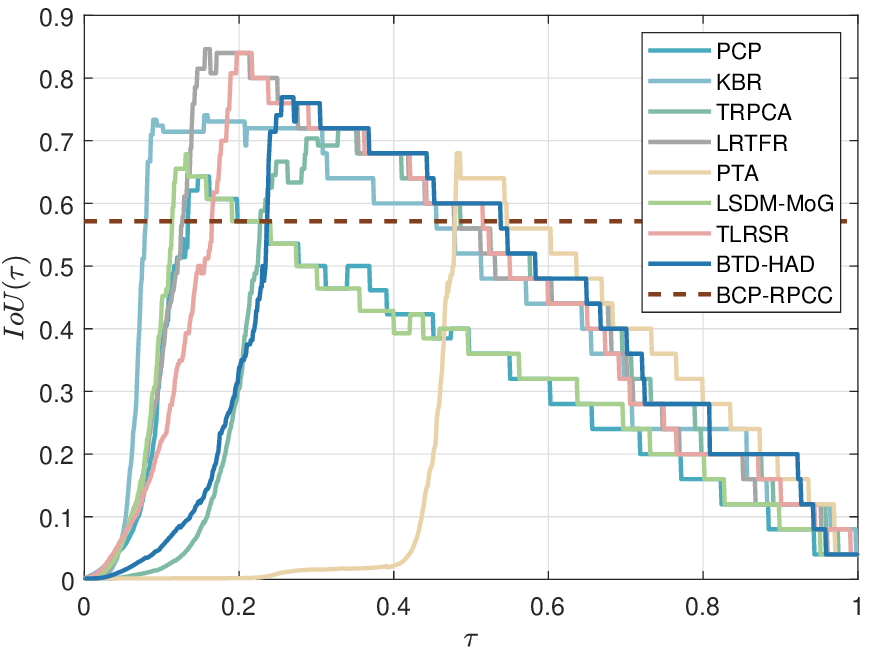} &
    \includegraphics[width=\smallplot]{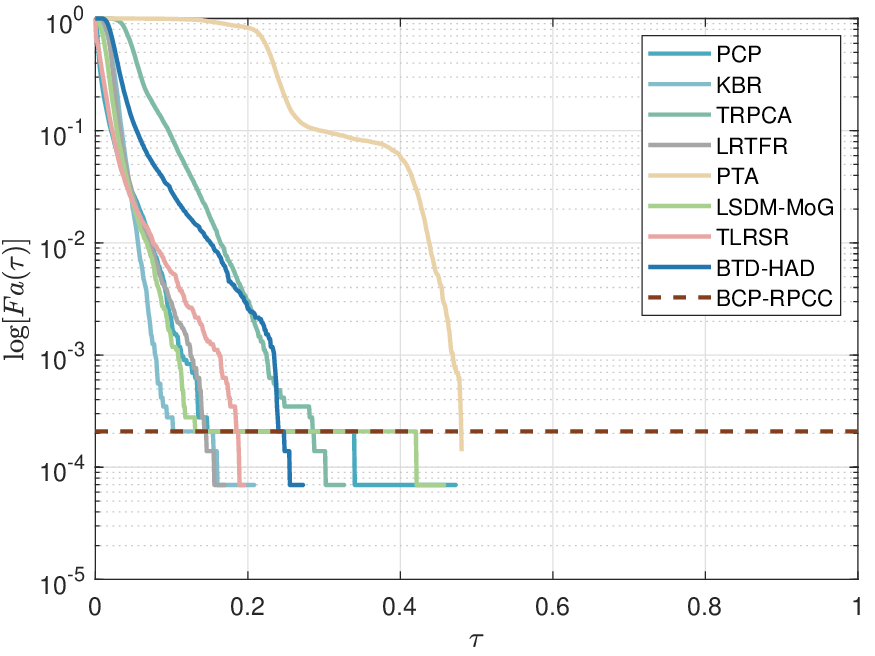}
  \end{tabular}
  \caption{Anomaly-detection performance in terms of
    $\operatorname{F1}(\tau)$, $\operatorname{IoU}(\tau)$ and
    $\operatorname{Fa}(\tau)$: Belcher (row~1), Urban (row~2), Beach
    (row~3) and Salinas (row~4).}
  \label{fig:HADcurves}
\end{figure}

We now consider the problem of identifying anomalies within a
hypspectral image. In this scenario, BCP-RPCC is used to separate the
anomalies as $\mathcal{S}$ from the underlying hyperspectral image,
$\mathcal{L}$.  For empirical comparison, we employ four hyperspectral
datasets: Belcher and Salinas from \cite{WLG2025a}, as well as Urban
and Beach from \cite{KZL2017}. These four datasets are depicted in
Fig.\ref{fig:HADdata} along with their ground-truth maps for known
anomalies with the datasets.  Mathematically, each dataset represented
by a $H\times W \times C$ tensor, where $H$, $W$, and $C$ correspond
to the spatial height, spatial width, and number of spectral bands,
respectively. We compare to four RPCA-driven methods, namely, PCP
\cite{CLM2011}, KBR \cite{XZM2017}, TRPCA \cite{LFC2020}, and LRTFR
\cite{LZL2024a}. Additionally, we consider four techniques developed
specifically for the hyperspectral-anomaly-detection task: prior-based
tensor approximation (PTA) \cite{LLQ2022};
low-rank and sparse decomposition with mixture of Gaussian (LSDM-MoG)
\cite{LLD2021}; tensor low-tank and sparse representation (TLRSR)
\cite{WWH2023}; and block-term decomposition (BTD) \cite{WLG2025a}.
We note that the most common measure of performance for hyperspectral
anomaly detection---the AUC of the receiver operating characteristic
(ROC) curve---does not apply for a hard classifier such as BCP-RPCC,
since the ROC curve degrades to a single point with no AUC.
Therefore, we use the same performance measures as in
Sec.~\ref{sec:resultsforeground}.


As an order-3 tensor with sizable dimensions in all three modes, a
hyperspectral cube can be very computationally costly to process, and
this is exacerbated for the fully probabilistic treatment taken by
BCP-RPCC.  Consequently, we follow exactly \cite{WWH2023} to reduce the spectral dimension of the cube via PCA
prior to anomaly detection, which transforms the hyperspectral volumes into size $H\times
W\times6$.  For BCP-RPCC, we treat each hyperspectral pixel vector as
its own block such that $J_1=1$, $J_2=1$, and $J_3=6$ with $K_1=H$,
$K_2=W$, and $K_3=1$. Fig.~\ref{fig:HADTuning} depicts the tuning of
$R$ and $\sigma$, which yields $\sigma=4\times10^{-3}$ for Belcher and Beach, $\sigma=9\times10^{-3}$ for Urban and $\sigma=3\times10^{-2}$ for Salinas, while we use $R=35$ for
Belcher and Beach, $R=40$ for Urban and Salinas.


Anomaly-detection results are shown in Fig.~\ref{fig:HADvisual}, with
quantitative results tabulated in
Tables~\ref{tab:HADBelcher}--\ref{tab:HADSalinas} and plotted in
Fig.~\ref{fig:HADcurves}. Once again, the proposed BCP-RPCC is the
only technique providing binary detection results. As can be seen in
Fig.~\ref{fig:HADvisual}, BCP-RPCC produces very few false alarms
whereas the other methods are prone to numerous false alarms.  In
Tables~\ref{tab:HADBelcher}--\ref{tab:HADSalinas}, BCP-RPCC delivers
the best performance by all three metrics across all four
datasets. Additionally, Fig.~\ref{fig:HADcurves} indicates that, while
it is possible that the other methods outperform BCP-RPCC with
their peak performance, the interval for such an ideal threshold is
very narrow. Furthermore, the best threshold varies significantly
between the datasets, suggesting that determining a suitable threshold
would be exceedingly difficult in practice.

\section{Conclusion}
\label{sec:conclusion}
In many applications, foreground elements replace, or occlude,
background elements, setting up a mismatch from the classic additive
formulation of RPCA and real-world scenarios.  Accordingly, in this
paper, we proposed a new low-rank and sparse decomposition in the form
of RPCC. To solve this accurate, yet NP-hard, model, a fully
probabilistic treatment based on BSTF was derived, and, to deliver
necessary randomness without over-parameterization, Gaussian noise
with manually adjustable variance was added to the noiseless
observation.  This unexpectedly led to a hard classification between
foreground and background as this variance approaches zero. We
observed that such hard-classifier performance stood in contrast to
conventional RPCA approaches which provide only soft classification
and require post-hoc thresholding to definitively establish foreground
support.  Experimental findings showed that the proposed model
achieved a near-optimal solution to the NP-hard RPCC problem on
synthetic datasets as well as robust foreground-extraction
and anomaly-detection performance on color video and hyperspectral
datasets, respectively.

\section{Further Discussion}
\label{sec:Discussion}
While BCP-RPCC may not outperform traditional RPCA in terms of peak performance, that is not the main message. By directly estimating the support of the sparse component, it entirely eliminates the need for a post-hoc thresholding step---a step that is often trickier than solving the RPCA model itself. One might argue that thresholding still exists implicitly in the form of the hyperparameter $\sigma$, which represents the variance of the added noise. However, there is an essential difference: by transforming the post-hoc threshold $\tau$ into the pre-hoc “threshold” $\sigma$, the thresholding step becomes an integral part of the low-rank/sparse decomposition model, whereas this is not the case with $\tau$. This shift opens the possibility of determining $\sigma$ theoretically, much like the balancing parameter $\lambda$ in PCP \cite{CLM2011}. As established by the proof of Prop. \ref{Prop:HardClass}, $\sigma$ is strongly correlated with the fitting error of the low-rank component. With the post-hoc threshold $\tau$, by contrast, we have almost no such theoretical grounding.

That said, this paper serves as a starting point for multiple research directions with both theoretical and practical significance:
\begin{itemize}
	\item \textbf{Identifiability Analysis:}
	Analogous to RPCA, it is now meaningful to understand under what conditions RPCC can identify the true posterior distribution of the low-rank component and its support, and up to what precision. Such an analysis would provide analytical guidelines for determining $\sigma$.
	
	\item \textbf{Model Advancement:}
	The current model does not yet achieve satisfactory restoration of the low-rank component, which is why its background reconstruction performance, unlike in typical RPCA studies, is not reported. This limitation stems from the relatively poor low-rank representation capability of CP decomposition, which necessitates a very large $R$ for accurate background reconstruction---a setting that is often infeasible on common devices. This motivates research into enhancing model performance by replacing CP decomposition with more powerful alternatives, such as tensor singular value decomposition \cite{LFC2020}, Tucker decomposition \cite{tucker1963implications}, tensor ring decomposition \cite{zhao2016tensor}, and others.
	
	\item \textbf{Task Enhancement:}
	Given the significance of the proposed model as a hard classifier, many real-world tasks can benefit from enhanced practicality. For instance, hyperspectral anomaly detection can now embark on a new chapter of  hyperspectral ``hard'' anomaly detection, which eliminates the need for threshold selection and thus offers greater pragmatism.
\end{itemize}

\bibliographystyle{IEEEtran}
\bibliography{string-defs,bibfile,fowler}

\end{document}